\documentclass[journal]{IEEEtran}
\usepackage[english]{babel}
\usepackage[utf8x]{inputenc}
\usepackage[T1]{fontenc}
\usepackage{bbm}

\usepackage[letterpaper, margin=0.75in]{geometry}

\usepackage[ruled, noend, linesnumbered]{algorithm2e}
\usepackage{amsmath}
\usepackage{amssymb}
\usepackage{dsfont}
\usepackage{bm}
\usepackage{tabularx, booktabs}
\usepackage{multicol}
\usepackage{multirow}

\newcommand{\ra}[1]{\renewcommand{\arraystretch}{#1}}

\DeclareMathOperator*{\argmin}{argmin}
\DeclareMathOperator*{\argmax}{argmax}
\newcommand{\pluseq}{\mathrel{+}=}
\newcommand{\minuseq}{\mathrel{-}=}
\renewcommand{\vec}[1]{\bm{#1}}

\usepackage{dblfloatfix} 
\usepackage{graphicx}
\usepackage{wrapfig}
\usepackage{subcaption}

\usepackage[colorinlistoftodos]{todonotes}
\usepackage[colorlinks=true, allcolors=blue]{hyperref}

\usepackage{amsthm}

\newtheorem{claim}{Claim}

\title{Learning-Based Proxy Collision Detection for Robot Motion Planning Applications}
\author{Nikhil Das and Michael Yip\thanks{Nikhil Das and Michael Yip are with the Department of Electrical and Computer Engineering, University of California San Diego, La Jolla, CA 92093 USA. {\tt\small \{nrdas, yip\}@ucsd.edu}}}

\begin{document}
\maketitle

\begin{abstract}
This paper demonstrates that collision detection-intensive applications such as robotic motion planning may be accelerated by performing collision checks with a machine learning model. We propose Fastron, a learning-based algorithm to model a robot's configuration space to be used as a proxy collision detector in place of standard geometric collision checkers. We demonstrate that leveraging the proxy collision detector results in up to an \textit{order of magnitude} faster performance in robot simulation and planning than state-of-the-art collision detection libraries. Our results show that Fastron learns a model more than 100 times faster than a competing C-space modeling approach, while also providing theoretical guarantees of learning convergence. Using the OMPL motion planning libraries, we were able to generate initial motion plans across all experiments with varying robot and environment complexities. With Fastron, we can repeatedly perform planning from scratch at a 56 Hz rate, showing its application toward autonomous surgical assistance task in shared environments with human-controlled manipulators. All performance gains were achieved despite using only CPU-based calculations, suggesting further computational gains with a GPU approach that can parallelize tensor algebra. Code is available online\footnote{\href{https://github.com/ucsdarclab/fastron}{https://github.com/ucsdarclab/fastron}}.
\end{abstract}

\section{Introduction}
Motion planning, the task of determing a path for a robot from a start to a goal position while avoiding obstacles, is a requirement for almost all applications in which a robot must move in its environment. Motion planning for robots is often performed in its configuration space (C-space), a space in which each element represents a unique configuration of the robot \cite{Choset2005}. The joint space of a robot manipulator is an example C-space, as a set of joint positions may be sufficient to fully define the pose of a robot manipulator in the workspace. The dimensionality of the C-space matches the number of controllable degrees of freedom (DOF) of the robot.

The workspace may include restrictions on the set of feasible robot configurations. Examples include workspace obstacles (such as the three cubes in Fig. \ref{fig:exampleWorkspace} restricting the 3 DOF arm), self-collisions, or joint limits. For a given workspace, each configuration in the C-space can belong to one of two subspaces: $\mathcal{C}_{free}$ or $\mathcal{C}_{obs}$. A configuration is in $\mathcal{C}_{free}$ if the robot is not in contact with any workspace obstacle when in that corresponding configuration; otherwise, the configuration is in $\mathcal{C}_{obs}$ \cite{Choset2005, Lozano1983}. The robot and obstacles in Fig. \ref{fig:exampleWorkspace} are represented as the point and amorphous bodies in the C-space representation in Fig. \ref{fig:exampleCspace}, respectively. The advantage of planning in C-space is it is easier to plan how to move a point than it is to move bodies with volume \cite{Choset2005}.

On the other hand, a disadvantage of the concept of C-space is no simple closed-form expression usually exists to perfectly separate $\mathcal{C}_{obs}$ from $\mathcal{C}_{free}$ \cite{Lozano1983,Burns2005,Pan2015}. Determining in which subspace a single configuration belongs typically requires discrete collision detection between the robot and the obstacles \cite{Pan2015}. For actual robotic systems, in addition to performing tests for intersection of robot links with workspace obstacles, the collision detection cycle also includes performing forward kinematics to determine where robot geometry would be located and using sensor readings to obtain obstacle information. Repeated queries to a collision checker is computationally expensive, taking up to 90\% of computation time for sampling-based motion planning \cite{Elbanhawi2014}. Apart from motion planning \cite{Choset2005, Elbanhawi2014}, applications that may require numerous collision checks include reinforcement learning \cite{Duguleana2012}, robot self-collision analysis \cite{Kuffner2002}, and robot simulations \cite{Rohmer2013}.

\begin{figure}[t!]
	\centering
	\begin{subfigure}[b]{0.48\linewidth}
      {\includegraphics[width=\linewidth,trim={2cm 1.7cm 1cm 2cm},clip]{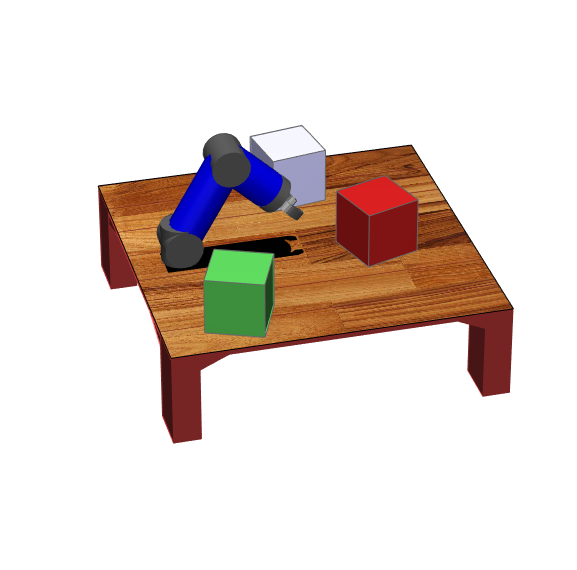}}
      \caption{Workspace}
      \label{fig:exampleWorkspace}
    \end{subfigure}
	\begin{subfigure}[b]{0.48\linewidth}
	  {\includegraphics[width=\linewidth,trim={0cm 0cm 0cm 0cm},clip]{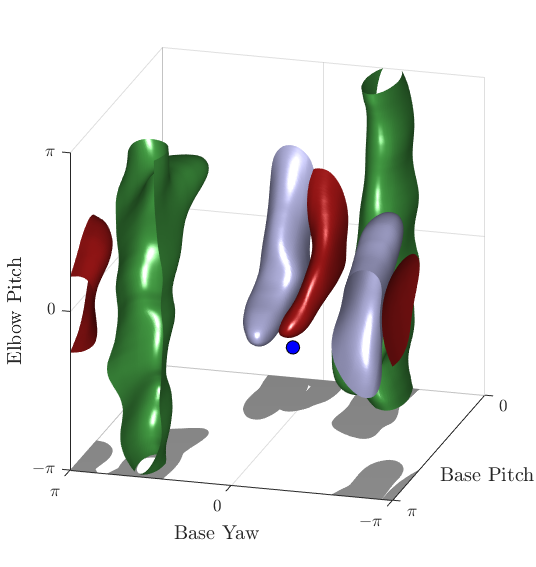}}
      \caption{Configuration space}
      \label{fig:exampleCspace}
    \end{subfigure}
    \caption{Example workspace with a 3 DOF robot and multiple cube obstacles (left) and its corresponding C-space model generated using the Fastron algorithm (right). Each workspace obstacle on the left and its corresponding C-space obstacle on the right match in color. C-space obstacles due to the table are excluded to improve clarity. The robot's current configuration is represented as the blue point on the right.}
    \label{fig:exampleWorkspaceCspace}
\end{figure}

\subsection{Contributions}
In this paper, we present a novel approach to collision detection using machine learning techniques and employ this approach to motion planning applications. Building and querying a computationally cheaper model to bypass the collision detection cycle will relieve significant computational burden due to collision checking, allowing faster sampling-based motion planning.

The subspaces of C-space may be modeled using machine learning techniques, allowing the model to serve as a proxy to the collision checking cycle. Moving the obstacles in the workspace will invalidate any C-space model trained on a specific static environment. Our previous work introduced a fast training algorithm, Fastron, that efficiently checks for changes in the C-space due to moving workspace obstacles \cite{Das2017}. The training procedure of Fastron was inspired by the kernel perceptron algorithm \cite{Freund1999} due its simplicity in implementation and its ability to model training points in an online manner; modifications were made to make training fast and allow efficient adaptation to a dataset with labels that change over time.

The original Fastron algorithm worked with a fixed dataset of robot configurations $\mathcal{X}$ and relabeled the collision statuses of the configurations when workspace obstacles moved to avoid having to generate new configurations and having to recompute the distance matrix needed for model updates. A disadvantage of retaining a fixed dataset is that a large amount of memory is required to store $\mathcal{X}$ and its associated distance matrix. The amount of memory required was often much larger than necessary because typically only a small subset of $\mathcal{X}$ is needed during configuration classification. Furthermore, using a fixed dataset did not allow augmenting $\mathcal{X}$ with new data, causing the model to be dependent on the original distribution of random samples.

While Fastron originates from our prior work, this paper revamps the algorithm and provides significant theoretical and empirical validation of Fastron for proxy collision detection. In this paper, we address the disadvantages of a fixed dataset by using a dynamically changing dataset to allow resampling and a reduced memory footprint while addressing the scalability of the system to complex environments. We also introduce lazy Gram matrix evaluation and utilize a cheaper kernel to further improve training and classification time. We analytically show our training algorithm will always create a model yielding positive margin for all training points. Finally, we demonstrate the capabilities of Fastron empirically by showing that Fastron allows faster motion planning on a variety of robots and sampling-based motion planners. Results shown in this paper for Fastron show positive results with only CPU-based calculations. Components of this algorithm such as query classification and dataset labeling are easily parallelizable to potentially yield even further improvements.

\subsection{Related Work}
Previous works have utilized machine learning to model the subspaces of C-space or to bypass or accelerate collision checking in motion planning.

Pan et al. \cite{Pan2015} use incremental support vector machines to represent an accurate collision boundary in C-space for a pair of objects and an active learning strategy to iteratively improve the boundary. This method is suitable for changing environments because moving one body relative to another body's frame is simply represented as translating a point in configuration space. However, a new model is required for each pair of objects, and each model must be trained offline.

\begin{figure}[t]
  \centering
  {\includegraphics[width=\linewidth]{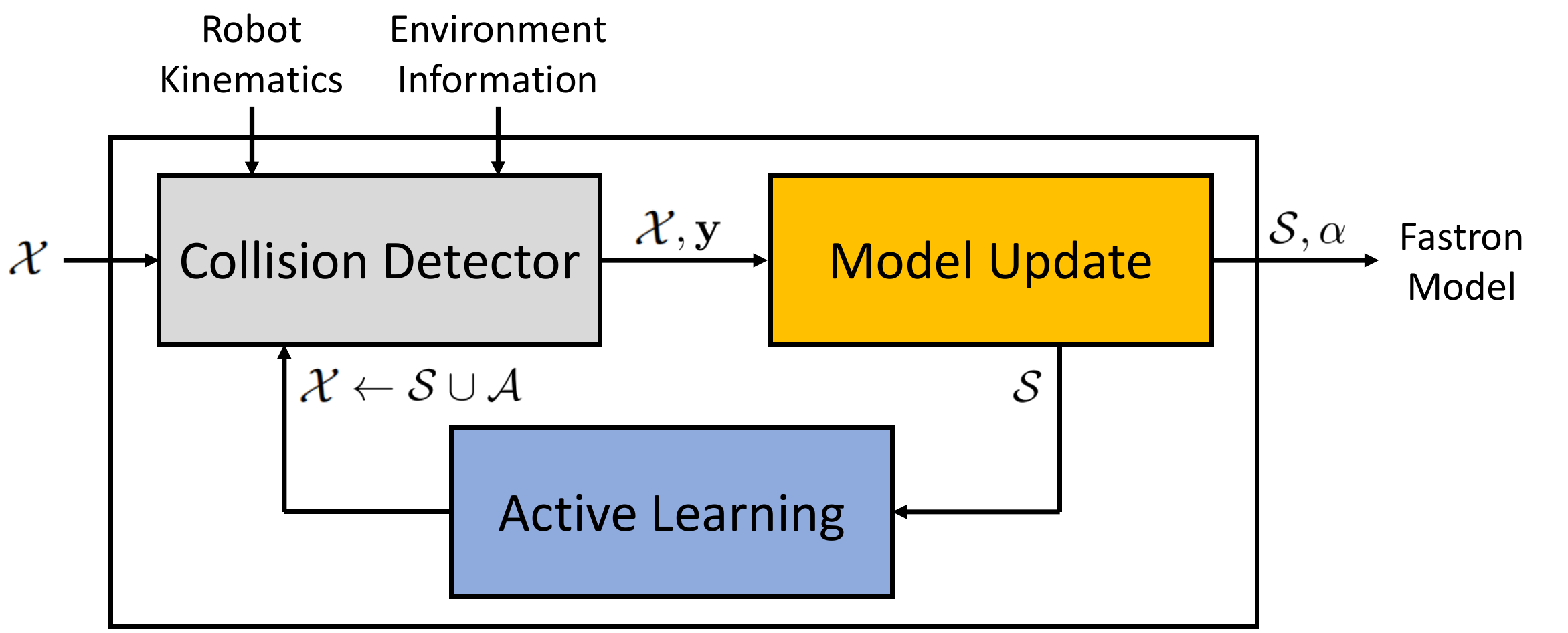}}
  \caption{Block diagram illustrating the pipeline of the Fastron algorithm. The collision detector assigns true labels $\vec{y}$ for the points in $\mathcal{X}$. The model update determines the weights $\vec{\alpha}$ and support set $\mathcal{S}$ which may be used for proxy collision detection. The active learning strategy augments $\mathcal{S}$ with a set of unlabeled configurations $\mathcal{A}$ before the cycle repeats. }
  \label{fig:blockDiagram}
\end{figure} 

Huh et al. \cite{Huh2016} use Gaussian mixture models (GMMs) to represent in-collision and collision-free regions specifically for use in motion planning using Rapidly-Exploring Random Trees (RRTs) \cite{LaValle2008}. The GMMs are used both as a proxy collision detector by labeling query points according to Mahalanobis distances from the GMM means and as sampling distributions when finding new nodes for RRT tree expansion. Their results show their GMM-based RRTs can generate motion plans up to 5 times faster than a bidirectional RRT. A downside to this algorithm is the RRT routine must be called repeatedly to generate enough data to achieve consistent proxy collision detector performance, e.g., 3 to 4 RRT calls were required in their algorithm execution time trials. Relying on repeated RRTs for data collection may slow performance, especially when the environment is continuously changing. Furthermore, a downside to GMMs is the number of components in the mixture is fixed, which indicates it may be more difficult to model all possible C-space obstacles for a given robot than it would be when using a nonparametric model.

KNNs have also been used for C-space modeling \cite{Burns2005} and proxy collision detection \cite{Pan2016}, but have only been implemented in static environments. Using KNNs for collision detection in sampling-based motion planning may generate motion plans up to 2 times faster \cite{Pan2016}. Disadvantages of the KNN approach include requiring to store the entire training set and not being able to easily adapt for a changing environment.

Neural networks have been applied to perform collision detection for box-shaped objects and have achieved suitably low enough error to calculate collision response in physics simulations \cite{Garcia2002}. A disadvantage of the neural network approach is there is typically no formulaic method to determine the optimal set of parameters for neural networks, which in this case required training thousands of networks to find the best-performing network. A significantly large amount of data was required to train and cross-validate the models. Finally, this method has only been tried on box obstacles, suggesting a new network must be trained for other objects.

Qureshi et al. proposes a potentially transformative approach to motion planning and bypasses collision checking during motion planning runtime by directly generating waypoints with MPNet, a pair of neural networks that encodes the workspace and generates feasible motion plans \cite{Qureshi}. Motion plans may be generated up to 100 times faster with MPNet than the state-of-the-art BIT* motion planning method. One limitation is the excessive amount of data needed to train MPNet.

The purpose for the Fastron algorithm is to provide a global C-space model that can adapt to a changing workspace. As an online algorithm, all datapoints and all training are achieved during runtime, enabling adaptability to various or changing environments without requiring a large amount of a priori data or requiring multiple instances of models to account for various scenarios. Additionally, Fastron directly adapts its discriminative model to correct misclassified configurations rather than relying on samples to influence generative models, which allows efficient updates of the decision boundary separating $\mathcal{C}_{obs}$ from $\mathcal{C}_{free}$ in continuously changing environments.

\section{Methods}
In this section, we describe the components of the Fastron algorithm. We begin by describing the binary classification problem before going into the details of how Fastron trains and updates a model. These details also include a proof that this algorithm will eventually always find a model that correctly classifies all points in the training dataset. We finally address some practical aspects of the algorithms to consider when implementing this algorithm. The block diagram in Fig. \ref{fig:blockDiagram} shows the pipeline of the entire algorithm.

\subsection{Binary Classification}
The goal of training a binary classifier is to find a model whose output predicts in which of two classes a query point belongs. The parameters and weights chosen for the model are based on a training set $\mathcal{X}=\{\mathcal{X}_1, \ldots, \mathcal{X}_N\}$ and its associated training labels $\vec{y}$. In this paper, the elements of $\mathcal{X}$ are $d$-dimensional robot configurations, e.g., the $d$ joint angles defining a robot manipulator's position. We assume further that these joint angles may be scaled to a subspace of $\mathbb{R}^d$\footnote{We apply joint limits to C-spaces that involve 1-sphere $\mathbb{S}^1$ spaces, e.g., for the 3 DOF robot in Fig. \ref{fig:exampleWorkspace}, the base yaw and elbow pitch revolute joints are limited to $(-\pi,\pi]$ and the base pitch joint is limited to $[0,\pi]$.}. We use the label $\vec{y}_i=+1$ to denote $\mathcal{X}_i$ is in the in-collision class and $\vec{y}_i=-1$ to denote $\mathcal{X}_i$ is in the collision-free class.

Since the labels are $\pm 1$, one possible model for prediction may be $sign(f(\vec{x}))$, where $f:\mathbb{R}^d\rightarrow \mathbb{R}$ is some hypothesis function. The hypothesis function may be $f(\vec x) = \sum_i\langle\phi(\mathcal{X}_i),\phi(\vec{x})\rangle_\mathcal{F}\vec{\alpha}_i$ where $\phi : \mathbb{R}^d \rightarrow \mathcal{F}$ is a mapping to some (often higher-dimensional) feature space $\mathcal{F}$ and weight vector $\vec{\alpha} \in \mathbb{R}^N$ contains a weight for each $\mathcal{X}_i\in \mathcal{X}$. This hypothesis can be interpreted as a weighted sum comparing the query configuration to each training configuration and is used in discriminative models such as kernel perceptrons \cite{Freund1999} and support vector machines \cite{Cortes1995, Diehl}. $f(\vec x)$ may be written in matrix form for the configurations in $\mathcal{X}$ as $\vec F = \vec{K\alpha}$ where $\vec{F}_i=f(\mathcal{X}_i)\ \forall \mathcal{X}_i\in\mathcal{X}$ and $\vec K$ is the Gram matrix for $\mathcal{X}$ where $\vec{K}_{ij}=\langle\phi(\vec{x}),\phi(\vec{x}')\rangle_\mathcal{F}$.

The goal of the Fastron learning algorithm is to find $\vec{\alpha}$ such that $sign(\vec{F}_i) = \vec{y}_i \ \forall i$. Alternatively, the goal is to find $\vec{\alpha}$ such that the margin $m(\mathcal{X}_i) = \vec{y}_i\vec{F}_i > 0 \ \forall i$. Only the configurations in $\mathcal{X}$ with a nonzero weight $\vec{\alpha}_i$ are needed to compute a label prediction, and the rest may be discarded after an $\vec{\alpha}$ satisfying the positive margin condition is found. The configurations with nonzero weights comprise the support set $\mathcal{S} \subseteq \mathcal{X}$ of the model. Once the Fastron model is trained, $sign(f(\vec{x}))$ may be used as a proxy to performing collision detection for query configuration $\vec{x}$.


\subsection{Kernel Function}
\label{sec:kernel}

\begin{figure}[b!]
	\centering
{\includegraphics[width=\linewidth, trim={0cm 0.8cm 0.3cm 0cm}, clip]{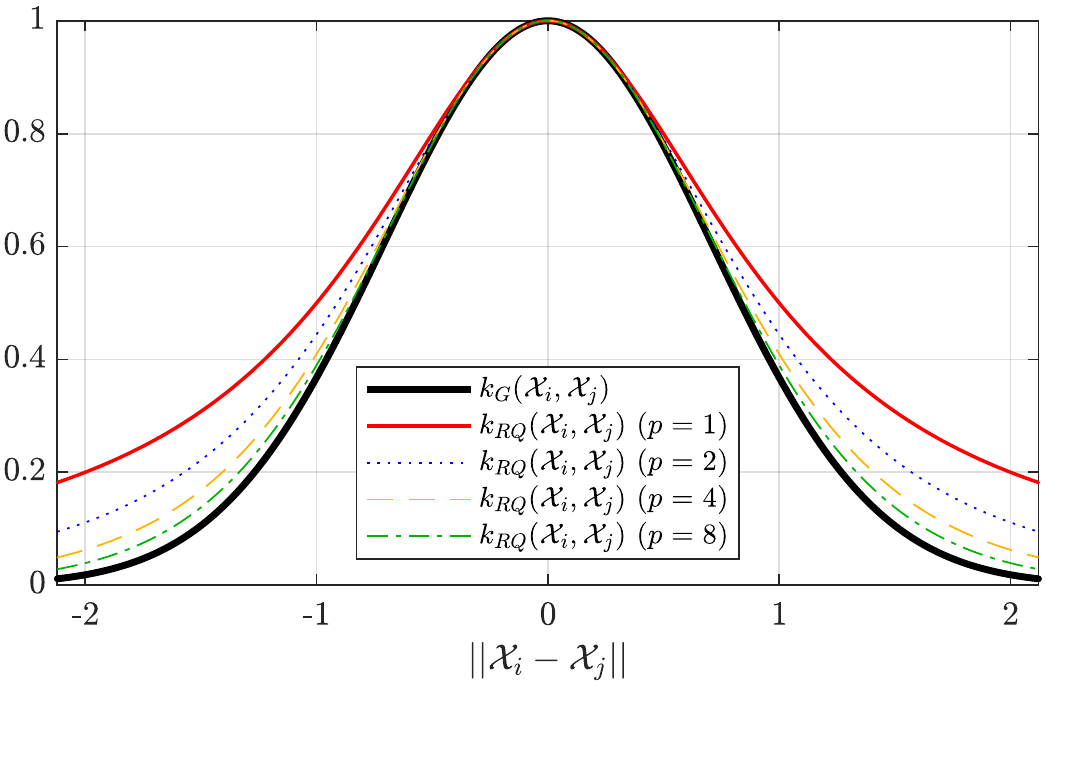}}
  \caption{The rational quadratic kernel $k_{RQ}(\cdot, \cdot$) is an approximation of the more expensive Gaussian kernel $k_{G}(\cdot, \cdot)$. $\gamma = 1$ in this plot.}
  \label{fig:kernelShapes}
\end{figure}

A kernel function $k(\vec{x},\vec{x}')=\langle\phi(\vec{x}),\phi(\vec{x}')\rangle_\mathcal{F}$ compares a configuration $\vec{x}$ to $\vec{x}'$ by mapping to some feature space $\mathcal{F}$ and taking an inner product. $k(\cdot,\cdot)$ should provide a large score for two similar configurations and a low score for dissimilar configurations. The Gram matrix $\vec{K}$ where $\vec{K}_{ij}=k(\mathcal{X}_i,\mathcal{X}_j)$ is thus a similarity matrix and is a useful tool in machine learning techniques, such as support vector machines \cite{Cortes1995, Diehl} and kernel perceptrons \cite{Freund1999}, where similarity between two samples may suggest they share classification labels.

A popular kernel function is the Gaussian kernel, defined as $k_G(\vec{x},\vec{x}') = exp(-\gamma\|\vec{x}-\vec{x}'\|^2)$, where $\gamma > 0$ is a parameter dictating the width of the kernel. $k_G(\cdot,\cdot)$ is often considered the default kernel choice for kernel-based machine learning applications when there is limited prior knowledge of the underlying structure of the data \cite{Smola1998}.

$k_G(\cdot,\cdot)$ is known to be a positive semidefinite kernel \cite{Rasmussen2006}, which means $\vec{K}$ is a positive definite matrix when each configuration in $\mathcal{X}$ is unique \cite{Hofmann2008}. According to Mercer's theorem, for each positive semidefinite kernel, there exists a mapping $\phi : \mathbb{R}^d \rightarrow \mathcal{H}$ where $\mathcal{H}$ is a Hilbert space such that $k(\vec{x},\vec{x}') =\langle\phi(\vec{x}),\phi(\vec{x}')\rangle_\mathcal{H}$ \cite{Minh}. Thus, there is no need for explicit mapping to a feature space when utilizing the Gaussian kernel function; computing the inner product without explicitly mapping to a feature space is known as the kernel trick. The kernel trick is useful when learning a classifier that requires richer features than provided in the input space $\mathbb{R}^d$ but explicit mapping to a higher-dimensional may be difficult or time-consuming. As the shapes of $\mathcal{C}_{free}$ and $\mathcal{C}_{obs}$ are typically complex for robotic applications, this implicit mapping to a richer feature space allows easier separation of the two classes.

While the Gaussian kernel is popular for machine learning applications, exponential evaluations are computationally expensive operations, which can slow down training and classification especially for large $N$. Exponentiation can be avoided by representing the Gaussian kernel with a limit:
\begin{equation}
k_{G}(\vec{x},\vec{x}') = \lim_{p\rightarrow\infty}\left(1+\frac{\gamma}{p}\|\vec{x}-\vec{x}'\|^2\right)^{-p}
\end{equation}

For finite values of $p$, the above kernel is the rational quadratic kernel \cite{Rasmussen2006}, which we represent as $k_{RQ}(\cdot,\cdot)$. The rational quadratic kernel may alternatively be derived as a superposition of Gaussian kernels of varying kernel widths: $k_{RQ}(\vec x,\vec y) = \int {P(\gamma;p,\theta)k_G(\vec x, \vec y)} d\gamma$, where $P(\gamma;p,\theta)$ is a gamma distribution \cite{Rasmussen2006}. The sum of positive semidefinite kernels is also a positive semidefinite kernel \cite{Hofmann2008}. As a sum of Gaussian kernels, the rational quadratic kernel satisfies the positive semidefiniteness property, showing $k_{RQ}(\vec x,\vec y)$ provides the result of an inner product in a richer Hilbert space via the kernel trick according to Mercer's theorem and yields positive definite Gram matrices when each configuration in $\mathcal{X}$ is unique.

Better approximations to the Gaussian kernel are achieved with higher values of $p$ as seen in Fig. \ref{fig:kernelShapes}. When $p$ is a power of 2, efficient implementation of the rational quadratic kernel may take advantage of tetration. In this paper, we use $p = 2$ as the approximation to a Gaussian need not be extremely precise to generate a similarly performing classifier. Using the Eigen C++ library, completely computing a $5000 \times 5000$ Gram matrix takes 257.0 ms using $k_G(\cdot, \cdot)$, while using $k_{RQ}(\cdot,\cdot)$ takes 171.3 ms, illustrating that kernel evaluations can be performed 1.5 times faster.

\subsection{Weight Update}

\subsubsection{Derivation of Update Rule}
The Fastron algorithm is inspired by the kernel perceptron in the sense that it iteratively adds or adjusts the weights for configurations that are incorrectly classified by the current model \cite{Freund1999}. Unlike the perceptron algorithm, the Fastron algorithm prioritizes the configurations in $\mathcal{X}$ with the most negative margin and makes adjustments to the weights such that the configuration is forced to be correctly classified immediately after the update, which is not a guarantee with the standard perceptron algorithm. These changes result in a model with a large number of weights equal to 0 that generally takes fewer updates to converge to a solution compared to the perceptron algorithm. 

\begin{algorithm}[t]
\caption{Fastron Model Updating}
\label{alg:fastronUpdate}

\SetKwInput{Param}{Parameters}
\SetKw{logicalor}{OR}
\SetKw{continue}{continue}
\SetKw{break}{break}
\DontPrintSemicolon

 \KwIn{Training dataset of configurations $\mathcal{X}$, collision status labels $\vec{y}$}
\Param{Maximum number of iterations $iter_{max}$, maximum number of support points $\mathcal{S}_{max}$, conditional bias parameter $\beta$, kernel parameter $\gamma$ for Gram matrix calculations}
  \KwOut{Updated $\vec{\alpha}$, support set of configurations $\mathcal{S}$}
  \tcp{Get weights, hypothesis, and Gram matrix from previous update}
  $\vec{\alpha}, \vec{F}, \vec{K} \leftarrow \mathrm{loadPreviousModel}()$ \\
  \tcp{Back up previous $\vec{\alpha}$ and $\vec{F}$}
$\vec{\alpha_{before}} \leftarrow \vec{\alpha},\ \vec{F_{before}} \leftarrow \vec{F}$ \\
 \For {$iter = 1$ to $iter_{max}$}
 {
   \tcp{Check for misclassifications}
 \uIf{$\min{\vec{y}\vec{F}} \leq 0$}{
 		$i \leftarrow \argmin \vec{y}\vec{F}$ \\
        $\mathrm{computeGramMatrixColumn}(i)$ \\
        \tcp{Add/adjust support point}
        \uIf {$\mathrm{count}(\vec{\alpha}\neq 0)<\mathcal{S}_{max}$ \logicalor $\vec{\alpha}_i\neq 0$}{
  		  $\Delta\vec{\alpha} \leftarrow {\beta}^{0.5(\vec{y}_i+1)}\vec{y}_i-\vec{F}_i$ \\
          $\vec{\alpha}_i \leftarrow \vec{\alpha}_i + \Delta \vec{\alpha}$ \\
          $\vec{F} \leftarrow \vec{F} + \Delta\vec{\alpha}\vec{K}\vec{e(i)}$ \\
          \continue
        }
   }
   \tcp{Back up $\vec{\alpha}$ and $\vec{F}$}
	$\vec{\alpha_{before}} \leftarrow \vec{\alpha},\ \vec{F_{before}} \leftarrow \vec{F}$ \\
   \tcp{Remove redundant support points}
   \If{$\max{\vec{y} (\vec{F} - \vec{\alpha})} > 0 \textup{ subject to } \vec{\alpha}\neq 0$}{
     $i \leftarrow \argmax \vec{y}(\vec{F} - \vec{\alpha})$ subject to $\vec{\alpha}\neq 0$\\
     $\vec{F} \leftarrow \vec{F} - \vec{\alpha}_i\vec{K}\vec{e(i)}$\\
     $\vec{\alpha}_i \leftarrow 0$ \\
     \continue
   }
   \break
 }
 \tcp{Revert solution if prior was better}
  \uIf{$\mathrm{count}({\vec{y}\vec{F_{before}}} \leq 0) < \mathrm{count}({\vec{y}\vec{F}} \leq 0)$}
 {
  $\vec{\alpha}\leftarrow\vec{{\alpha}_{before}},\ \vec{F}\leftarrow \vec{F_{before}}$ \label{removalSafeguard}}
 \tcp{Remove all elements corresponding to non-support points}
 $\alpha, \mathcal{S}, \vec{F}, \vec{K} \leftarrow \textup{removeNonsupportPoints}()$ \label{removeNonsupportPoints} \\

 \textbf{return} {$\vec{\alpha}$, $\mathcal{S}$}
\end{algorithm}

By prioritizing the training configuration with the most negative margin, we guarantee that the weight we are updating is for a misclassified configuration. In what follows, parenthetical superscripts denote the training iteration upon which the given value depends. If $\mathcal{X}_i$ has the most negative margin of all configurations in $\mathcal{X}$ on iteration $n$, $\vec{\alpha}_i^{(n)}$ must be adjusted by $\Delta \vec{\alpha}_i^{(n)}$ such that $m^{(n+1)}(\mathcal{X}_i)>0$ (thereby ensuring $\mathcal{X}_i$ is correctly classified immediately after the weight update):
\begin{align}
\vec{y}_i\left(\sum_{j\neq i}\vec{K}_{ji}\vec{\alpha}_j^{(n)} + \vec{K}_{ii}\left(\vec{\alpha}_i^{(n)} + \Delta\vec{\alpha}_i^{(n)}\right)\right)>0
\end{align}

As $\vec{K}_{ii} = 1$ when using the rational quadratic kernel, the above condition may be simplified to $m^{(n)}(\mathcal{X}_i) + \vec{y}_i\Delta \vec{\alpha}_i^{(n)} > 0$, which may be enforced by setting $\Delta \vec{\alpha}_i^{(n)} = \vec{y}_i - \vec{F}_i^{(n)}$. This update rule makes $m^{(n+1)}(\mathcal{X}_i) = 1$. The weight and training hypothesis vector updates are thus (where $\vec{e(i)}$ is the $i^{th}$ standard basis vector):
\begin{align}
\Delta \vec{\alpha}_i^{(n)} &= \vec{y}_i - \vec{F}_i^{(n)}  \label{eqn:updateRule1}\\
\vec{\alpha}^{(n+1)} &= \vec{\alpha}^{(n)} + \Delta\vec{\alpha}_i^{(n)}\vec{e(i)}  \label{eqn:updateRule2}\\
\vec{F}^{(n+1)} &= \vec{F}^{(n)} + \Delta\vec{\alpha}_i^{(n)}\vec{K}\vec{e(i)} \label{eqn:updateRule3}
\end{align}

Note that only one element of $\vec{\alpha}$ is updated per iteration, while all $N$ elements of $\vec{F}$ are updated per iteration.

\subsubsection{Alternative Derivation of Update Rule}
Representing the problem with a loss function provides an alternative method for deriving the update rule and yields interesting insights. Consider the loss function
\begin{equation}
\mathcal{L}(\vec \alpha) = \frac{1}{2}\vec{\alpha}^\mathsf{T} \vec{K\alpha} - \vec{y}^\mathsf{T}\vec{\alpha}
\label{eqn:lossFn}
\end{equation}

This loss function can be derived using the method of Lagrange multipliers, where we seek to minimize $\|\sum_i\phi(\mathcal{X}_i)\vec{\alpha}_i\|^2$ (equivalent to maximizing the distance of the projections of training points in the feature space) subject to the constraint $\vec{y}_if(\mathcal{X}_i)\geq 1 \ \forall \mathcal{X}_i\in \mathcal{X}$. The loss function in Eq. \ref{eqn:lossFn} is similar to that used in support vector machines \cite{Cortes1995}, but the constraint when minimizing Eq. \ref{eqn:lossFn} is now $\vec{y}_i\vec{\alpha}_i\geq 0\ \forall i$. See the Appendix for the derivation of Eq. \ref{eqn:lossFn} for our definition of the hypothesis function $f(\vec{x})$.

A quadratic programmer may be used to minimize Eq. \ref{eqn:lossFn}, but finding the optimal solution is undesired for computational effort and lack of sparsity. To improve the training time and sparsity of the model, Fastron takes a greedy coordinate descent approach and terminates when $m^{(n)}(\mathcal{X}_i)>0 \ \forall \mathcal{X}_i\in \mathcal{X}$. Coordinate descent updates one element of $\vec{\alpha}^{(n)}$ per iteration while leaving all other elements fixed \cite{Wright2015}. On iteration $n$, coordinate descent minimizes $\mathcal{L}(\vec{\alpha}^{(n)})$ along the $i^{th}$ axis by setting $\vec{\alpha}_i^{(n+1)}$ to the solution of $\frac{\partial\mathcal{L}(\vec{\alpha}^{(n)})}{\partial\vec{\alpha}_i^{(n)}}=0$:
\begin{equation}
\vec{\alpha}_i^{(n+1)} = \vec{y}_i - \sum_{j\neq i}\vec{K}_{ij}\vec{\alpha}_j^{(n)} \label{updateRule}
\end{equation}

Replacing $\vec{\alpha}_i^{(n+1)}$ in Eq. \ref{updateRule} with $\vec{\alpha}_i^{(n)} + \Delta \vec{\alpha}_i^{(n)}$, we realize $\Delta \vec{\alpha}_i^{(n)}=\vec{y}_i - \vec{F}_i^{(n)}$, matching the result in Eq. \ref{eqn:updateRule1}. It follows that $\vec{\alpha}^{(n)}$ and $\vec{F}^{(n)}$ will be incremented as shown in Eq. \ref{eqn:updateRule2} and \ref{eqn:updateRule3}.

\begin{figure*}[t]
	\centering
	\begin{subfigure}[b]{0.49\linewidth}
      \includegraphics[width=\linewidth]{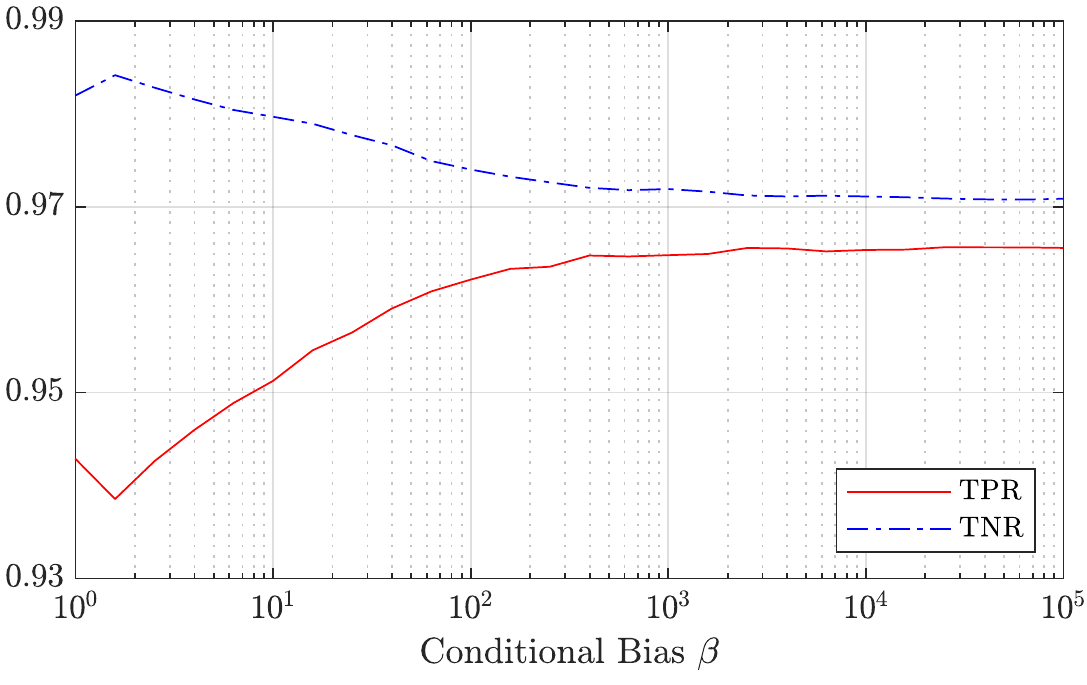}
      \caption{2 DOF, $N = 2000$}
      \label{fig:condBias2dof}
    \end{subfigure}
	\begin{subfigure}[b]{0.49\linewidth}
      \includegraphics[width=\linewidth]{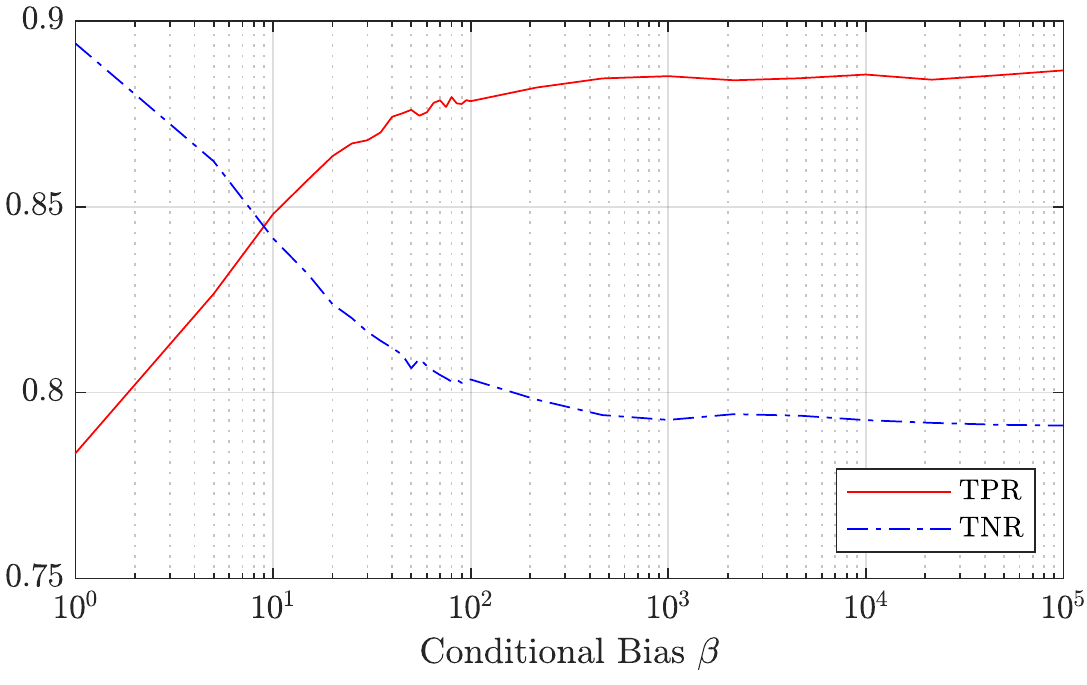}
      \caption{4 DOF, $N = 4000$}
      \label{fig:condBias4dof}
    \end{subfigure}
    \caption{TPR and TNR of Fastron for various conditional bias $\beta$ values for the 2 and 4 DOF robot shown in Fig. \ref{fig:2and4dofEnv}. The ground truth labels are provided using the GJK \cite{Gilbert1988} collision detection method. Increasing the parameter improves TPR at the cost of TNR, but the effects taper off as $\beta$ gets large.}
    \label{fig:condBias}
\end{figure*}

\begin{figure}[b!]
	\centering
	\begin{subfigure}[b]{0.48\linewidth}
	\includegraphics[width=\linewidth,trim={1cm 0.5cm 1cm 2cm},clip]{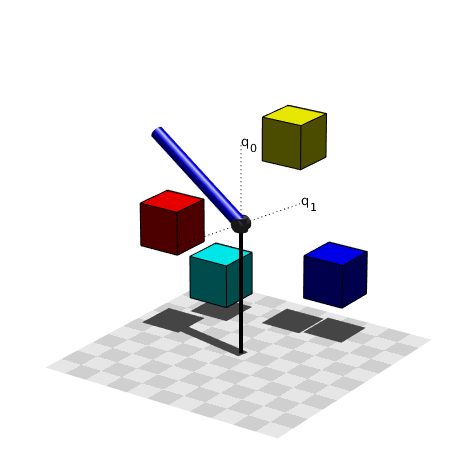}
      \caption{2 DOF}
      \label{fig:2dofEnv}
    \end{subfigure}
	\begin{subfigure}[b]{0.48\linewidth}
      \includegraphics[width=\linewidth,trim={1cm 0.5cm 1cm 2cm},clip]{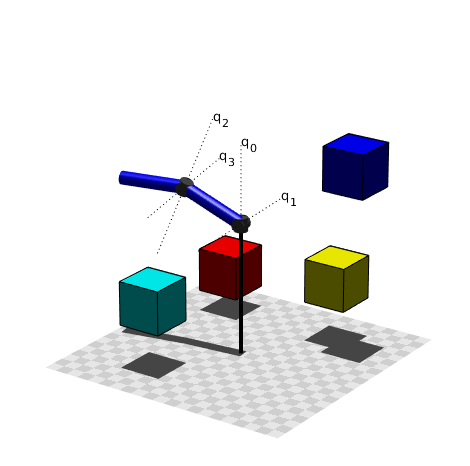}
      \caption{4 DOF}
      \label{fig:4dofEnv}
    \end{subfigure}
    \caption{Example environments with up to 4 cube obstacles and simple 2 and 4 DOF robots. There are half as many links as there are joints, and each pair of joints is overlapping.}
    \label{fig:2and4dofEnv}
\end{figure}


Updating $\vec{\alpha}_i^{(n)}$ changes $\mathcal{L}(\vec{\alpha}^{(n)})$ as follows:
\begin{align}
\mathcal{L}\left(\vec{\alpha}^{(n+1)}\right) &= \mathcal{L}\left(\vec{\alpha}^{(n)} + \Delta\vec{\alpha}_{i}^{(n)}\vec{e(i)}\right) \\
&= \frac{1}{2}{\vec{\alpha}^{(n)}}^\mathsf{T}\vec{K\alpha}^{(n)} - \vec{y}^\mathsf{T}\vec{\alpha}^{(n)} + \frac{1}{2}{\Delta\vec{\alpha}_i^{(n)}}^2 \notag \\
&\phantom{=\ }  + \Delta\vec{\alpha}_i^{(n)}\vec{F}_i^{(n)} - \Delta\vec{\alpha}_i^{(n)}\vec{y}_i \\
&= \mathcal{L}\left(\vec{\alpha}^{(n)}\right) - \frac{1}{2}{\Delta\vec{\alpha}_i^{(n)}}^2 \label{deltaLoss}
\end{align}

The coordinate descent direction is selected greedily, i.e., the $\vec{\alpha}_i^{(n)}$ to update should be selected such that there is maximal decrease in loss. Additionally, since the algorithm will terminate when $m^{(n+1)}(\mathcal{X}_i)>0 \ \forall \mathcal{X}_i\in \mathcal{X}$, the directions considered are restricted to only include directions where $m^{(n)}(\mathcal{X}_i)\le 0$. Note that this restriction ensures that if $\vec{\alpha}^{(n)}$ satisfies $\vec{y}_i\vec{\alpha}_i^{(n)}\geq 0\ \forall i$, then $\vec{\alpha}^{(n+1)}$ satisfies  $\vec{y}_i\vec{\alpha}_i^{(n+1)}\geq 0\ \forall i$. Initializing $\vec{\alpha}_i^{(0)}=0\ \forall i$ would allow these conditions to be satisified during the updates. Maximizing the decrease in loss given the negative margin restriction shows that greedy coordinate descent selects to update $\vec{\alpha}_i^{(n)}$ according to:
\begin{align}
i &= \argmax_{j:m^{(n)}(\mathcal{X}_j)\le 0} \left({\frac{1}{2}\Delta\vec{\alpha}_j^{(n)}}^2\right) \\
 &= \argmax_{j:m^{(n)}(\mathcal{X}_j)\le 0} \left(\left(\vec{y}_j - \vec{F}_j^{(n)}\right)^2\right) \\
&= \argmax_{j:m^{(n)}(\mathcal{X}_j)\le 0} \left(1 - 2m^{(n)}(\mathcal{X}_j) + (m^{(n)}(\mathcal{X}_j))^2\right) \\
&= \argmin_{j} \left(m^{(n)}(\mathcal{X}_j)\right) \label{minMargin}
\end{align}

Eq. \ref{minMargin} shows that the training point with the most negative margin (intuitively, the point farthest from the separating hyperplane on the wrong side) is prioritized when updating. The combination of the update rule in Eq. \ref{updateRule} and the descent direction in Eq. \ref{minMargin} guarantees $\mathcal{L}\left(\vec{\alpha}^{(n+1)}\right) \leq \mathcal{L}\left(\vec{\alpha}^{(n)}\right) - \frac{1}{2}$ because ${\Delta\vec{\alpha}_i^{(n)}}^2 = (1-{m^{(n)}(\mathcal{X}_i)})^2 \geq 1$ whenever there exists a misclassified training point, i.e., $m^{(n)}(\mathcal{X}_i)\le 0$. In other words, each iteration guarantees a decrease in loss by at least $\frac{1}{2}$ whenever there are still training points with nonpositive margin.

\begin{claim}
Minimization of $\mathcal{L}(\vec{\alpha})$ with the greedy coordinate descent rule defined in Eq. \ref{updateRule} and \ref{minMargin} will always eventually yield a hypothesis with positive margin for all samples given nonsingular Gram matrix $\vec K$.
\label{claim:convergence}
\end{claim}

\begin{proof}
If $\min_i{m^{(n)}(\mathcal{X}_i)}\leq 0$, the upper bound on the change in loss per descent step is $\sup\left(\mathcal{L}\left(\vec{\alpha}^{(n+1)}\right) - \mathcal{L}\left(\vec{\alpha}^{(n)}\right)\right)= \sup\left(-\frac{1}{2}(1-{m^{(n)}(\mathcal{X}_i)})^2\right)=-\frac{1}{2}$. A lower bound of $\mathcal{L}(\vec \alpha)$ is $\inf\mathcal{L}(\vec \alpha) = \mathcal{L}(\vec{K}^{-1}\vec{y}) = -\frac{1}{2}\vec{y}^\mathsf{T}\vec{K}^{-1}\vec{y}$ for nonsingular $\vec{K}$. The margin is exactly 1 for all samples when $\vec{\alpha} = \vec{K}^{-1}\vec{y}$. A loose upper bound on the number of descent steps required to reach $\mathcal{L}(\vec{K}^{-1}\vec{y})$ from initial loss $\mathcal{L}(\vec{\alpha}^{(0)})$ is $\frac{\mathcal{L}(\vec{K}^{-1}\vec{y})-\mathcal{L}(\vec{\alpha}^{(0)})}{-\frac{1}{2}} = \vec{y}^\mathsf{T}\vec{K}^{-1}\vec{y} + 2\mathcal{L}(\vec{\alpha}^{(0)})$.

If $\min_i{m^{(n)}(\mathcal{X}_i)} > 0$, the hypothesis at iteration $n$ successfully provides a positive margin for all samples.
\end{proof}





Claim \ref{claim:convergence} means the weight update algorithm can terminate once all training samples have positive margin or will otherwise work toward achieving positive margin for all samples. As established in Section \ref{sec:kernel}, $\vec{K}$ is a positive definite matrix if each configuration in $\mathcal{X}$ is unique (which is the case when sampling from a continuous space), implying $\vec{K}$ is nonsingular in practice. While $\vec{y}^{\mathsf{T}}\vec{K}^{-1}\vec{y}$ may be large, the total number of iterations required to satisfy $m(\mathcal{X}_i)>0 \ \forall \mathcal{X}_i \in \mathcal{X}$ is typically much smaller, as will be shown empirically by the short training times in Section \ref{sec:results}. In the case that training takes longer than desired, an iteration limit may be defined for early termination (defined as $iter_{max}$ in Algorithm \ref{alg:fastronUpdate}) at the cost of yielding a classifier with lower accuracy.

\subsection{Conditional Bias Parameter}
As with the standard kernel perceptron, the Fastron model does not have an additive bias term. On the other hand, the kernel SVM model contains a bias term $b \in \mathbb{R}$ in its hypothesis: $f_{SVM}(\vec{x}) = f(\vec{x}) + b$. The benefit of including a bias term is that points with the label $\vec{y}_i = sign(b)$ are more likely to be classified correctly because the model is universally biased toward labeling query points as $sign(b)$.

As false negatives (mistaking $\mathcal{C}_{obs}$ for $\mathcal{C}_{free}$) are more costly than false positives (mistaking $\mathcal{C}_{free}$ for $\mathcal{C}_{obs}$) in the context of collision detection for motion planning, we would like to bias the model toward the $\mathcal{C}_{obs}$ label to err on the side of caution. Rather than trying to learn a bias term, we instead conditionally multiply the target values by a user-selected value during training such that $\mathcal{C}_{obs}$ points are more likely to be classified correctly. More specifically, we adjust the rule defined in Eq. \ref{eqn:updateRule1} to
\begin{equation}
\Delta \vec{\alpha}_i^{(n)} = b_i\vec{y}_i - \vec{F}_i^{(n)}
	\end{equation}
where $b_i=\beta^{0.5(\vec{y}_i+1)}$ (assuming the labels are $\pm 1$) and $\beta \geq 1$ is a user-selected conditional bias parameter. When $\beta>1$, weight updates are larger when correcting for a $\mathcal{C}_{obs}$ configuration compared to a $\mathcal{C}_{free}$ configuration. Larger weights for $\mathcal{C}_{obs}$ points would ultimately influence a larger neighborhood in C-space, thereby padding C-space obstacles and potentially increasing true positive rate.

Fig. \ref{fig:condBias} shows the effect of the conditional bias parameter on true positive rate (TPR) and true negative rate (TNR) (averaged over environments with randomly placed obstacles) for the 2 DOF and 4 DOF robots shown in Fig. \ref{fig:2and4dofEnv}, respectively. Increasing the parameter improves TPR at the cost of TNR, but the effects taper off as $\beta$ gets large.

Note that an upper bound on the number of iterations still exists when including the conditional bias parameter. As the target values are now $b_i\vec{y}_i$ instead of $\vec{y}_i$, we can consider a modified loss function $\mathcal{L}(\alpha) = \frac{1}{2}\vec{\alpha}^\mathsf{T}\vec{K\alpha} - \vec{y}^\mathsf{T}\vec{B\alpha}$, where $\vec{B}$ is a diagonal matrix containing the $b_i$ value for each $\vec{y}_i$. Noting that ${\Delta\vec{\alpha}_i^{(n)}}^2\geq 1$ for $\min_i{m^{(n)}(\mathcal{X}_i)}\leq 0$, $\beta\geq 1$, and $\vec{\alpha}^{(0)} = \vec{0}$, the new upper bound on iterations to achieve positive margin for all samples is $\vec{y}^\mathsf{T}\vec{BK}^{-1}\vec{By}$. Once again, in practice, the total number of iterations to satisfy $m(\mathcal{X}_i)>0\ \forall \mathcal{X}_i\in \mathcal{X}$ is typically significantly smaller.

\subsection{Redundant Support Point Removal}
A support point whose margin would be positive even if it were not in the support set is redundant. Redundant support points should be removed from $\mathcal{S}$ (by setting its corresponding weight in $\vec{\alpha}$ to 0) to promote the sparsity of the model. Redundant support point removal is useful when the workspace obstacles move, causing the collision statuses of the points in $\mathcal{X}$ and the decision boundary to change. Removing outdated support points that no longer contribute to the shifted decision boundary is necessary to reduce the model complexity, allowing computational efficiency in changing environments through model sparsity.

To remove $\mathcal{X}_i$ from $\mathcal{S}$, the weight and training hypothesis vector updates are computed as follows:
\begin{equation}
\vec{\alpha}^{(n+1)} = \vec{\alpha}^{(n)} - \vec{\alpha}_i^{(n)}\vec{e(i)}
\end{equation}
\begin{equation}
\vec{F}^{(n+1)} = \vec{F}^{(n)} - \vec{\alpha}_i^{(n)}\vec{K}\vec{e(i)}
\end{equation}

The resultant margin at point $\mathcal{X}_i$ if it were removed from the support set is $m^{(n+1)}(\mathcal{X}_i) = \vec{y}_i\sum_{j\neq i}\vec{K}_{ij}\vec{\alpha}_j^{(n)}=m^{(n)}(\mathcal{X}_i)-\vec{y}_i\vec{\alpha}_i^{(n)}$. Considering margin to be a measure of how well a point is classified, points are iteratively removed in decreasing order of positive resultant margin until $m^{(n)}(\mathcal{X}_i)-\vec{y}_i\vec{\alpha}_i^{(n)} < 0 \ \forall \mathcal{X}_i \in \mathcal{S}$, i.e., removing an additional support point will cause it to be misclassified. We perform redundant support point removal only after a hypothesis that yields positive margin for all training points is found, thus allowing Claim \ref{claim:convergence} to remain valid.

\begin{figure*}[t]
	\centering
	\begin{subfigure}[b]{0.49\linewidth}
      \includegraphics[width=\linewidth]{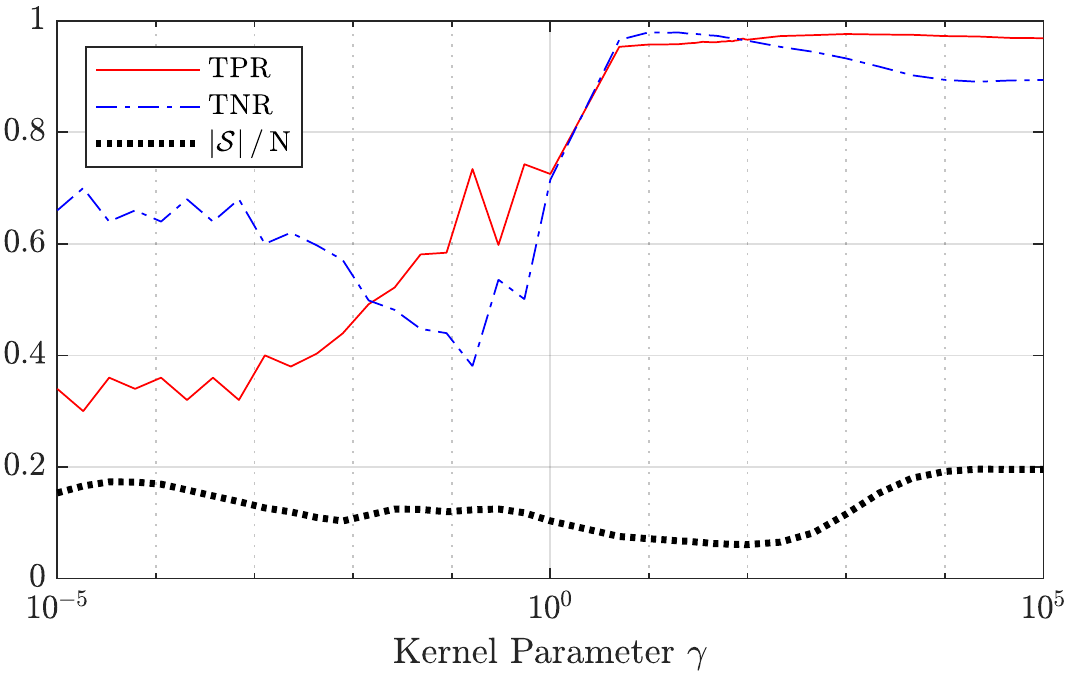}
      \caption{2 DOF, $N = 2000$}
      \label{fig:gamma2dof}
    \end{subfigure}
	\begin{subfigure}[b]{0.49\linewidth}
      \includegraphics[width=\linewidth]{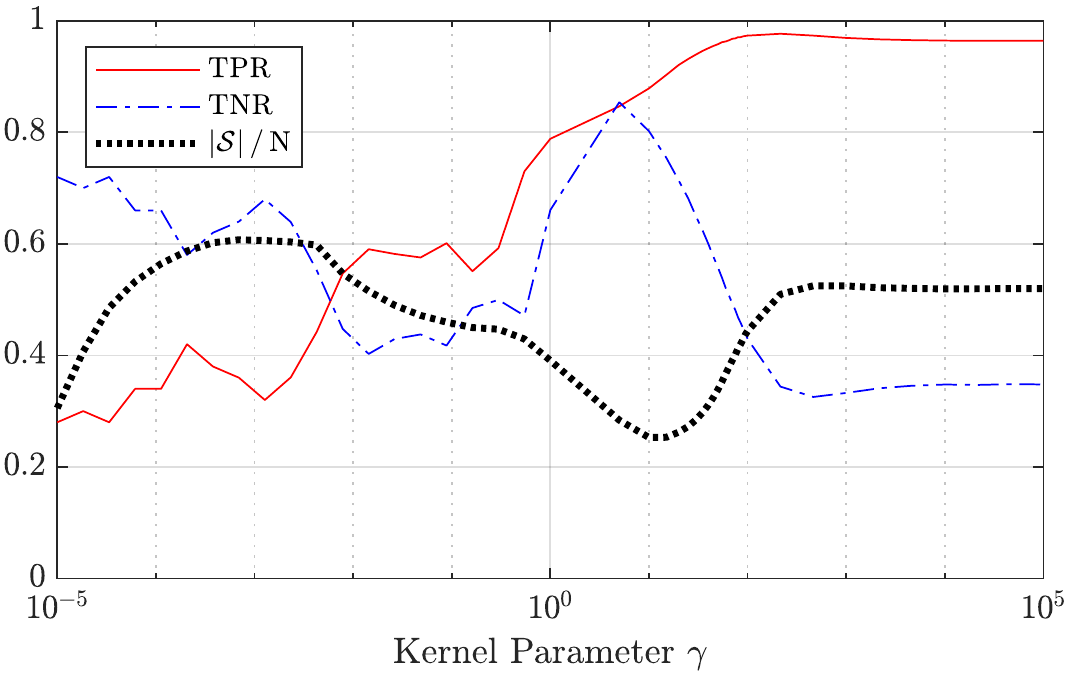}
      \caption{4 DOF, $N = 4000$}
      \label{fig:gamma4dof}
    \end{subfigure}
    \caption{TPR, TNR, and model sizes (as a ratio of number of support points to training set size, $|\mathcal{S}|/N$) of Fastron for various $\gamma$ values for the 2 and 4 DOF robots shown in Fig. \ref{fig:2and4dofEnv}. The ground truth labels are provided using the GJK \cite{Gilbert1988} collision detection method. These curves motivate the choices in $\gamma$ for the 2 DOF and 4 DOF case in which large TPR and TNR are desired while keeping support set sizes as small as possible.}
    \label{fig:gamma}
\end{figure*}

Each redundant support point removal step clearly improves the sparsity of the model by 1. However, following similar steps required to find the change in loss in Eq. \ref{deltaLoss}, the change in loss when removing support point $\mathcal{X}_i\in \mathcal{S}$ is $\mathcal{L}(\vec{\alpha}^{(n)} - \vec{\alpha}_i\vec{e(i)}) - \mathcal{L}(\vec{\alpha}^{(n)}) = \frac{1}{2}\vec{\alpha}_i^2+(b_i\vec{y}_i-\vec{F}_i^{(n)})\vec{\alpha}_i$. This change in loss may be positive or negative depending on the values of $\vec{y}_i$, $\vec{F}_i$, and $\vec{\alpha}_i$, showing that redundant support point removal can potentially step away from the optimal solution. Furthermore, note that redundant support point removal changes the margin at $\mathcal{X}_j$ when removing $\mathcal{X}_i$ from $\mathcal{S}$ as follows: $m^{(n+1)}(\mathcal{X}_j) = m^{(n)}(\mathcal{X}_j) - \vec{y}_j\vec{K}_{ij}\vec{\alpha}_i$. Clearly, the margin decreases for any $\mathcal{X}_j$ where $\vec{y}_j\vec{\alpha}_i>0$, and the number of misclassified training points becomes nonzero if $m^{(n)}(\mathcal{X}_j) < \vec{y}_j\vec{K}_{ij}\vec{\alpha}_i$ for any $\mathcal{X}_j$. If training points become misclassified when a support point is removed, more weight updates are required to correct these points. In practice, a solution is still found quickly as will be shown in the training speed results in Section \ref{sec:results}. As a safeguard, if the iteration limit is reached and the solution obtained before support point removal has fewer misclassifications than after removal, the more correct solution will be returned as shown on Line \ref{removalSafeguard} in Algorithm \ref{alg:fastronUpdate}.


\begin{algorithm}[t!]
\caption{Fastron Active Learning}
\label{alg:activeLearning}

\DontPrintSemicolon
\SetKwInput{Param}{Parameters}
\SetKw{break}{break}
 \KwIn{Support set $\mathcal{S}$}
 \Param{Number of configurations to add to dataset $\mathcal{A}_{max}$, maximum number of exploition samples to generate near each support point $\kappa$, Gaussian variance $\sigma^2$}
 \KwOut{Unlabeled set $\mathcal{A}$}
$\mathcal{A}\leftarrow \emptyset$\\
\tcp{Exploitation}
 \For {$k = 1$ to $\kappa$}
 {
 	\ForEach {$x$ in $S$}
    {
    	\lIf{$|\mathcal{A}| == \mathcal{A}_{max}$}{\break}
        $\mathcal{A} \leftarrow \mathcal{A} \cup \{\vec{x}_{new}\sim\ $\FuncSty{$\mathcal{N}(x, \sigma^2I)\}$ \label{exploitation}}
    }
 }
 \tcp{Exploration} 
 \For {$i = 1$ to $\mathcal{A}_{max} - \kappa|\mathcal{S}|$} {
 	$\mathcal{A} \leftarrow \mathcal{A} \cup \{\vec{x}_{new}\sim\ $\FuncSty{$U(-\mathds{1},\mathds{1})\}$ \label{exploration}}
 }
 
 
 \textbf{return} {$\mathcal{A}$}
\end{algorithm}

\subsection{Active Learning}
When the environment changes, the true decision boundary will change. The labels for the training set must be updated using the collision checker prior to updating the model. Rather than resampling the entire space, which may result in many unnecessary collision checks, we can exploit our previous model to search for new information. Active learning is a methodology to query an oracle for additional information and is intended to reduce the amount of labeling needed to update a model \cite{Schohn2000}. The oracle in this case is the geometry- and kinematics-based collision detection paradigm (KCD).

Our active learning strategy determines a new set of configurations $\mathcal{A}$ to add to the dataset on which to perform collision checks. Collision checks are to be performed on all previous support points $\mathcal{S}$ and $\mathcal{A}$. As active learning takes place after sparsifying the trained model, the number of collision checks that will be performed for the next model update is always $|\mathcal{S}|+\mathcal{A}_{max}$, where the $\mathcal{A}_{max}$ is the user-specified number of configurations to add to $\mathcal{A}$. 

Fastron's active learning strategy is conducted in two stages: exploitation and exploration. In the exploitation stage, up to $\kappa$ Gaussian distributed random samples are generated near each support point and are added to $\mathcal{A}$ as shown on Line \ref{exploitation} in Algorithm \ref{alg:activeLearning}. For simplicity, we sample using isotropic Gaussians with the variance in each direction $\sigma^2={\left(2\gamma\right)}^{-1}$. Up to $\kappa|\mathcal{S}|$ unlabeled configurations are added to $\mathcal{A}$ during the exploitation stage. The idea behind the exploitation stage is to search for small changes in or improvements to C-space obstacles such as when workspace obstacles move an incremental amount or the C-space obstacle boundary may be defined more precisely.

If the number of points in $\mathcal{A}$ is less than $\mathcal{A}_{max}$, the exploration stage fills the rest of $\mathcal{A}$ with uniformly random samples generated in the C-space as shown on Line \ref{exploration} in Algorithm \ref{alg:activeLearning}. For certain robots such as manipulators, newly introduced workspace obstacles can cause new C-space obstacles to materialize in difficult-to-predict locations. The purpose of the random exploration stage is thus to search for these new or drastically different C-space obstacles.

\subsection{Practical Considerations}
\subsubsection{Algorithm Pipeline}
The block diagram in Fig. \ref{fig:blockDiagram} shows the sequence of steps for the algorithm. Initially, a uniformly random set of unlabeled configurations $\mathcal{X}$ are generated and the KCD labels each configuration in $\mathcal{X}$. The KCD requires a snapshot of the obstacles in the current workspace and knowledge of the robot's kinematics and geometry. We regard the KCD as a black box, fully encompassing the entire collision detection cycle, including forward kinematics to locate robot geometry, sensor readings to obtain obstacle information, and tests for intersection of the robot's links with obstacles.

The labeled dataset $\mathcal{X}$ is used to update the Fastron model. $\mathrm{loadPreviousModel}()$ in Algorithm \ref{alg:fastronUpdate} loads the weight vector, hypothesis vector, and partially filled Gram matrix from the previous update, or initializes all elements to $0$ if a previous model does not exist. After the support set $\mathcal{S}$ and the weights $\alpha$ are determined, non-support points are discarded and the model is ready to be used for proxy collision detection. $\mathrm{removeNonsupportPoints}()$ in Algorithm \ref{alg:fastronUpdate} removes all elements in the weight vector, hypothesis vector, dataset, and Gram matrix that correspond to non-support points. Active learning augments the previous support set with $|\mathcal{A}|$ unlabeled configurations. This new dataset is then fed into the KCD before the cycle repeats.

We expect this entire pipeline to run in parallel with all other processes, and the most recent Fastron model is used for proxy collision detection.

\subsubsection{Joint Limits}
This algorithm is intended to be used with robots with joint limits so that there are bounds within which C-space samples may be generated. As we work with isotropic kernels, configurations must be mapped such that the kernels affect the same proportion of each DOF. In this paper, we choose to map the bounded $d$-dimensional joint space to a $d$-dimensional input space $[-1,1]^d$. We choose $[-1,1]^d$ as the input space because if the upper and lower limits for any joint are symmetric about $0$, a joint position of $0$ still maps to $0$ in input space. If $\vec{q}_u$ and $\vec{q}_l$ are vectors containing the upper and lower joint limits, respectively, then the following formula can be applied to map a joint space configuration $\vec{q}_{joint}$ to an input space point $\vec{q}_{input}$:
\begin{equation}
\vec{q}_{input} = (2\vec{q}_{joint}-\vec{q}_u-\vec{q}_l)\oslash(\vec{q}_u-\vec{q}_l)
\end{equation}
where $\oslash$ performs element-wise division.

\subsubsection{Support Point Cap}
Whenever a weight update increases the number of nonzero elements in $\vec{\alpha}$, another point is added to the support set $\mathcal{S}$. To limit the computation time during classification, a support point limit $\mathcal{S}_{max}$ may be defined to prevent $|\mathcal{S}|$ from growing too large.

During the update steps, if the worst margin occurs at $\mathcal{X}_i\in \mathcal{S}$, the weight update can proceed without consideration of $\mathcal{S}_{max}$ because $|\mathcal{S}|$ will not increase when adjusting the weight of a preexisting support point. On the other hand, if the worst margin occurs at $\mathcal{X}_i\not \in \mathcal{S}$, $\mathcal{X}_i$ is added to $\mathcal{S}$ only if $|\mathcal{S}|<\mathcal{S}_{max}$. Otherwise, if $|\mathcal{S}|=\mathcal{S}_{max}$, there is an attempt to remove a redundant support point before continuing to train. If no point can be removed, training is terminated.

\subsubsection{Lazy Gram Matrix Evaluation}
An advantage of prioritizing the most negative margin rather than sequentially adjusting the weight for each point in $\mathcal{X}$ is that not all values of $\vec{K}$ are utilized. The Gram matrix may thus be evaluated lazily. For example, if the most negative margin occurs at point $\mathcal{X}_i$, then only the $i^{th}$ column of $\vec{K}$ is required for the update. Once the $i^{th}$ column of $\vec{K}$ is computed, it does not need to be recomputed if $\vec{\alpha}_i$ needs to be adjusted in a later training iteration. $\mathrm{computeGramMatrixColumn}()$ in Algorithm \ref{alg:fastronUpdate} computes the $i^{th}$ column of the Gram matrix if it has not been computed yet.

\subsubsection{Dynamic Data Structures}
To avoid training from scratch every time collision status labels are updated, some data structures must be retained and updated throughout the lifetime of the algorithm. The training dataset of configurations $\mathcal{X}$ and Gram matrix $\mathbf{K}$ are stored as two-dimensional arrays, while the weight vector $\vec{\alpha}$, hypothesis vector $\vec{F}$, and true collision status labels $\vec{y}$ are stored as one-dimensional arrays. The dimensionality of each of these data structures always depends on how many points are currently in $\mathcal{X}$.

After training is complete, non-support points are discarded from $\mathcal{X}$, and all other data structures are shrunk accordingly. All sparsification happens on Line \ref{removeNonsupportPoints} in Algorithm \ref{alg:fastronUpdate}. After active learning, $\mathcal{A}_{max}$ new points are added to $\mathcal{X}$, and all data structures must be resized accordingly. If repeated deallocation and reallocation of memory slows performance, a fixed amount of memory can be reserved for each data structure based on $\mathcal{S}_{max}$.

Typically the values in $\vec{F}$ are determined incrementally. However, elements in $\vec{F}$ corresponding to the $\mathcal{A}_{max}$ new points from active learning cannot be set the same way. Instead, the columns in $\vec{K}$ that have been partially filled through lazy Gram matrix evaluation should first be completely filled. Next, the uninitialized elements of $\vec{F}$ should be calculated directly using the updated Gram matrix and nonzero weights in $\vec{\alpha}$.

\section{Experimental Results}
\label{sec:results}

\subsection{Performance in Static Environment}
\subsubsection{Descriptions of Alternative Methods}\label{sec:altMethods}
We compare the performance of decision boundary machine learning techniques for static environments. We compare Fastron with two state-of-the-art, kernel-based algorithms: incremental SVM \cite{Karasuyama2010} with active learning \cite{Pan2015} and sparse SVM \cite{Huang2010}, which we refer to as ISVM and SSVM, respectively. As with the Fastron model, we use $[-1,1]^d$ as the input space and $\pm 1$ as the labels for both ISVM and SSVM. We train and validate all methods with the GJK algorithm \cite{Gilbert1988}, a standard for collision detection for convex polyhedra.

\begin{algorithm}[t]
\caption{Incremental SVM with Active Learning Method (Adapted from Original \cite{Pan2015})}
\label{alg:isvm}

\DontPrintSemicolon
\SetKwInput{Param}{Parameters}
\SetKw{break}{break}
\SetKw{and}{AND}
\SetKw{true}{TRUE}
\SetKw{false}{FALSE}
 \KwIn{Initial training dataset of configurations $\mathcal{X}_{init}$, initial collision status labels $\vec{y}_{init}$}
 \Param{Kernel parameter $\gamma$ (for SVM), regularization parameter $C$ (for SVM), amount to change exploitation/exploration probability threshold $\Delta P_{explore}$, number of collision checks to perform per active learning update $N_\mathcal{A}$}
 \KwOut{Weight vector $\vec{\alpha} \in \mathbb{R}^{N+1}$ (where first element is model bias), updated dataset $\mathcal{X}$}
 \tcp{Train initial SVM model}
  $\mathcal{X} \leftarrow \mathcal{X}_{init}$\\
  $\vec{\alpha} \leftarrow \mathrm{incSvmUpdate}(\mathcal{X}_{init},\vec{y}_{init})$ \\
  \tcp{Generate up to $\mathcal{A}_{max}$ new samples}
  \While{$|\mathcal{X}\setminus \mathcal{X}_{init}|<\mathcal{A}_{max}$}
  {
  	\tcp{Active Learning}
  	$\mathcal{A} \leftarrow \emptyset$ \\
  	$explore \leftarrow \mathrm{rand}() < P_{explore}$ \\
  	\If{$explore$}
    {
    	 \For {$i = 1$ to $N_\mathcal{A}$}
        {
         	$\mathcal{A} \leftarrow \mathcal{A} \cup \{\vec{x}_{new}\sim\ $\FuncSty{$U(-\mathds{1},\mathds{1})\}$}
        }
    }
    \Else
    {
    	\For {$i = 1$ to $N_\mathcal{A}$}
        {
        	$\mathcal{X}_{sel} \leftarrow \mathrm{randPick}(\mathcal{S})$ \\
            $\mathcal{X}_{near} \leftarrow \argmin_{\mathcal{X}_j}{|\mathcal{X}_{sel} - \mathcal{X}_j|}$ \qquad\qquad\qquad s.t. $\mathcal{X}_j\in \mathcal{S} \land \vec{y}_j \neq \vec{y}_{sel}$\\
            $\vec{x}_{new} \leftarrow \argmax_{\vec{x}}{k(\vec{x},\mathcal{X}_{sel}) + k(\vec{x},\mathcal{X}_{near})}$ \\
            $\mathcal{A} \leftarrow \mathcal{A} \cup \{\vec{x}_{new}\}$
        }
    }
    $\mathcal{X} \leftarrow \mathcal{X} \cup \mathcal{A}$\\
  	$\vec{y}_{\mathcal{A}}\leftarrow \mathrm{colCheck}(\mathcal{A})$ \\

    \tcp{Update SVM with new information}
    $\vec{\alpha} \leftarrow \mathrm{incSvmUpdate}(\mathcal{A},\vec{y}_{\mathcal{A}})$ \\

    \tcp{Increase $P_{explore}$ if exploration samples are poorly classified}
    \If{$explore$ \and $\mathrm{errRate}(\mathcal{A})>err_{thres}$}
    { $P_{explore}\pluseq\Delta P_{explore}$ }
    \Else { $P_{explore}\minuseq\Delta P_{explore}$ }

    \tcp{Exit if all new samples are correctly classified}
    \If {$\mathrm{errRate}(\mathcal{X}\setminus \mathcal{X}_{init})==0$} { \break }
  }
  \Return{$\vec{\alpha}, \mathcal{X}$}
\end{algorithm}

Pan et al. use ISVM to create an accurate C-space model \cite{Pan2015,Pan2013}. Starting by fitting an SVM model to a small set of labeled configurations, Pan et al. iteratively improve upon this initial model using an active learning approach which randomly selects either exploration sampling or exploitation sampling when seeking new points to add to the training set. Exploration requires generating uniformly random samples, while exploitation generates new points between support points of opposite labels. Note that the purpose for Pan et al.'s active learning strategy is to improve the accuracy of the model, while Fastron's active learning strategy is meant to search for changes in C-space when the workspace changes. 

While SSVM has not yet been applied to C-space approximation, we include SSVM for comparison because it attempts to minimize the $L_0$-norm of the weight vector, i.e., the number of support points. $L_0$-norm minimization is advantageous because classification time is directly dependent on the number of support points in the model. SSVM approximates the $L_0$-norm by using a weighted $L_2$-norm of model weights $\vec{\alpha}$, $\vec{\alpha}^\textsf{T}\vec{\Lambda\alpha}$, where $\vec{\Lambda}$ is a diagonal matrix. The elements of $\vec{\Lambda}$ are set to $\vec{\alpha}_i^{-2}$ for $|\vec{\alpha}_i|>\epsilon$ and to $\epsilon^{-2}$ otherwise, where $\epsilon$ is a small positive value. Including this approximated $L_0$-norm into the objective function forces many elements of $\vec{\alpha}$ to approach 0. Points whose $\vec{\alpha}_i$ approach 0 are removed from the training set before repeating the optimization.

\begin{algorithm}[t]
\caption{Sparse SVM Method (Adapted from Original \cite{Huang2010})}
\label{alg:ssvm}

\DontPrintSemicolon
\SetKwInput{Param}{Parameters}
\SetKw{break}{break}
\SetKw{and}{AND}
\SetKw{true}{TRUE}
\SetKw{false}{FALSE}
 \KwIn{Training dataset of configurations $\mathcal{X}$, collision status labels $\vec y$}
 \Param{Kernel parameter $\gamma$ (for SVM), regularization parameter $C$ (for SVM), maximum number of iterations $iter_{max}$ }
 \KwOut{Weight vector $\vec\alpha \in \mathbb{R}^{N+1}$ (where first element is model bias)}
 \tcp{Initialize weights and Gram matrix}
 $\vec\alpha \leftarrow \mathds{1}$ \\
 $\vec{K}_{ij} \leftarrow k(\mathcal{X}_i,\mathcal{X}_j)\ \forall\ \mathcal{X}_i,\mathcal{X}_j \in \mathcal{X}$\\
 $\bar{\vec K} \leftarrow [\ \mathds{1}\ \vec{K}\ ]^\mathsf{T}\mathrm{diag}(\vec{y})$ \\
 \For{$iter = 1$ to $iter_{max}$}
 {
 	\tcp{Perform optimization}
 	$\vec A \leftarrow \mathrm{diag}(\vec \alpha)$, 
 	$\vec H \leftarrow \bar{\mathbf{K}}^\mathsf{T}\vec A^2\bar{\mathbf{K}}$\\
    $\vec \beta \leftarrow \argmin_{b}{\frac{1}{2}\vec b^\mathsf{T}\vec{Hb} - \mathds{1}^\mathsf{T}\vec b}\ $ s.t. $0\leq \vec{b}_i \leq C\ \forall i$ \\
    \tcp{Obtain new weights \vec{$\alpha$}}
    $\vec{\alpha_{prev}} \leftarrow \vec \alpha$ \\
    $\vec \alpha \leftarrow \vec A^2\bar{\vec{K}}\vec\beta$, 
    $\vec\alpha_i \leftarrow 0\ \forall |\vec\alpha_i| < \epsilon$ \\
    \tcp{Exit if no change in model size}
    \If {$\|\vec{\alpha}\|_0==\|\vec{\alpha_{prev}}\|_0$}{\break}
 }
 \Return{$\vec \alpha$}
\end{algorithm}

\subsubsection{Description of Experiment}
For simplicity, in this set of tests we implemented all algorithms in MATLAB and assume that each algorithm's standing will remain the same when transferred to a compiled language. For the training portion of ISVM, we use Diehl et al.'s MATLAB implementation of the algorithm \cite{Diehl}. Furthermore, note that Pan et al. \cite{Pan2015} trains on pairs of objects, which will require multiple models when working with robot arms. As we are interested in comparing the performance of single models, we instead train one incremental SVM model on collision status labels generated for the entire robot arm. For training the SSVM, we use MATLAB's $\mathrm{quadprog}()$ to solve the more efficient dual problem described by Huang et al. \cite{Huang2010}. To further improve the training time of SSVM, we terminate training when the number of support points stops decreasing. In our experience, other than training time, no other metric seemed largely affected compared to the stopping rule provided by Huang et al. (i.e., terminate when the change in $\vec{\alpha}$ is less than $10^{-4}$). Our interpretations of the training and active learning of ISVM and training of SSVM are provided in Algorithms \ref{alg:isvm} and \ref{alg:ssvm}, respectively. For Algorithm \ref{alg:isvm}, $\mathrm{incSvmUpdate}()$ performs the incremental SVM update according to Diehl et al. \cite{Diehl}, $\mathrm{rand}()$ generates a random number between 0 and 1, $\mathrm{randPick}()$ selects a sample from a set, $\mathrm{colCheck}()$ performs geometry- and kinematics-based collision checks, and $\mathrm{errRate}()$ determines the proportion of misclassifications of a given set.

Ground truth collision checking is performed using the GJK algorithm \cite{Gilbert1988}. We work with cube obstacles and fit the tightest possible bounding box around each link of the robot. Robot kinematics and visualization are performed using Peter Corke's MATLAB Robotics Toolbox \cite{Corke2017}.

We begin with a simple 2 DOF robot whose body is a single rod whose yaw and pitch may be controlled. Up to 4 obstacles are randomly placed in the environment. Fig. \ref{fig:2dofEnv} shows an example test environment. We repeat the experiments using a 4 DOF robot, which is created by concatenating two of the 2 DOF robots as shown in Fig. \ref{fig:4dofEnv}.

Fig. \ref{fig:gamma} provides TPR, TNR, and model sizes (metrics demonstrating classification accuracy and speed) for Fastron for various $\gamma$ values in randomly generated environments to motivate our choices for the $\gamma$ values. We select $\gamma = 30$ for the 2 DOF case and  $\gamma = 10$ for the 4 DOF case as these choices seem to provide an adequate balance between TPR/TNR and model size. A smaller $\gamma$ for the 4 DOF case makes sense as a wider kernel is needed to account for the larger C-space. These choices of $\gamma$ values also worked well for ISVM and SSVM.

The rational quadratic kernel is used for Fastron and SSVM, and the Gaussian kernel is used for ISVM as Pan et al. use \cite{Pan2015}. Both Fastron and SSVM are trained on a uniformly random set of samples in the input space. The incremental SVM uses the active learning approach to build its training set \cite{Pan2015}. Each algorithm is ultimately trained on $N=2000$ samples for the 2 DOF case, and $N=4000$ samples for the 4 DOF case. The 4 DOF case is undersampled because larger dataset sizes become computationally intractable for the comparison methods using our MATLAB implementations.

\begin{figure*}
    \centering
    \begin{subfigure}[b]{0.195\linewidth}
      \includegraphics[width=\linewidth,trim = {1.1cm 0.9cm 1.1cm 1.3cm},clip]{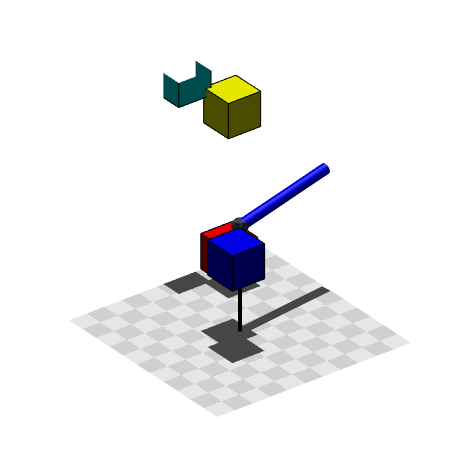}
    \end{subfigure}
    \begin{subfigure}[b]{0.195\linewidth}
      \includegraphics[width=\linewidth]{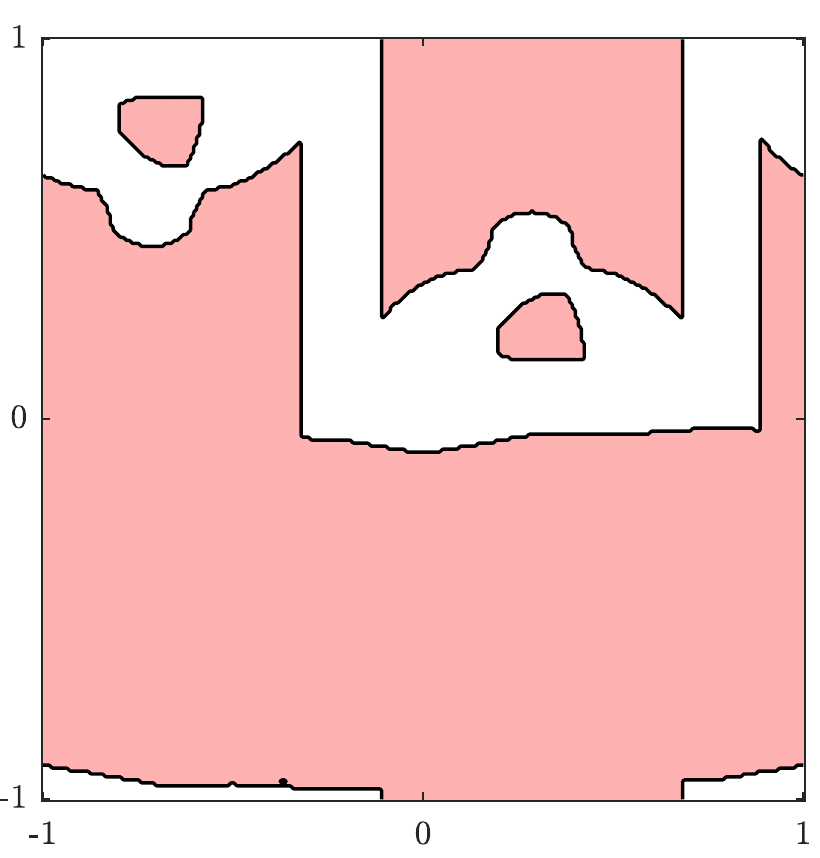}
    \end{subfigure}
	\begin{subfigure}[b]{0.195\linewidth}
      \includegraphics[width=\linewidth]{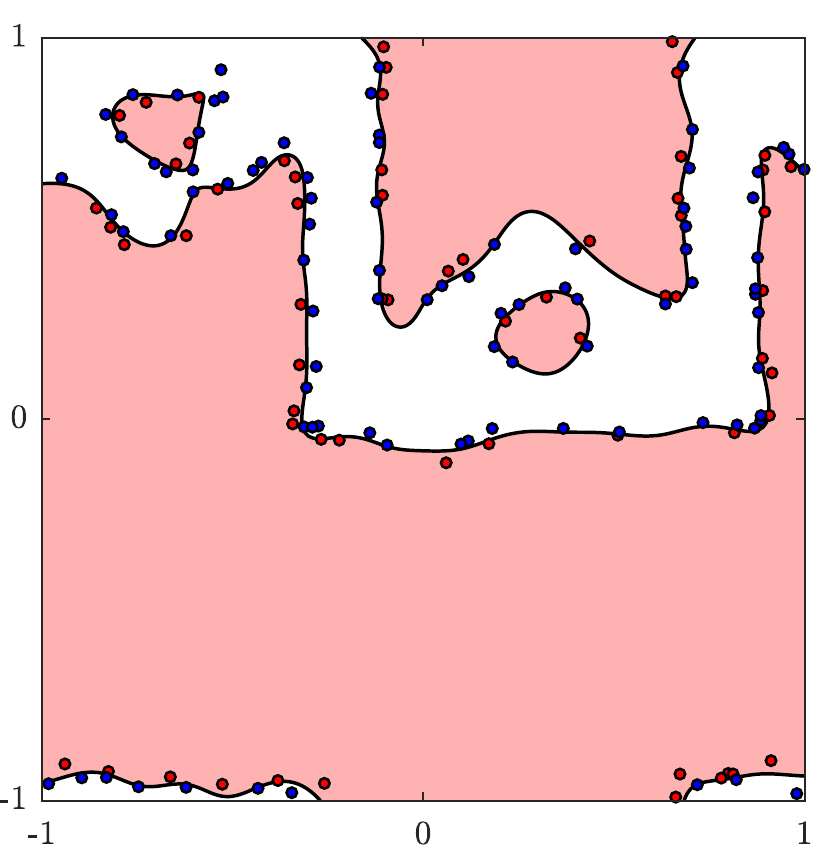}
    \end{subfigure}
    \begin{subfigure}[b]{0.195\linewidth}
      \includegraphics[width=\linewidth]{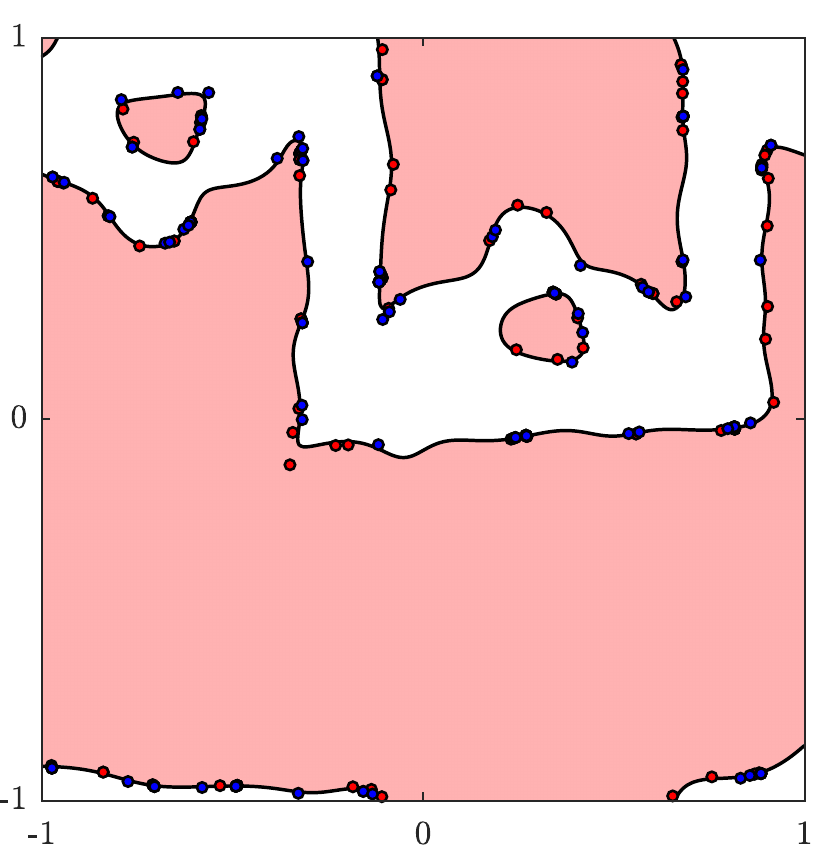}
    \end{subfigure}
    \begin{subfigure}[b]{0.195\linewidth}
      \includegraphics[width=\linewidth]{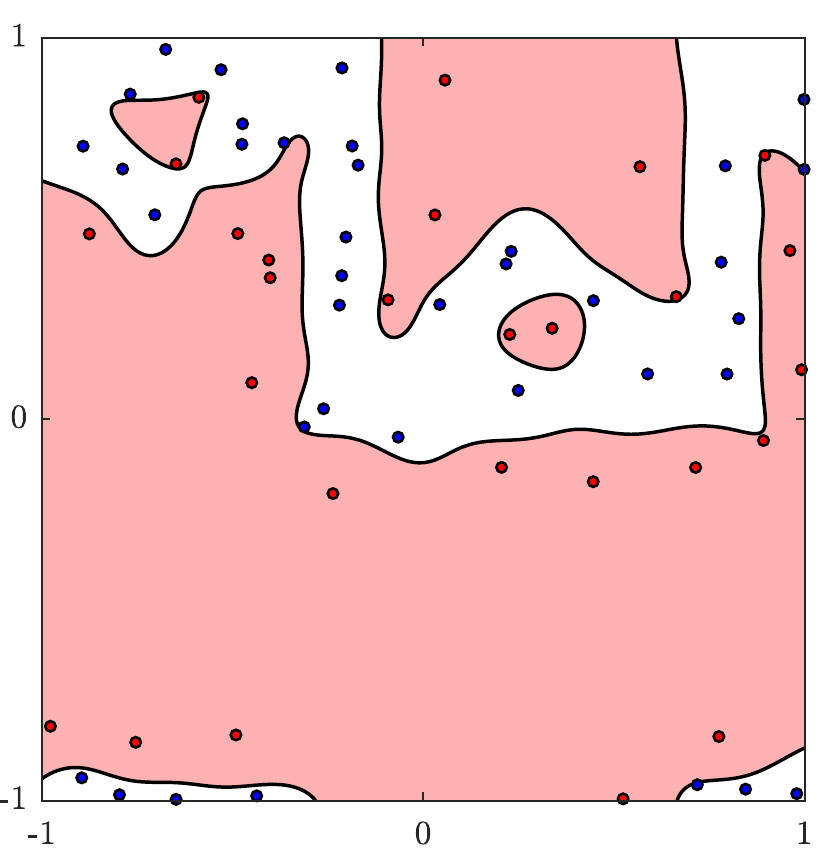}
    \end{subfigure}
    
    \vspace{32pt}
    
    \begin{subfigure}[b]{0.195\linewidth}
      \includegraphics[width=\linewidth,trim = {1.1cm 0.9cm 1.1cm 1.3cm},clip]{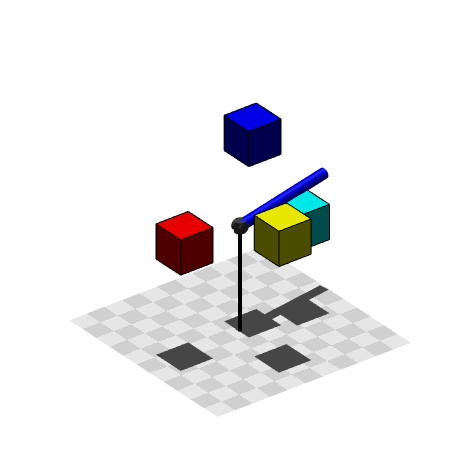}
    \end{subfigure}
        \hfill
    \begin{subfigure}[b]{0.195\linewidth}
      \includegraphics[width=\linewidth]{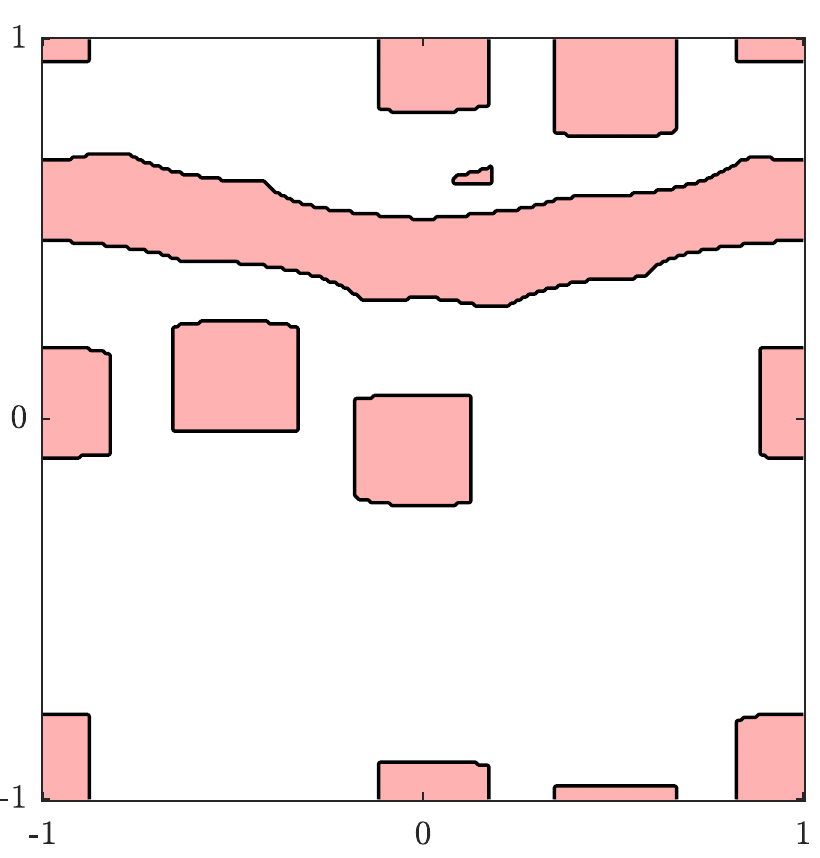}
    \end{subfigure}
        \hfill
	\begin{subfigure}[b]{0.195\linewidth}
      \includegraphics[width=\linewidth]{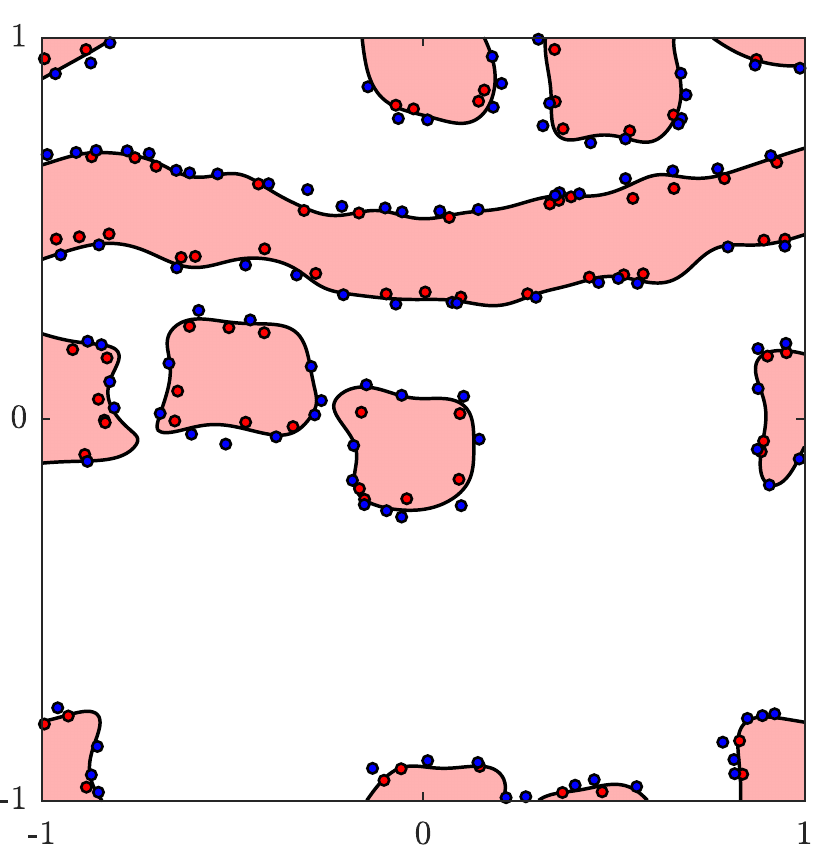}
    \end{subfigure}
        \hfill
    \begin{subfigure}[b]{0.195\linewidth}
      \includegraphics[width=\linewidth]{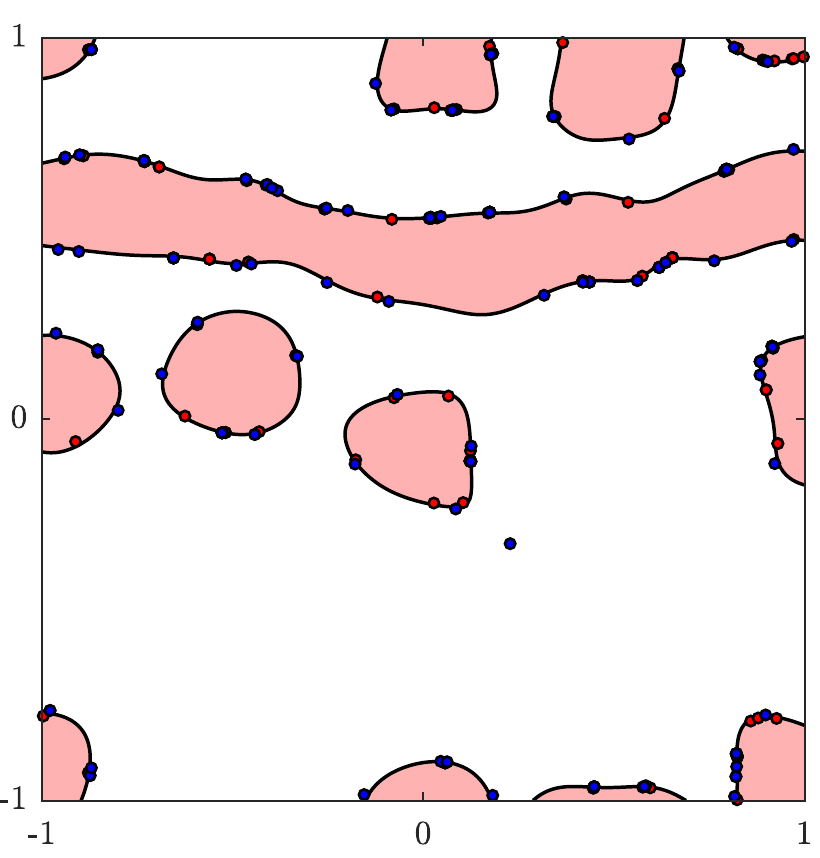}
    \end{subfigure}
        \hfill
    \begin{subfigure}[b]{0.195\linewidth}
      \includegraphics[width=\linewidth]{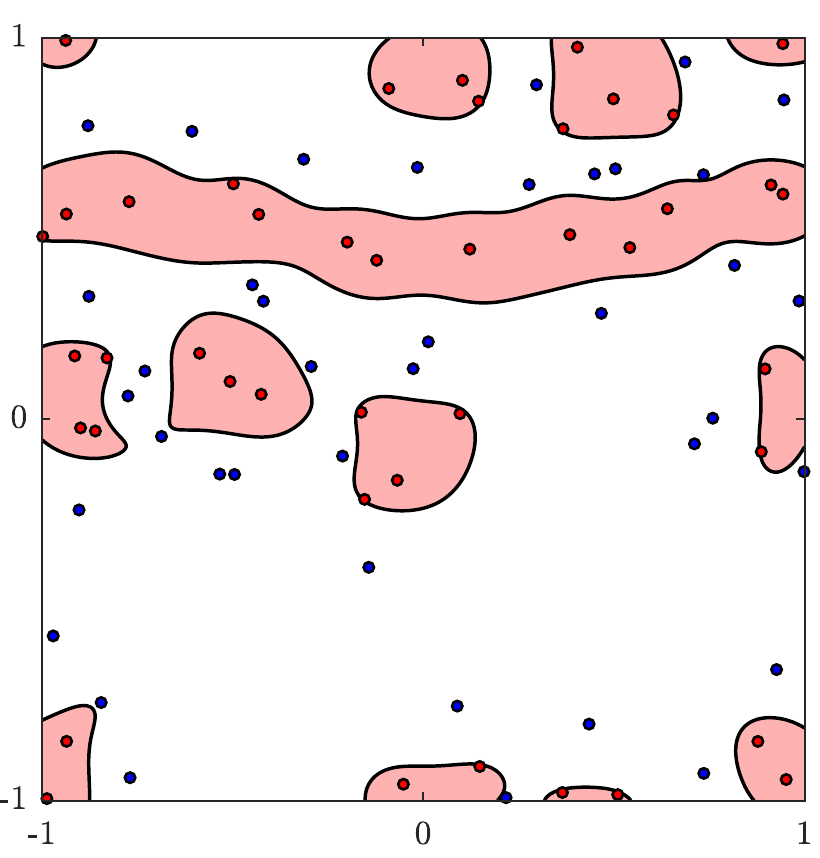}
    \end{subfigure}
    
    \vspace{32pt}
    
        \begin{subfigure}[b]{0.195\linewidth}
      \includegraphics[width=\linewidth,trim = {1.1cm 0.9cm 1.1cm 1.3cm},clip]{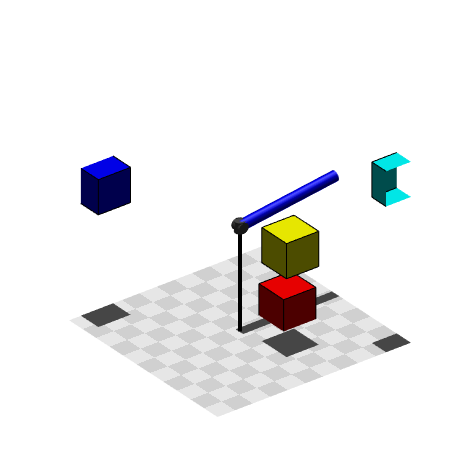}
    \end{subfigure}    \hfill
    \begin{subfigure}[b]{0.195\linewidth}
      \includegraphics[width=\linewidth]{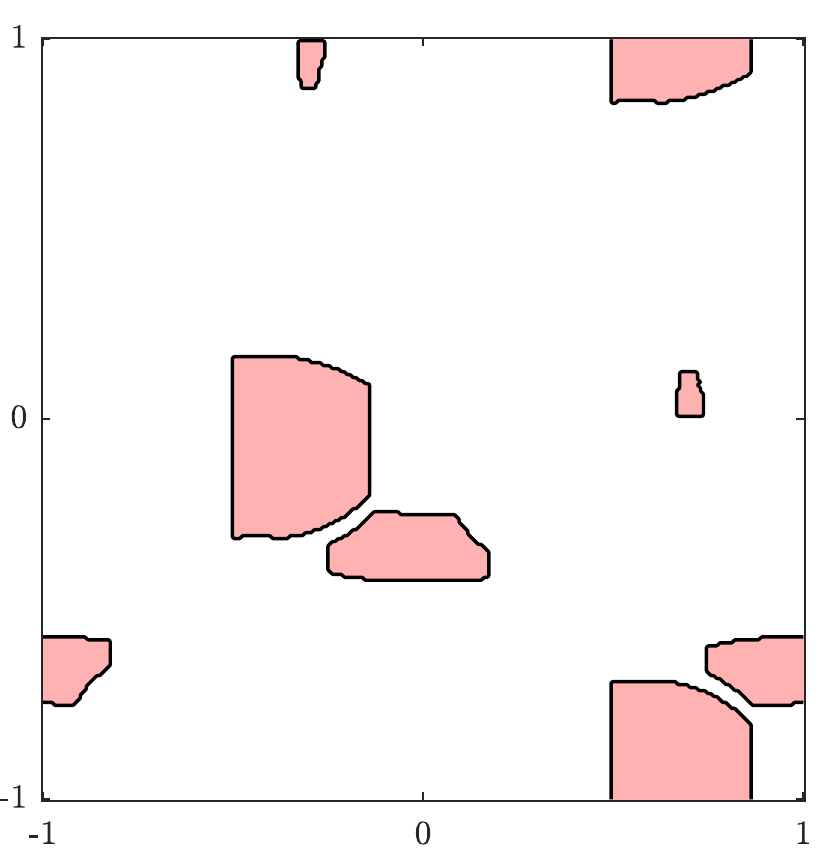}
    \end{subfigure}    \hfill
	\begin{subfigure}[b]{0.195\linewidth}
      \includegraphics[width=\linewidth]{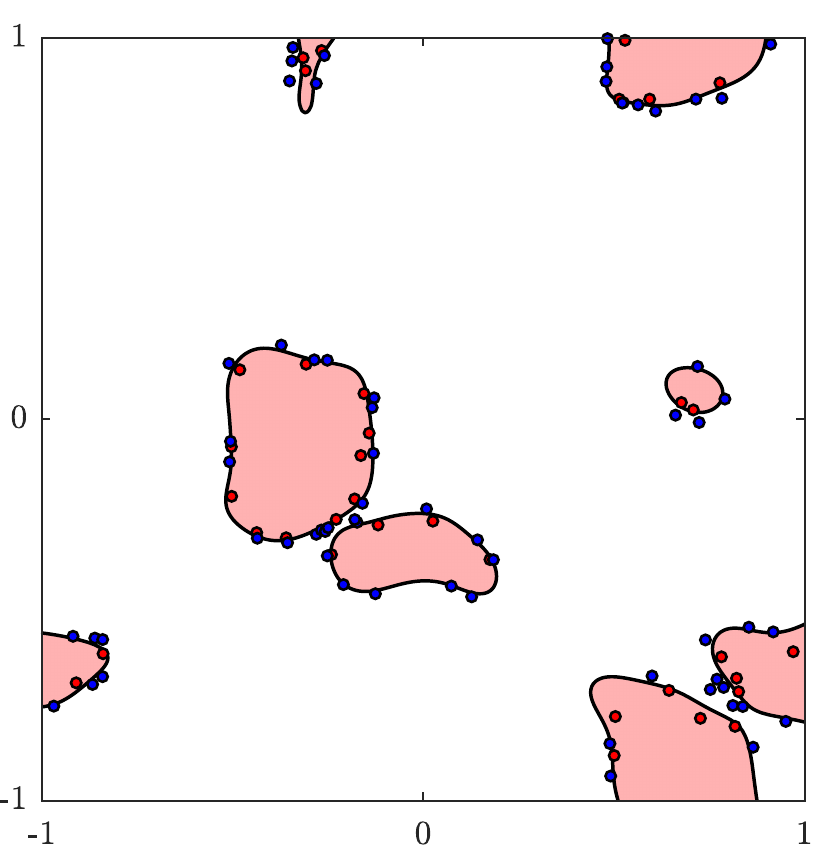}
    \end{subfigure}    \hfill
    \begin{subfigure}[b]{0.195\linewidth}
      \includegraphics[width=\linewidth]{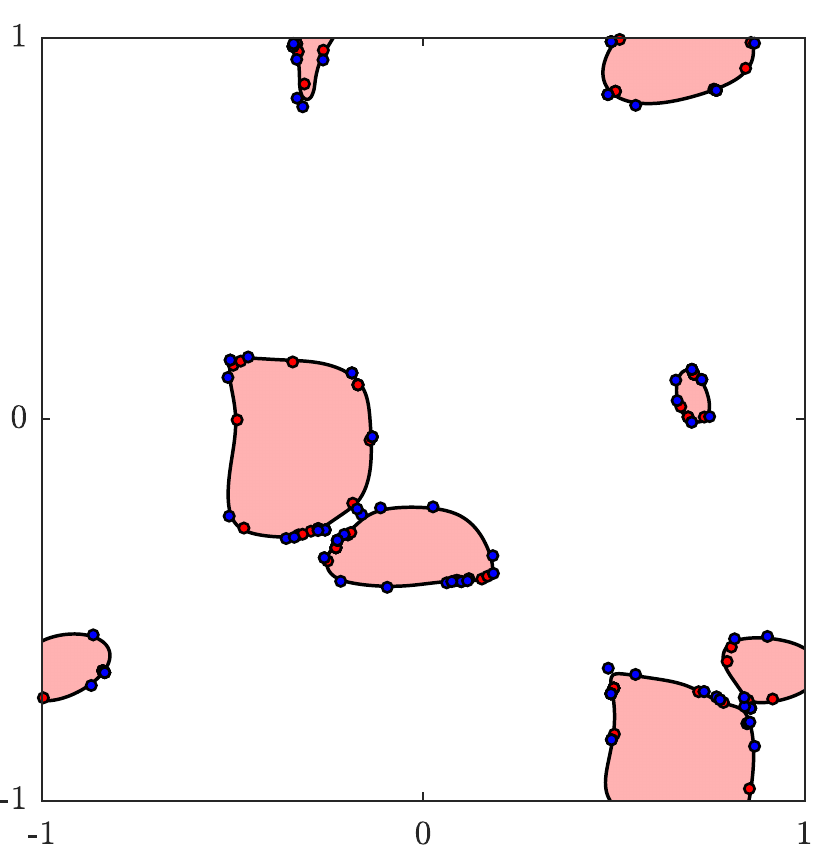}
    \end{subfigure}    \hfill
    \begin{subfigure}[b]{0.195\linewidth}
      \includegraphics[width=\linewidth]{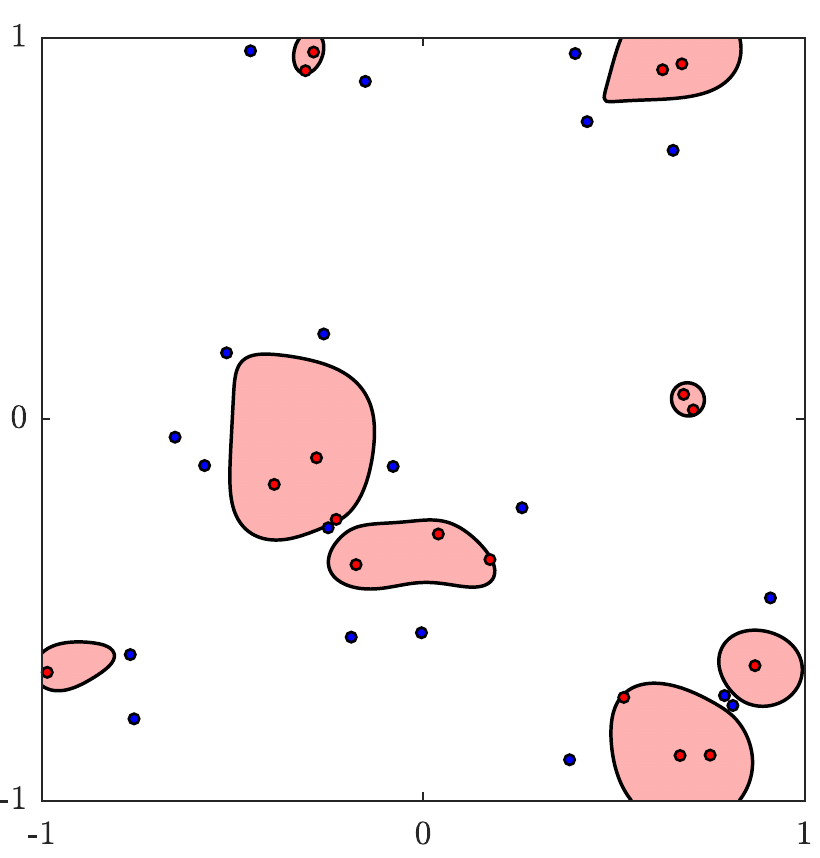}
    \end{subfigure}

    \vspace{32pt}
    
     \begin{subfigure}[b]{0.195\linewidth}
      \includegraphics[width=\linewidth,trim = {1.1cm 0.9cm 1.1cm 1.3cm},clip]{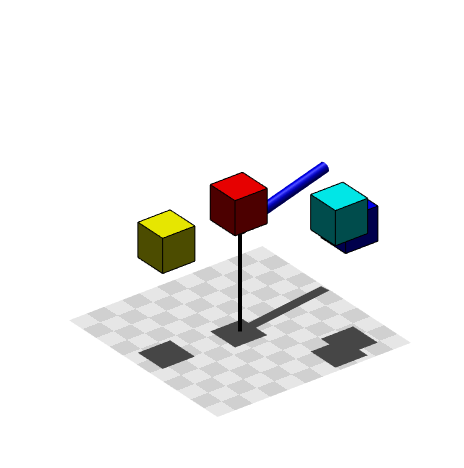}
    \end{subfigure}    \hfill
    \begin{subfigure}[b]{0.195\linewidth}
      \includegraphics[width=\linewidth]{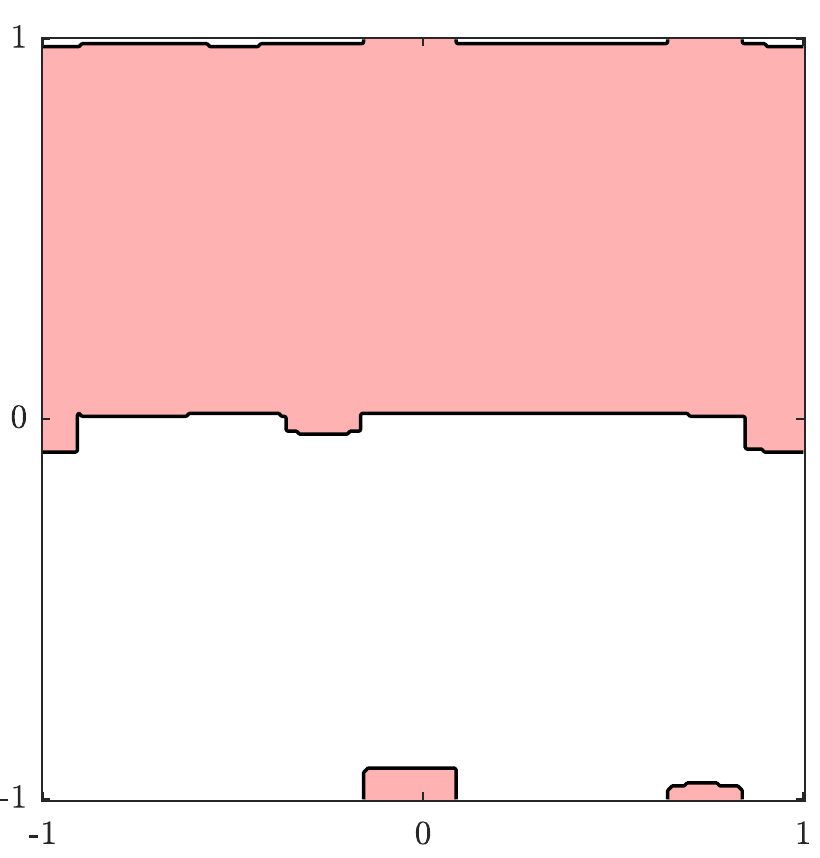}
    \end{subfigure}    \hfill
	\begin{subfigure}[b]{0.195\linewidth}
      \includegraphics[width=\linewidth]{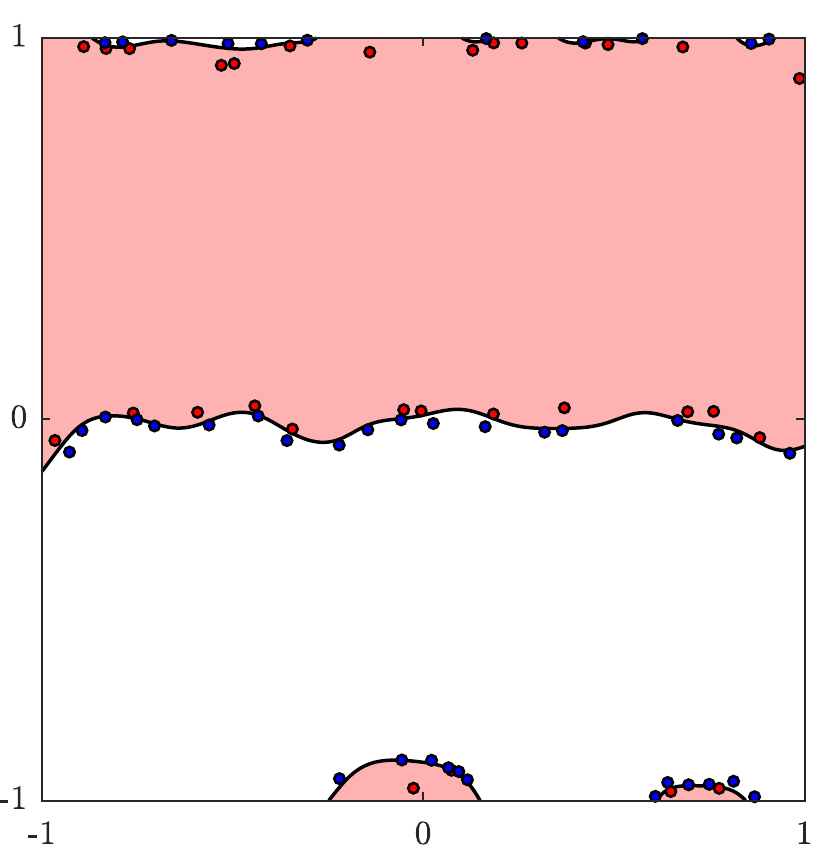}
    \end{subfigure}    \hfill
    \begin{subfigure}[b]{0.195\linewidth}
      \includegraphics[width=\linewidth]{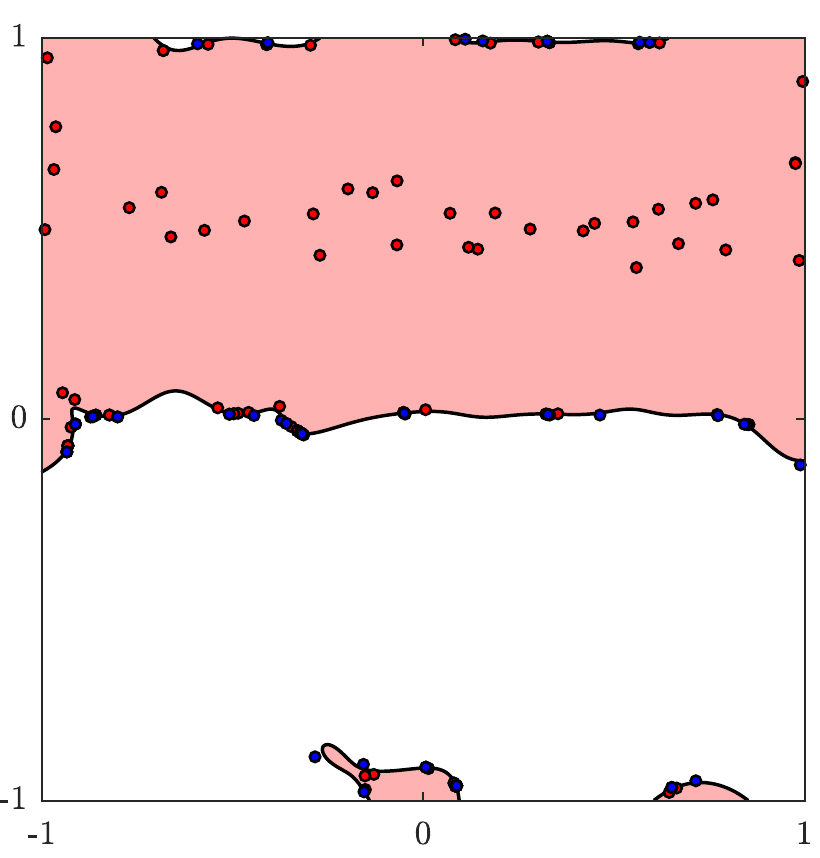}
    \end{subfigure}    \hfill
    \begin{subfigure}[b]{0.195\linewidth}
      \includegraphics[width=\linewidth]{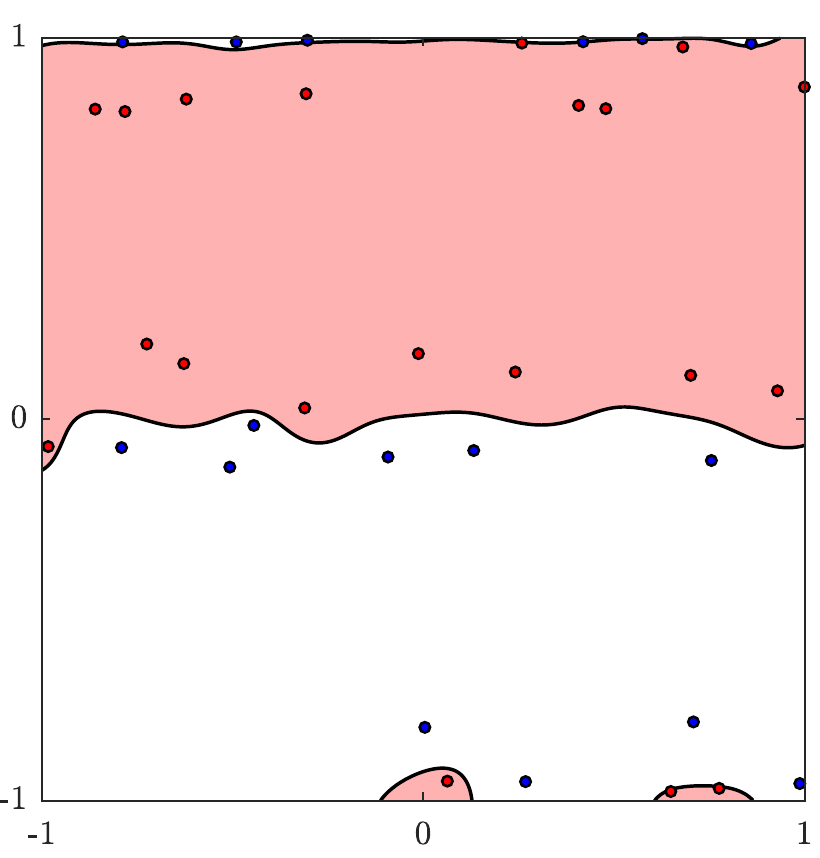}
    \end{subfigure}

    \vspace{32pt}
    
	\begin{subfigure}[b]{0.195\linewidth}
      \includegraphics[width=\linewidth,trim = {1.1cm 0.9cm 1.1cm 1.3cm},clip]{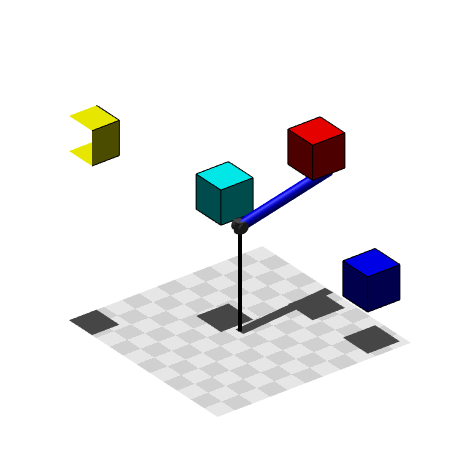}
      \caption{Workspace}
      \label{fig:methodsComparisonWorkspace}
    \end{subfigure}    \hfill
    \begin{subfigure}[b]{0.195\linewidth}
      \includegraphics[width=\linewidth]{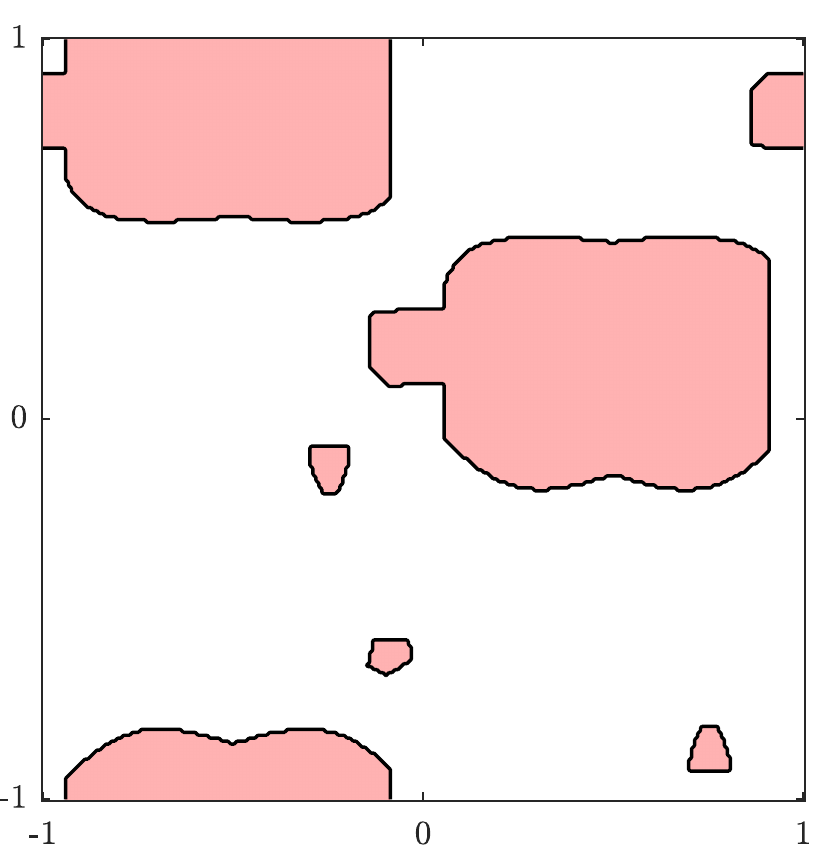}
      \caption{FK + GJK}
      \label{fig:methodsComparisonGT}
    \end{subfigure}    \hfill
	\begin{subfigure}[b]{0.195\linewidth}
      \includegraphics[width=\linewidth]{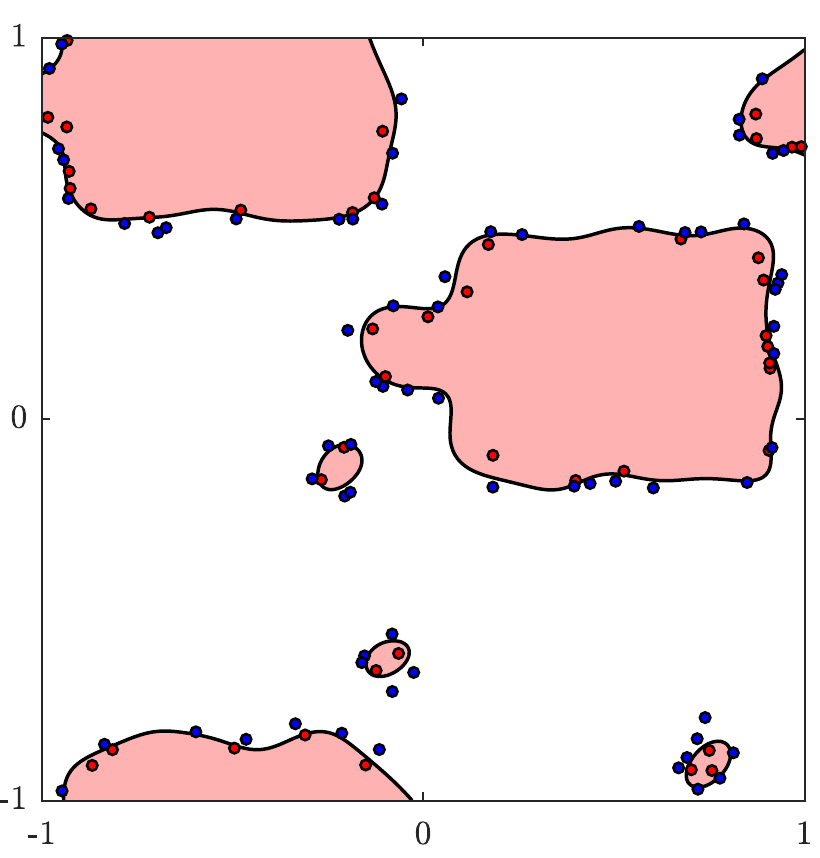}
      \caption{Fastron}
      \label{fig:methodsComparisonFastron}
    \end{subfigure}    \hfill
    \begin{subfigure}[b]{0.195\linewidth}
      \includegraphics[width=\linewidth]{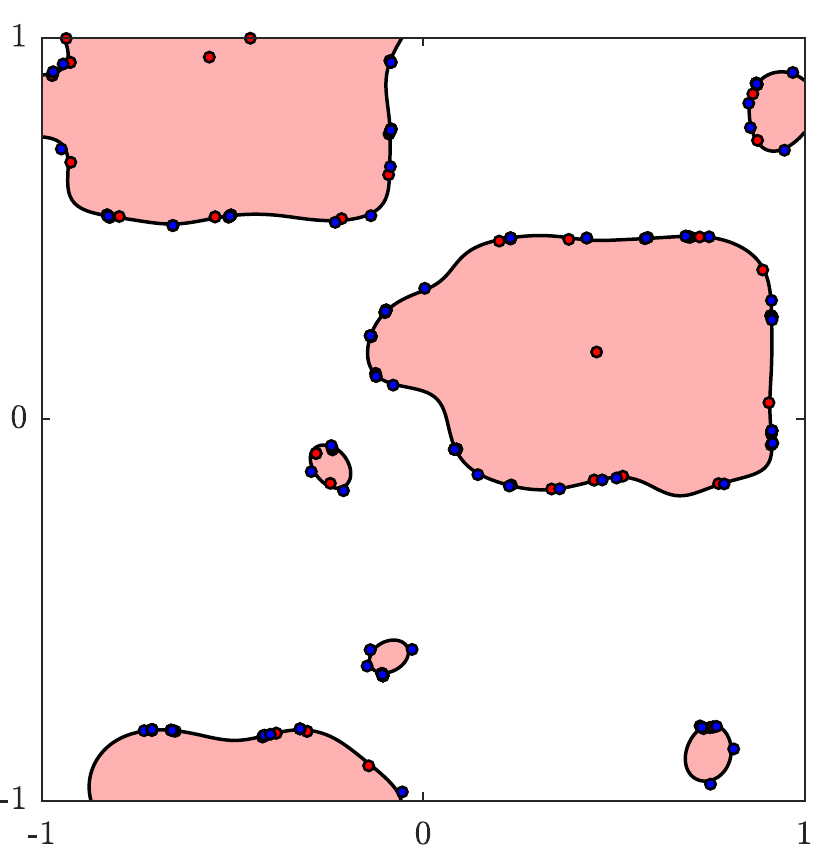}
      \caption{ISVM \cite{Diehl,Pan2015}}
      \label{fig:methodsComparisonSVM}
    \end{subfigure}    \hfill
    \begin{subfigure}[b]{0.195\linewidth}
      \includegraphics[width=\linewidth]{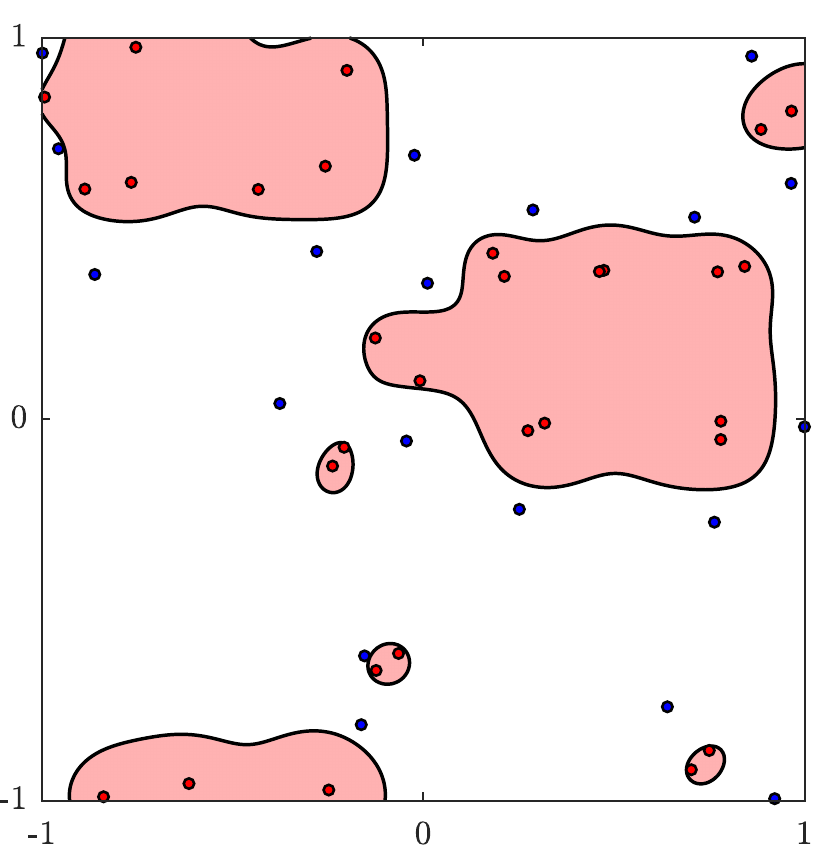}
      \caption{SSVM \cite{Huang2010}}
      \label{fig:methodsComparisonSSVM}
    \end{subfigure}
    
    \caption{Example static environments and C-space approximations using various methods for a simple 2 DOF robot.}
    \label{fig:methodsComparison}
\end{figure*}

\subsubsection{Description of Metrics}
As the sizes of $\mathcal{C}_{free}$ and $\mathcal{C}_{obs}$ may be unbalanced depending on the locations of workspace obstacles, overall accuracy may be skewed when one of the classes dominates the C-space. Thus, we also include within class accuracy, i.e., TPR and TNR. TPR measures the proportion of in-collision samples correctly classified, and TNR measures the proportion of collision-free samples correctly classified.

As measures of classification performance, we measure the average time required to perform a proxy collision check and the size of the model in terms of the number of support points $|\mathcal{S}|$. As proxy checks using each of these models requires a weighted sum of kernel evaluations, proxy check times would be directly related to the number of support points and computational complexity of the type of kernel. As labeling a given configuration when using GJK first requires forward kinematics (FK) to locate where the links of the arm would be, FK is included in the query timing for GJK.

\begin{figure}[t!]
\begin{subfigure}[t]{\linewidth}
\scriptsize
\centering
\ra{1.3}
\begin{tabularx}{\linewidth}{@{}llll@{}}
\toprule
Method & $|\mathcal{S}|$ & Query Time & Training Time \\
\midrule
Fastron				& $133.6\pm 32.4$ 			& $424\pm 90\ $ns 			&$\bm{231.4 \pm 38.3}\ $\textbf{ms} \\
ISVM \cite{Diehl,Pan2015} 	& $349.1\pm 128.6$ 			& $1.78\pm 0.62\ \mu$s 		&$2.72 \pm 1.60\ $s \\
SSVM \cite{Huang2010}		& $\bm{52.4\pm 11.3}$ 	& $\bm{220\pm 52}\ $\textbf{ns} & $18.8 \pm 2.69\ $s \\
FK + GJK & --- & $105.1\pm 16.5\ \mu$s & ---\\
\bottomrule
\end{tabularx}
\caption{2 DOF}
\label{table:staticCase2DOF}
\end{subfigure}

\begin{subfigure}[t]{\linewidth}
\scriptsize
\centering
\setlength{\tabcolsep}{3.5pt}
\ra{1.3}
\begin{tabularx}{\linewidth}{@{}llll@{}}
\toprule
Method & $|\mathcal{S}|$ & Query Time & Training Time \\
\midrule
Fastron					& $1084.1\pm 206.9$ 			& $3.4\pm 0.6\ \mu$s 			&$\bm{707.6\pm 75.2\ }$\textbf{ms} \\
ISVM \cite{Diehl,Pan2015} 	& $1078.1\pm 125.8$ 			& $5.8\pm 0.7\ \mu$s 		&$75.5 \pm 20.6 \ $s			  \\
SSVM \cite{Huang2010}		& $\bm{509.8\pm 106.3}$ 	& $\bm{1.5\pm 0.3\ \mu}$\textbf{s} & $193.7\pm 33.6 \ $s \\
FK + GJK & --- & $166.6\pm 20.0\ \mu$s & --- \\
\bottomrule
\end{tabularx}
\caption{4 DOF}
\label{table:staticCase4DOF}
\end{subfigure}

\caption{Performance of Fastron in a static environment for the 2 DOF and 4 DOF robots shown in Fig. \ref{fig:2and4dofEnv} compared against incremental SVM (ISVM) \cite{Diehl,Pan2015}, sparse SVM (SSVM) \cite{Huang2010}, and the ground truth collision detection method GJK \cite{Gilbert1988}. Lower is better for model size $|\mathcal{S}|$, query time, and model training time. The best results are shown in bold.}
\label{table:staticCase2and4DOF}
\end{figure}

Finally, even though we are working with static environments, we compute training time to determine how long it takes for each algorithm to generate a model. We include training time to assess which algorithms would be suitable for changing environments. This training time includes both the time required to learn the weights for the model and the time required to label each point in the training set using a collision checker.

\subsubsection{Results for 2 DOF Robot}
Fig. \ref{fig:staticResults2DOF} and the table in Fig. \ref{table:staticCase2DOF} provide comparisons of three machine learning methods trained on an environment with up to 4 obstacles for a 2 DOF case. All methods have comparable accuracy, TPR, and TNR. All methods performed significantly faster than GJK, which took on average $105.1\ \mu$s per collision check. Of the machine learning based methods, SSVM had the sparsest solution, which results in it providing the fastest proxy collision checks. Furthermore, ISVM has on average more than 6 times the number of support points probably because more support points are placed near the boundary after its active learning procedure. Due to its larger number of support points and slower kernel, ISVM has the slowest proxy collision check timings.

SSVM takes nearly 100 times longer than Fastron to train, the fastest among all tested methods. Profiling the code suggests the slowest parts of SSVM and ISVM is having to repeatedly solve an optimization problem. Fastron's speed in training compared to the SVM methods is due to not having to optimize an objective function completely and not requiring the entire Gram matrix. 

Visualization of the C-space is straightforward when working with a 2 DOF robot and allows qualitative comparison. Figs. \ref{fig:methodsComparisonGT} shows example ground truth C-spaces (with axes scaled to the input space) for the workspaces shown in Figs. \ref{fig:methodsComparisonWorkspace} along with approximations provided by the three kernel-based methods in Fig. \ref{fig:methodsComparisonFastron}-\ref{fig:methodsComparisonSSVM}. Examining the placement of support points, it is apparent that support points are typically placed closest to the boundary for ISVM except when some C-space obstacles are large. On the other hand, support point placement is much farther from the boundary for SSVM.

\begin{figure}[t]
	\centering
	\begin{subfigure}[t]{0.49\linewidth}
	\includegraphics[width=\linewidth,trim={5.1cm 10.4cm 5.3cm 10.5cm},clip]{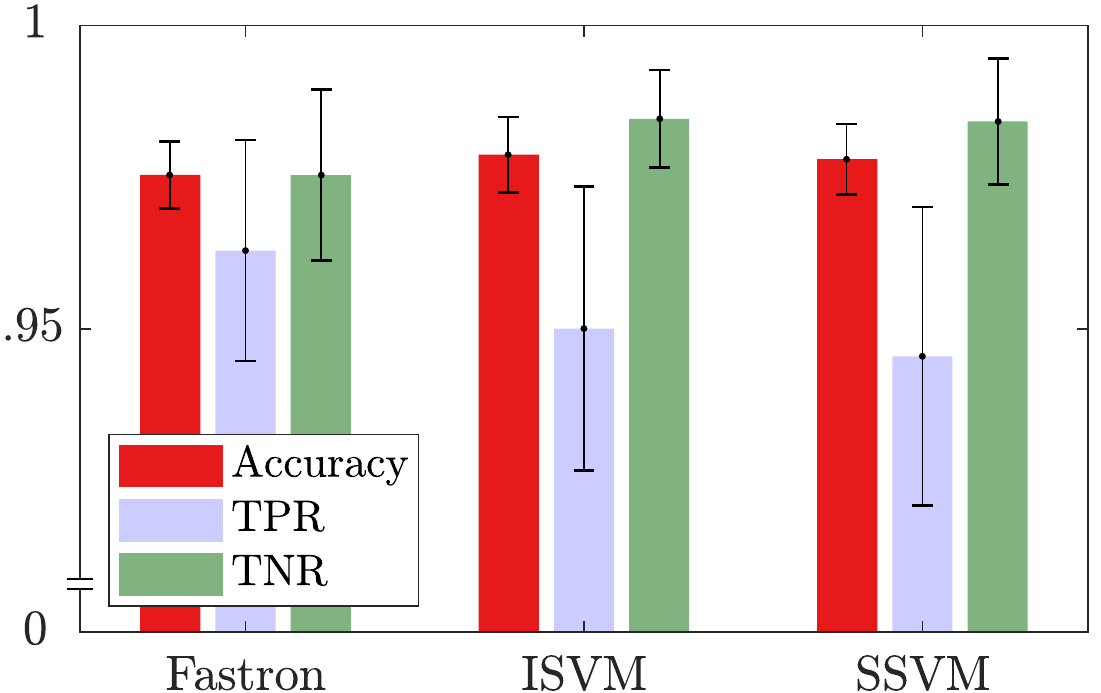}
	\caption{2 DOF}
	\label{fig:staticResults2DOF}
	\end{subfigure}
	\hfill
		\begin{subfigure}[t]{0.49\linewidth}
	\includegraphics[width=\linewidth,trim={5.1cm 10.4cm 5.3cm 10.5cm},clip]{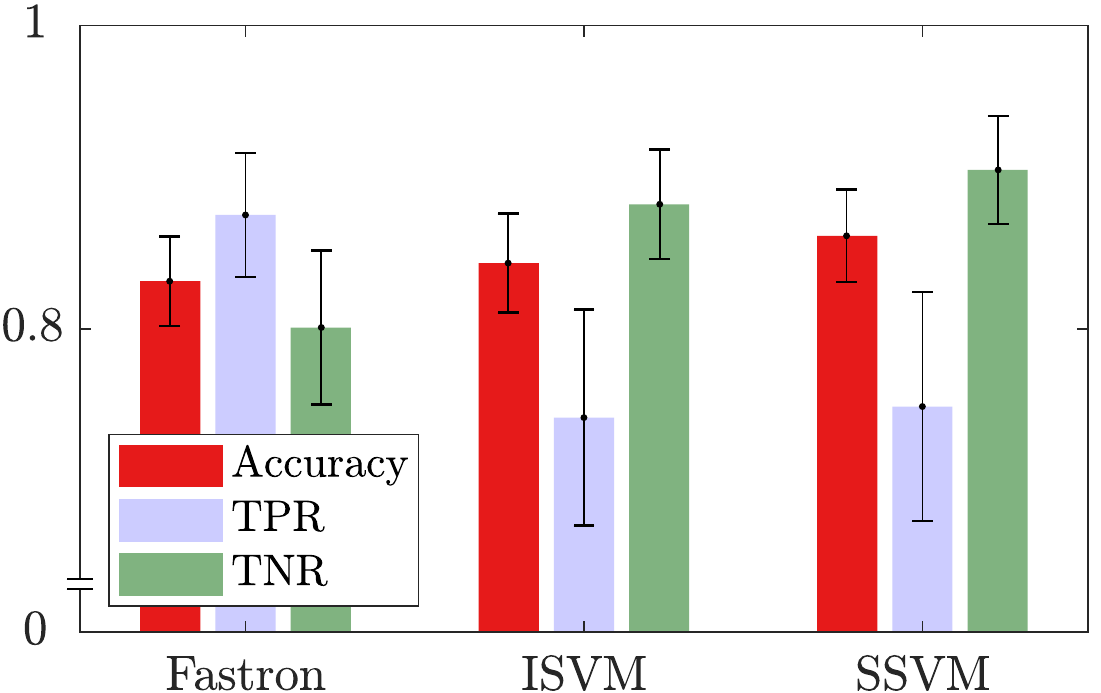}
	\caption{4 DOF}
	\label{fig:staticResults4DOF}
	\end{subfigure}
	\caption{Accuracy, TPR, and TNR (higher is better) for Fastron, ISVM \cite{Karasuyama2010} with active learning \cite{Pan2015}, and SSVM \cite{Huang2010} in a static environment using the 2 DOF and 4 DOF robots shown in Fig. \ref{fig:2and4dofEnv}. GJK \cite{Gilbert1988} is used as the ground truth collision detection method.}	\label{fig:staticResults2and4DOF}
\end{figure}

\begin{figure}[b]
	\centering
	\begin{subfigure}[b]{0.49\linewidth}
		{\includegraphics[width=\linewidth,trim={4cm 9cm 5.1cm 2.2cm},clip]{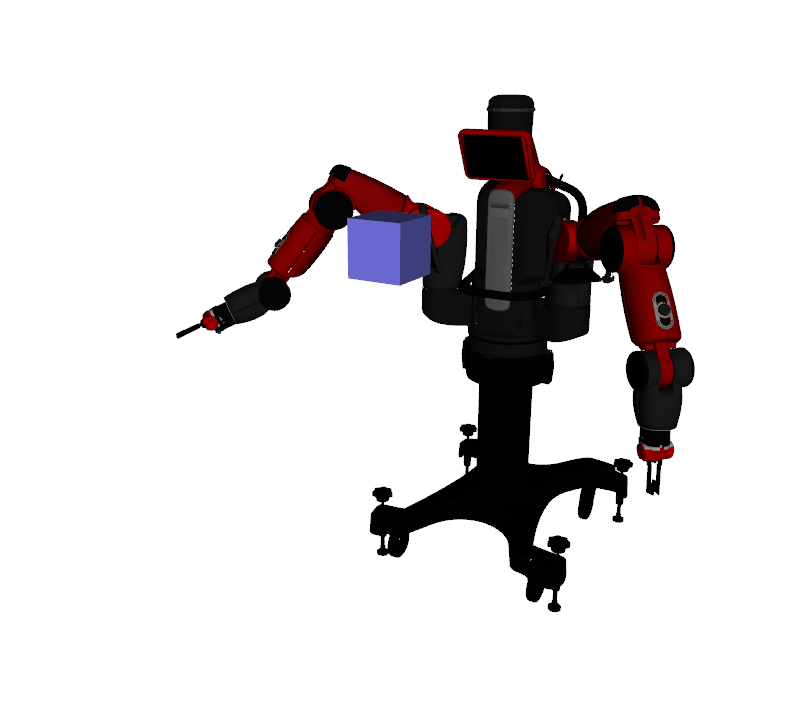}}
	\end{subfigure}
	\begin{subfigure}[b]{0.49\linewidth}
		{\includegraphics[width=\linewidth,trim={4cm 9cm 5.1cm 2.2cm},clip]{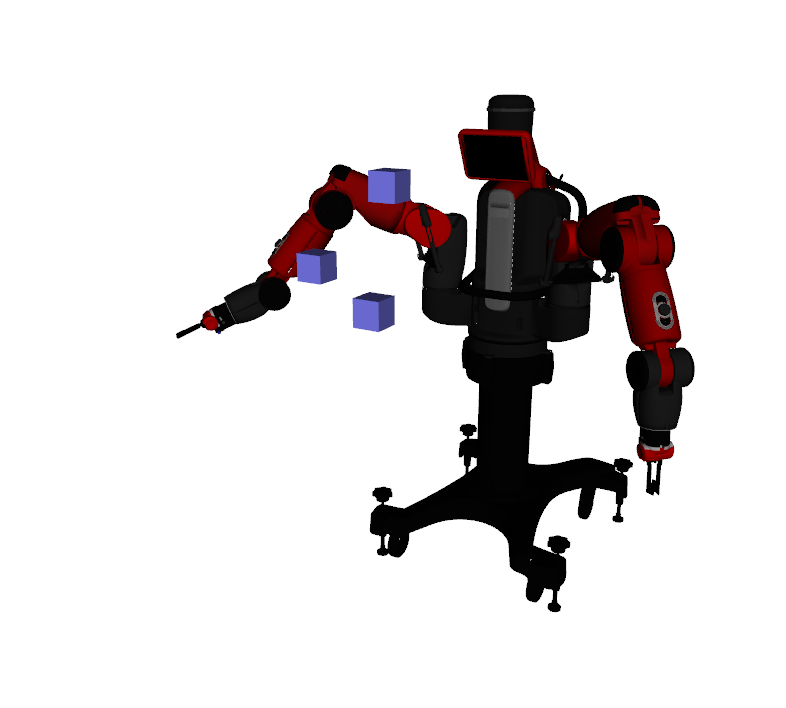}}
	\end{subfigure}
	\caption{Baxter robot with cube obstacles used for collision detection and motion planning experiments.}
	\label{fig:baxterExample}
\end{figure}

\begin{figure}[t]
  \centering
  \begin{subfigure}[b]{\linewidth}
  \scriptsize
\centering
\ra{1.3}
\begin{tabular*}{\linewidth}{@{}llll@{}}
\toprule
Method & $|\mathcal{S}|$ & Query Time & Update Time \\
\midrule
Fastron (FCL)					& $638.6\pm 139.4$ 			& $2.1\pm 0.5\ \mu$s 			&$42.0\pm 5.8\ $ms \\
Fastron (GJK)					& $\bm{558.1\pm 151.1}$ 			& $\bm{1.9\pm 0.5\ \mu}$\textbf{s} 			&$\mathbf{15.8\pm 3.8\ }$\textbf{ms} \\
FK + GJK \cite{Gilbert1988}  & --- & $9.1\pm 0.2\ \mu$s & --- \\
FK + FCL \cite{Pan2012} & --- & $29.2\pm 0.5 \ \mu$s & --- \\
\bottomrule
\end{tabular*}
\caption{4 DOF}
\label{table:dynCase4DOF}
  \end{subfigure}
  
  \begin{subfigure}[b]{\linewidth}
  \scriptsize
\centering
\ra{1.3}
\begin{tabular*}{\linewidth}{@{}lllll@{}}
\toprule
Method & $|\mathcal{S}|$ & Query Time & Update Time \\
\midrule
Fastron (FCL)					& $1060.5\pm 217.9$ 			& $4.1\pm 0.8\ \mu$s 			&$62.4\pm 10.4\ $ms \\
Fastron (GJK)					& $\bm{886.8\pm 215.5}$ 			& $\bm{3.5\pm 0.9\ \mu}$\textbf{s} 			&$\bm{24.2\pm 6.4\ }$\textbf{ms} \\
FK + GJK \cite{Gilbert1988}		& --- & $9.5 \pm 0.2 \ \mu$s & --- \\
FK + FCL \cite{Pan2012} & --- & $30.5\pm 0.5\ \mu$s & --- \\
\bottomrule
\end{tabular*}
\caption{6 DOF}
\label{table:dynCase6DOF}
  \end{subfigure}
  
  \begin{subfigure}[b]{\linewidth}
  
\scriptsize
\centering
\ra{1.3}
\begin{tabular*}{\linewidth}{@{}lllll@{}}
\toprule
Method & $|\mathcal{S}|$ & Query Time & Update Time \\
\midrule
Fastron (FCL)		& $1306.0\pm 291.7$ 			& $5.5\pm 1.2 \ \mu$s 			&$74.7\pm 15.7\ $ms \\
Fastron (GJK)					& $\bm{1151.3\pm 283.9}$ 			& $\bm{4.9\pm 1.2\ \mu}$\textbf{s} 			&$\bm{32.6\pm 7.5 \ }$\textbf{ms} \\
FK + GJK \cite{Gilbert1988}  & --- & $9.6\pm 0.4 \ \mu$s & --- \\
FK + FCL \cite{Pan2012} & --- & $29.3\pm 0.5\ \mu$s & --- \\
\bottomrule
\end{tabular*}
\caption{7 DOF}
\label{table:dynCase7DOF}
  \end{subfigure}
\caption{Performance in a changing environment when actuating the first 4 DOF, the first 6 DOF, and all 7 DOF of the Baxter robot's right arm, respectively, shown in Fig. \ref{fig:baxterExample}. Lower is better for model size $|\mathcal{S}|$, query time, and model update time. The best results are shown in bold.}%
\label{table:dynCase467DOF}
\end{figure}

\begin{figure}
\scriptsize
\centering
\ra{1.3}
\begin{tabular*}{\linewidth}{@{}lllll@{}}
\toprule
Method & $|\mathcal{S}|$ & Query Time & Update Time \\
\midrule
Fastron (FCL)		& $1637.5\pm 305.9$ 			& $6.7\pm 1.2 \ \mu$s 			&$89.2\pm 17.4\ $ms \\
Fastron (GJK)					& $\bm{1582.9\pm 321.1}$ 			& $\bm{6.6\pm 1.4\ \mu}$\textbf{s} 			&$\bm{56.3\pm 14.5 \ }$\textbf{ms} \\
FK + GJK \cite{Gilbert1988}  & --- & $15.3\pm 1.6 \ \mu$s & --- \\
FK + FCL \cite{Pan2012} & --- & $29.1\pm 0.8\ \mu$s & --- \\
\bottomrule
\end{tabular*}
\caption{Performance in a changing environment with three moving obstacles when actuating all 7 DOF of the Baxter robot's right arm, respectively, shown in Fig. \ref{fig:baxterExample}. Lower is better for model size $|\mathcal{S}|$, query time, and model update time. The best results are shown in bold.}
\label{table:dynCase7DOFmultiobs}
\end{figure}

\begin{figure*}[t]
\centering
\begin{subfigure}[t]{0.32\linewidth}
	\includegraphics[width=\linewidth]{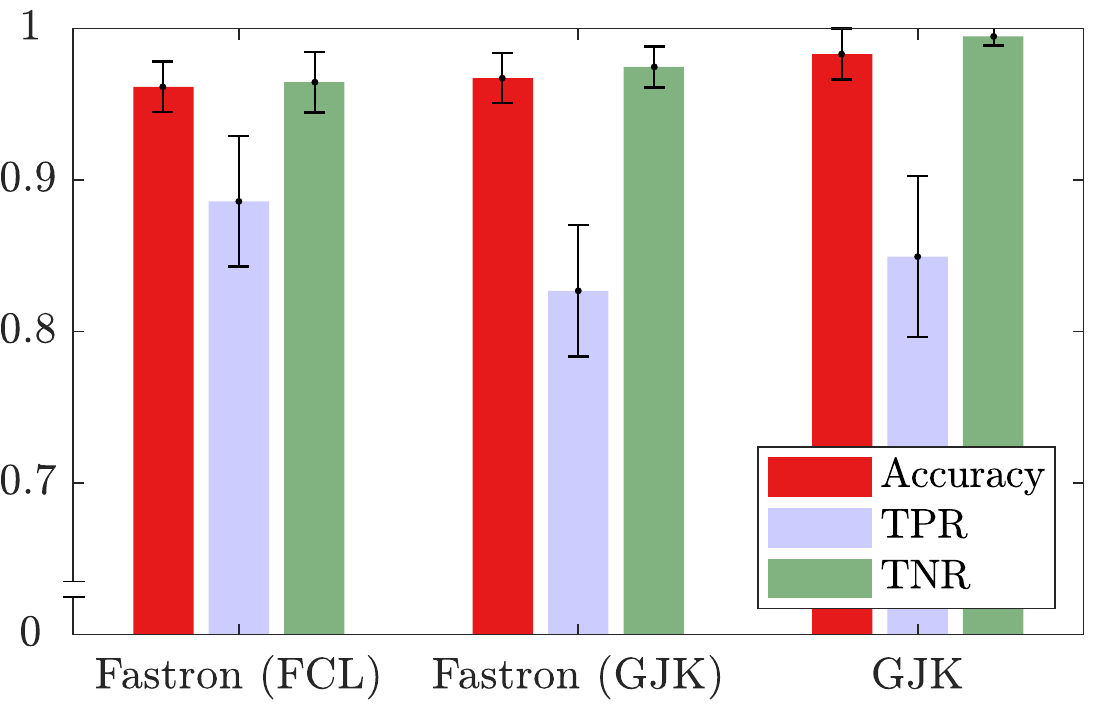}
	\caption{4 DOF}
	\label{fig:dynamicResults4DOF}
\end{subfigure}
\hfill
\begin{subfigure}[t]{0.32\linewidth}
	\includegraphics[width=\linewidth]{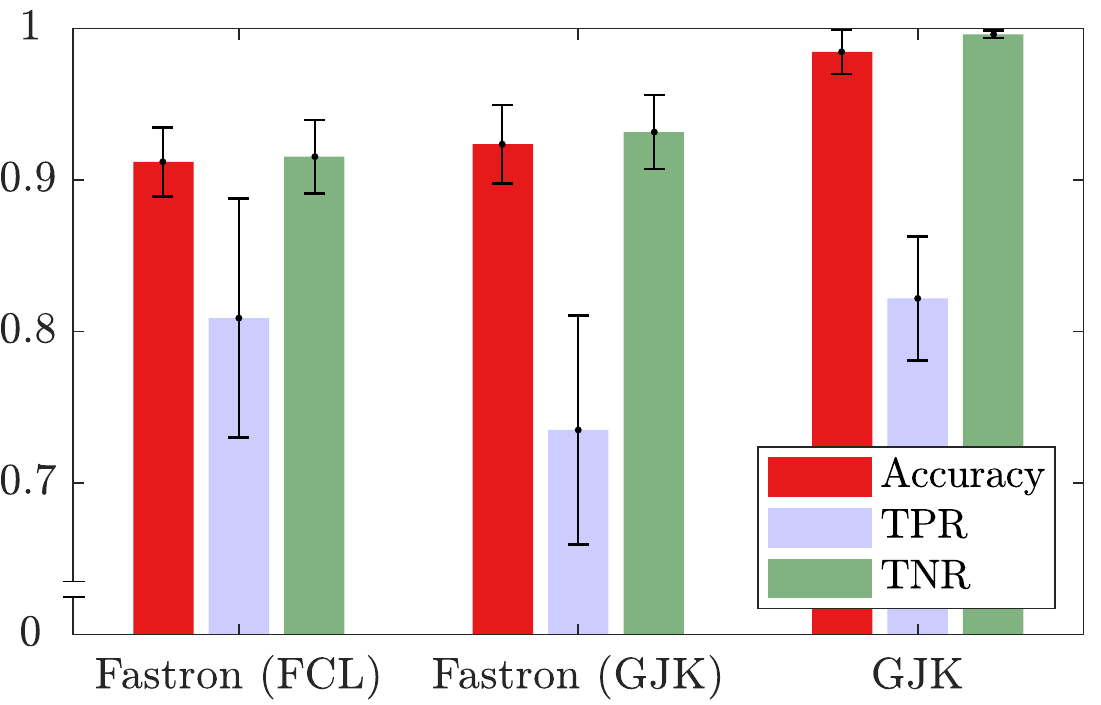}
	\caption{6 DOF}
	\label{fig:dynamicResults6DOF}
\end{subfigure}
\hfill
\begin{subfigure}[t]{0.32\linewidth}
	\includegraphics[width=\linewidth]{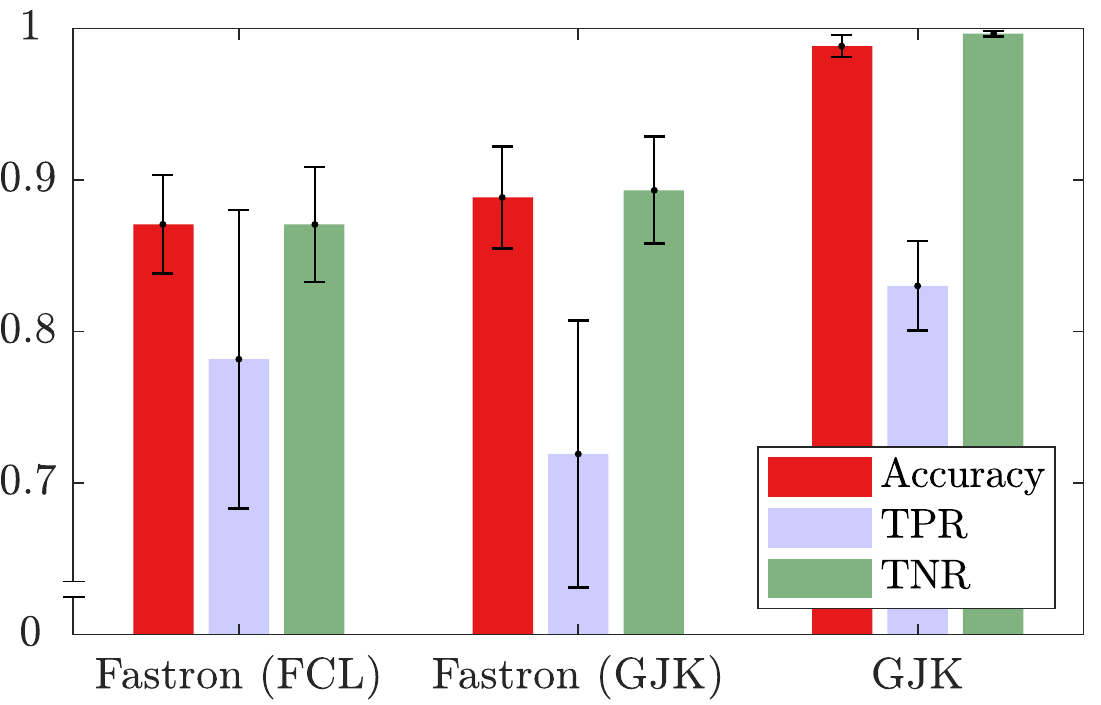}
	\caption{7 DOF}
	\label{fig:dynamicResults7DOF}
\end{subfigure}
\caption{Accuracy, TPR, and TNR (higher is better) for Fastron and GJK \cite{Gilbert1988} in a continuously changing environment using the Baxter's right arm and one workspace obstacle as shown in Fig. \ref{fig:baxterExample}. Fastron is trained on both FCL \cite{Pan2012} and GJK. FCL is used as the ground truth collision detection method. The three cases correspond to actuating the first 4 DOF, the first 6 DOF, and all 7 DOF of the arm, respectively. Note that the GJK results do not change significantly because the entire arm is always used for geometry-based collision detection while the dimensionality of the space that Fastron models equals the number of DOF.}
\label{fig:dynamicResults467DOF}
\end{figure*}

\subsubsection{Results for 4 DOF Robot}
Fig. \ref{fig:staticResults4DOF} and the table in Fig. \ref{table:staticCase4DOF} provide a comparison of the three machine learning methods trained on an environment with up to 4 obstacles. ISVM and SSVM have comparable accuracy, TPR, and TNR. On the other hand, Fastron has slightly less accuracy than the other methods, but has significantly more TPR and less TNR. Higher TPR is advantageous for collision detection where a more conservative prediction is preferred. We anticipate all methods would improve in terms of accuracy, TPR, and TNR if given more training data.

As with the 2 DOF case, all methods performed significantly faster than GJK, which took on average $166.6\ \mu$s per collision check. SSVM once again had the sparsest solution, requiring half as many support points as the other methods. SSVM's sparsest solution also makes it provide the fastest predictions. Fastron and ISVM require approximately the same number of support points, but Fastron classifies faster due to its cheaper kernel.

SSVM takes the longest to train, while Fastron trains at least 200 times faster. Once again, Fastron's training speed is due to not having to optimize an objective function completely and not requiring the entire Gram matrix.

\subsection{Performance in Changing Environment}
\label{changingEnvironment}
\subsubsection{Description of Experiment and Metrics}
We test the Fastron algorithm in changing environments using C++. As Fastron is the only learning-based method that globally models C-space for changing environments to our knowledge, we do not include comparisons to other machine learning methods. Furthermore, the training timings for ISVM and SSVM from the static case suggest that our implementations of these methods are not suited for changing environments. Instead, we compare Fastron to two collision detection methods: GJK and the Flexible Collision Library (FCL) \cite{Pan2012}, the default collision library for the MoveIt! motion planning framework. We use FCL as the ground truth labels for these tests, and use GJK \cite{Gilbert1988} as an approximate geometry method. We train Fastron on labels from FCL and GJK and include the results for both variations.

The experiments involve moving boxes around the reachable workspace of a Baxter robot's right arm. For GJK, we use cylinders as approximations for the geometry of each link in the arm. Robot kinematics and visualization are performed using MoveIt! \cite{Sucan} and rviz \cite{Hershberger2018} in ROS \cite{Quigley}. We anticipate that having only one workspace obstacle (such as the workspace shown on the left in Fig. \ref{fig:baxterExample}) would be the most challenging case for Fastron due to a misbalance of class sizes in C-space. On the other hand, one workspace obstacle would be the easiest case for geometry- and kinematics-based collision detectors whose performance is dependent on the number of objects in the workspace. Thus, most analysis in the following experiments involve one workspace obstacle, but results from multiple obstacle workspaces (such as that shown on the right in Fig. \ref{fig:baxterExample}) are also included.

For the experiments with one obstacle, we perform three versions of the tests: using the first 4 DOF of the arm (excluding all wrist motions), the first 6 DOF (excluding gripper rotation), and all 7 DOF. For the 4 DOF and 6 DOF experiments, we perform full collision checking on the arm, but leave the unactuated joints' positions fixed at 0. We only consider the 7 DOF version when using workspaces with multiple obstacles.

We used a coarse grid search to select the parameters for Fastron. We use $\gamma = 10$ for the 4 DOF case and $\gamma = 5$ for the higher DOF cases. The conditional bias parameter is $\beta=100$ for the 4 DOF case and $\beta = 500$ for the higher DOF cases. Our active learning strategy found and labeled 500 additional points for each case.

As with the static case, we compute accuracy, TPR, and TNR to measure the correctness against FCL of each approximation method. Additionally, we include model size for Fastron and the query timing of all methods. FK is included in the query timing for both GJK and FCL. Finally, we include the time required to update the Fastron model given the previously trained model.

\subsubsection{Results for 4, 6, and 7 DOF Manipulators}
Fig. \ref{fig:dynamicResults467DOF} and the tables in Fig. \ref{table:dynCase467DOF} provide comparisons of Fastron (trained on FCL and GJK) to the geometry-based collision detectors for the 4, 6, and 7 DOF cases, respectively, when using one workspace obstacle. In all cases, GJK had the highest accuracy and TNR. The TPR of Fastron trained on FCL almost meets that of GJK, but only beats GJK for the 4 DOF case.

Comparing Fastron trained on FCL to Fastron trained on GJK, we can see a noticeable improvement in update time. The update rate increases from around 24 to 63 Hz for the 4 DOF case, from 16 to 41 Hz for the 6 DOF case, and from 13 to 30 Hz for the 7 DOF case. The reason for the drastic decrease in update time is obvious: as a large portion of time to update the model is spent on performing collision checks, using a cheaper collision detector for labeling would decrease the update time. On the other hand, we also see that the TPR of Fastron (GJK) is lower compared to that of Fastron (FCL), probably because training on the approximate collision detection method causes some missed detections that would only be detected with a more precise detector. These missed detections may also explain why Fastron (GJK) requires fewer support points on average than Fastron (FCL).

As the degrees of freedom increase, we notice the accuracies of Fastron decrease while model sizes, query times, and update times increase. This is due to the fact that higher dimensional spaces requires more data to correctly model. GJK and FCL are unaffected by the change in DOF because we always need to use the entire arm for collision detections.

The table in Fig. \ref{table:dynCase7DOFmultiobs} provides model sizes, query times, and update times for the full 7 DOF arm when there are three workspace obstacles. Comparing the table in Fig. \ref{table:dynCase7DOFmultiobs} to that in Fig. \ref{table:dynCase7DOF}, we can see that Fastron query times increase marginally when the number of workspace obstacles increase because the number of support points required to model the corresponding C-space is higher. We can see that GJK suffers when the number of workspace obstacles increase while FCL does not significantly change. Consequently, while update times increase for both versions of Fastron with more workspace obstacles, the update times increase more for the GJK-based Fastron than for the FCL-based Fastron.

\begin{figure}[!ht]
  \centering
  \includegraphics[width=\linewidth]{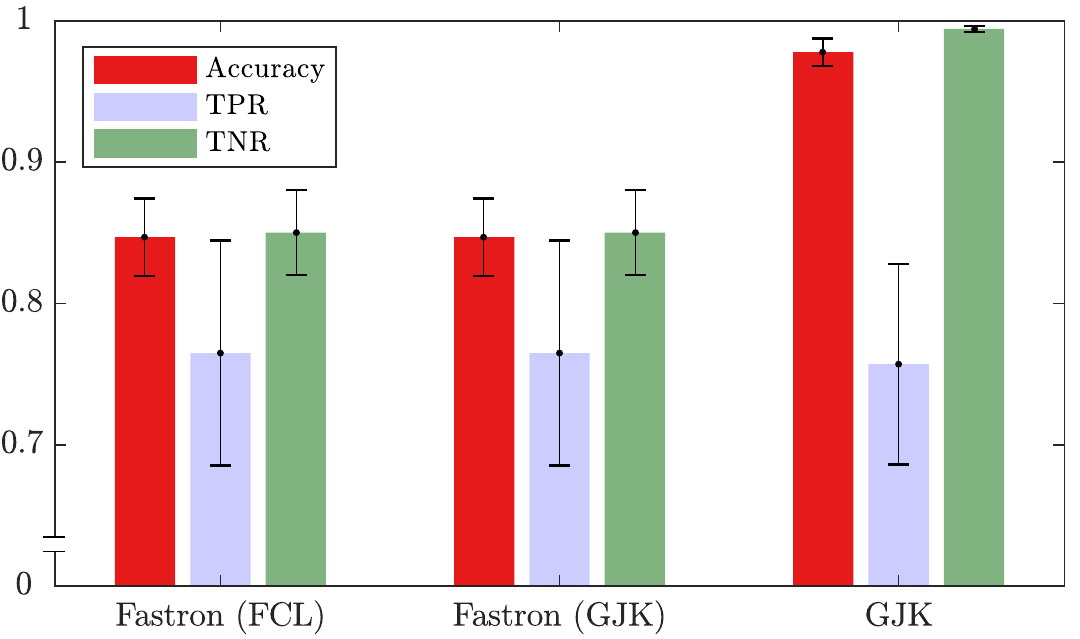}
  \caption{Accuracy, TPR, and TNR (higher is better) for Fastron (trained on FCL \cite{Pan2012}), Fastron (trained on GJK), and GJK \cite{Gilbert1988} in a continuously changing environment using one of the Baxter's 7 DOF arms and three workspace obstacles as shown in Fig. \ref{fig:baxterExample}. FCL is used for ground truth.}
  \label{fig:dynamicResults7Multiobs}
\end{figure}

\begin{figure*}[!ht]
  \centering
  \includegraphics[width=\linewidth,trim={0.8cm 0cm 0.3cm 0cm},clip]{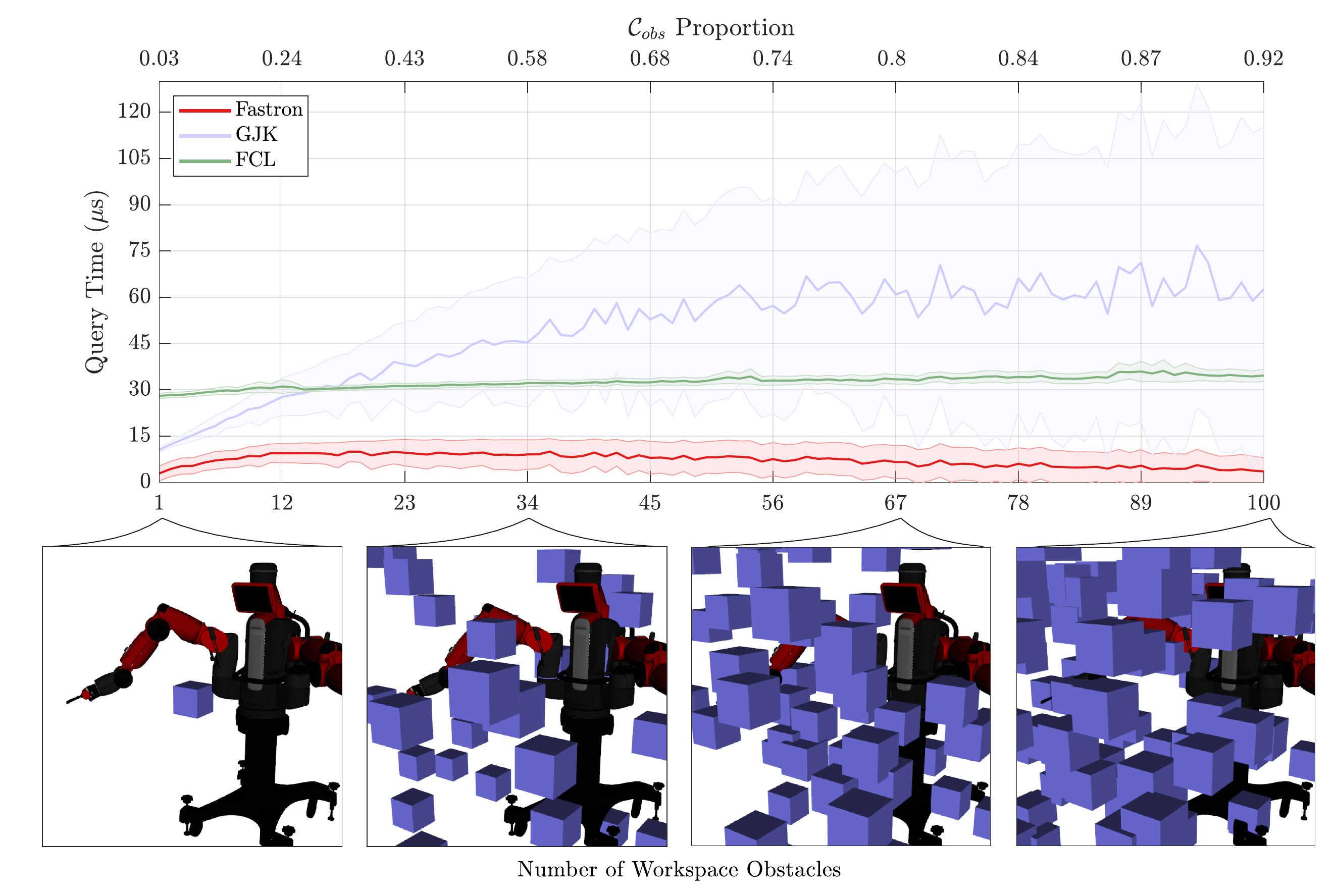}
  \caption{Query times for Fastron, GJK \cite{Gilbert1988}, and FCL \cite{Pan2012} against the number of workspace obstacles (bottom axis) and the estimated proportion that $\mathcal{C}_{obs}$ occupies of the C-space (top axis). GJK's and FCL's query times increase linearly with the number of workspace obstacles, while Fastron's query times increase before decreasing around where $\mathcal{C}_{obs}$ occupies $50\%$ of the C-space.}
  \label{fig:numObsSweep}
\end{figure*}

\begin{figure*}[ht]
    \centering
    \begin{subfigure}[b]{0.245\linewidth}
      \includegraphics[width=\linewidth,trim={0cm 0cm 0.8cm 0cm},clip]{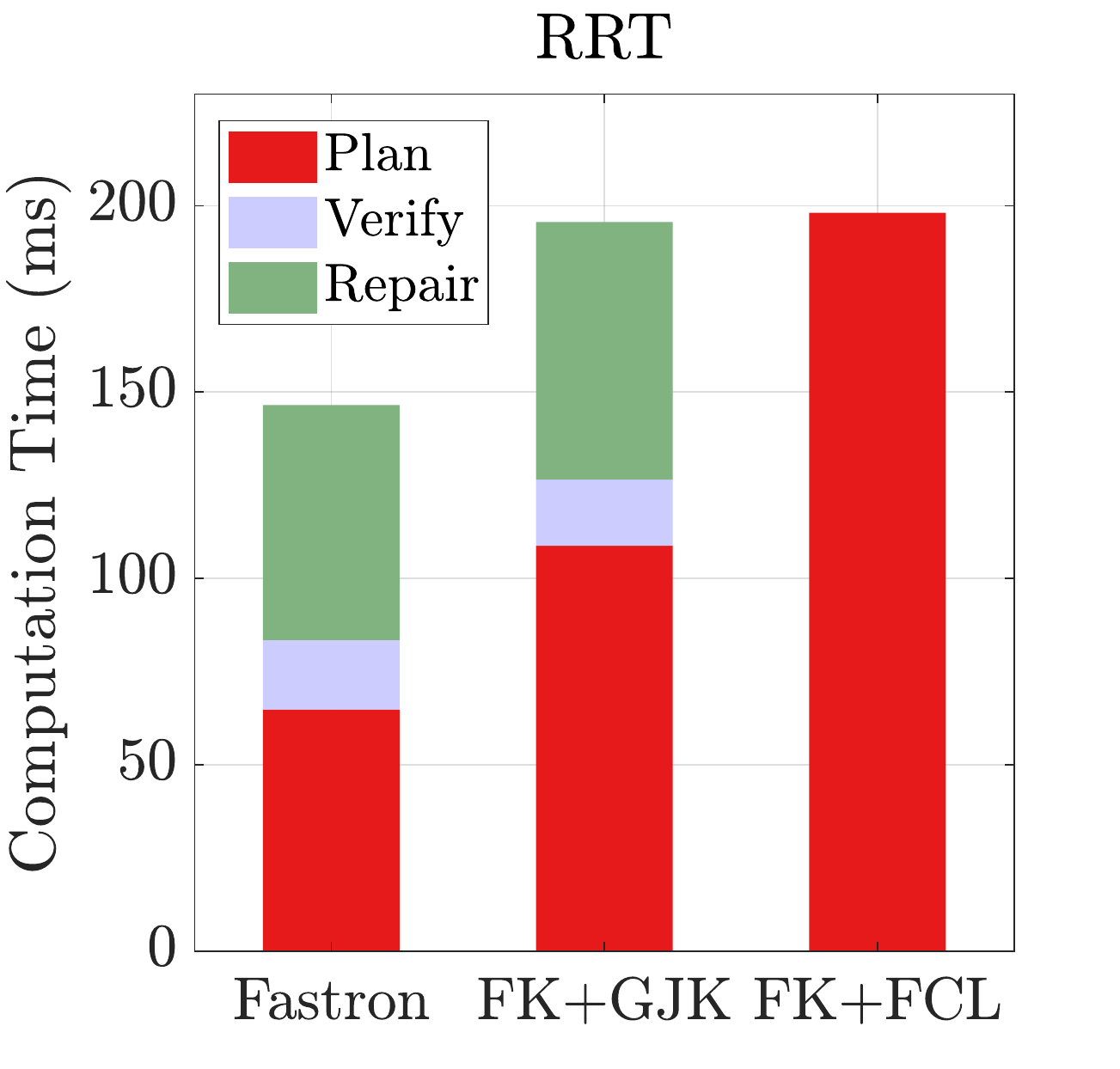}
    \end{subfigure}
    \hfill
    \begin{subfigure}[b]{0.245\linewidth}
      \includegraphics[width=\linewidth,trim={0cm 0cm 0.8cm 0cm},clip]{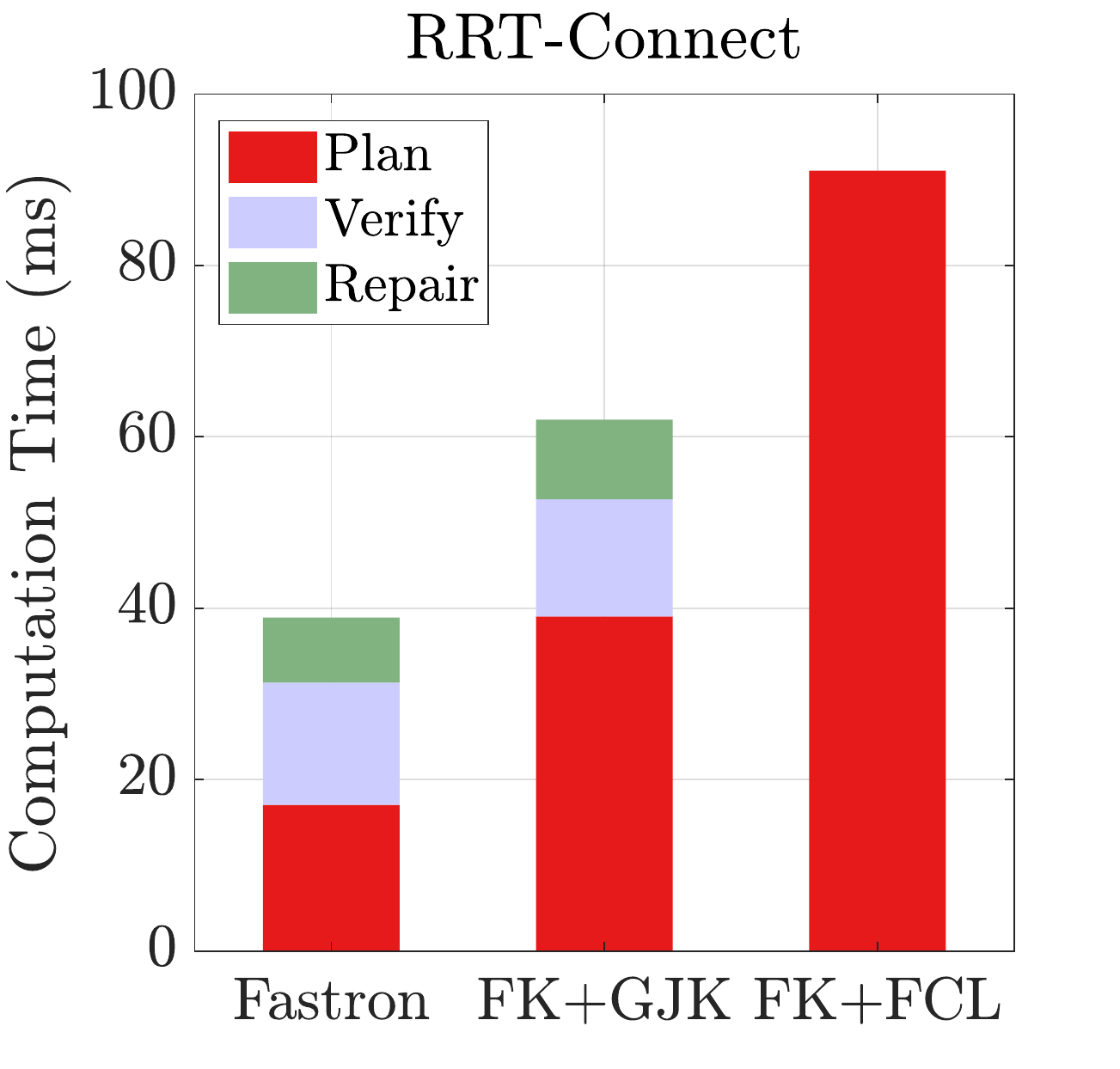}
    \end{subfigure}
    \hfill
	\begin{subfigure}[b]{0.245\linewidth}
      \includegraphics[width=\linewidth,trim={0cm 0cm 0.8cm 0cm},clip]{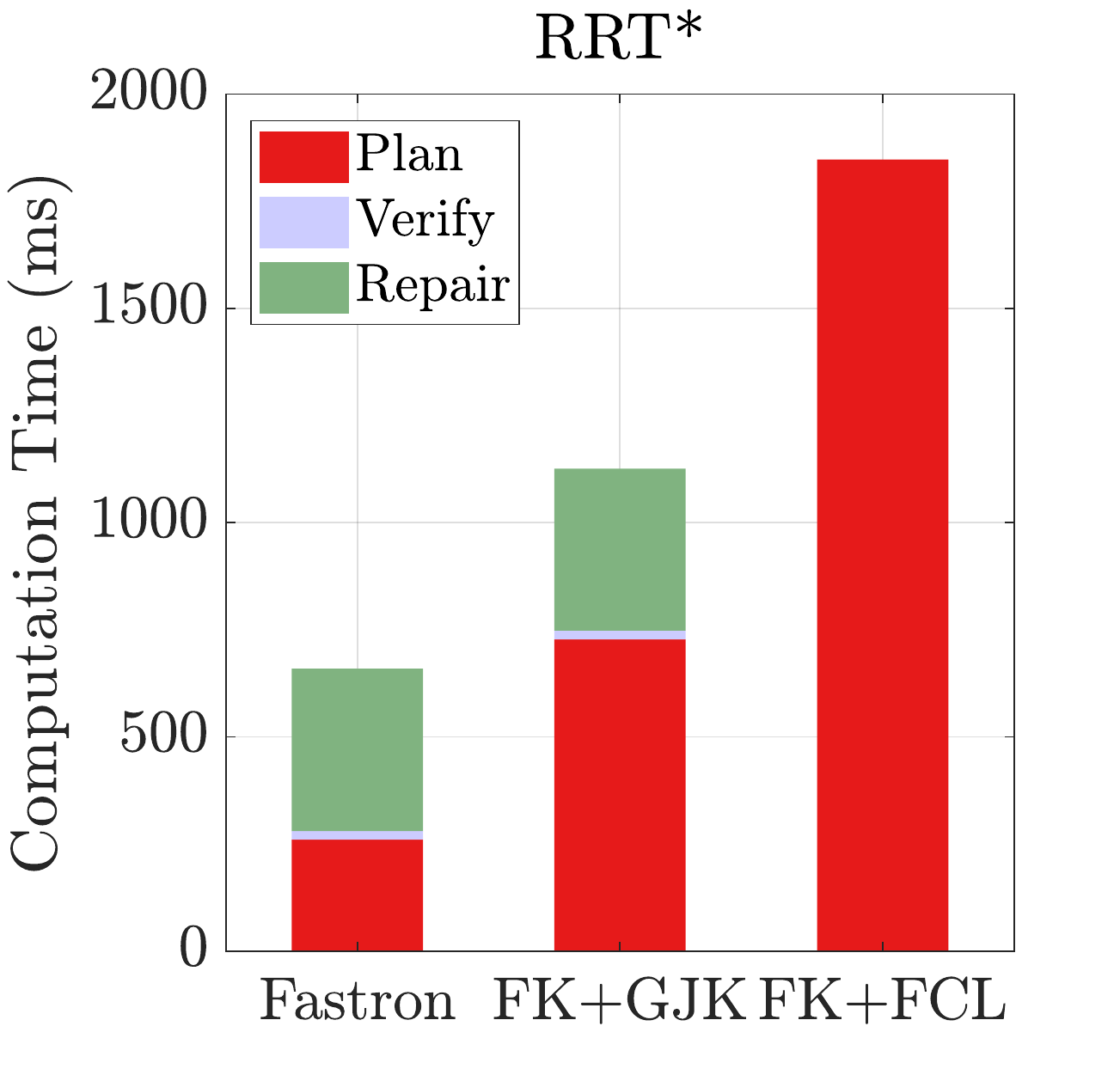}
    \end{subfigure}
    \hfill
    \begin{subfigure}[b]{0.245\linewidth}
      \includegraphics[width=\linewidth,trim={0cm 0cm 0.8cm 0cm},clip]{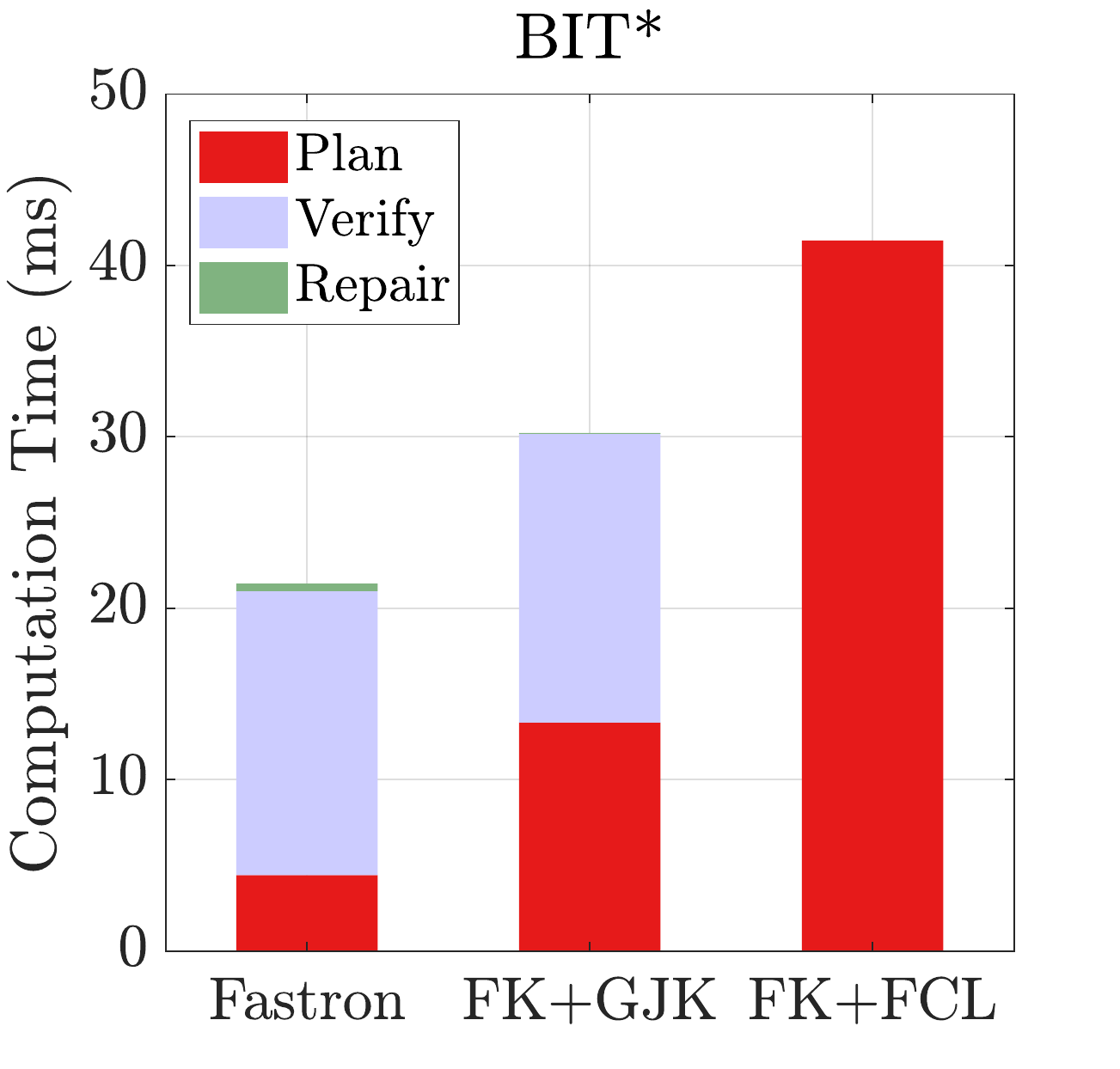}
    \end{subfigure}
    
    \vspace{20pt}
       
    \begin{subfigure}[b]{0.245\linewidth}
      \includegraphics[width=\linewidth,trim={0cm 0cm 0.8cm 0cm},clip]{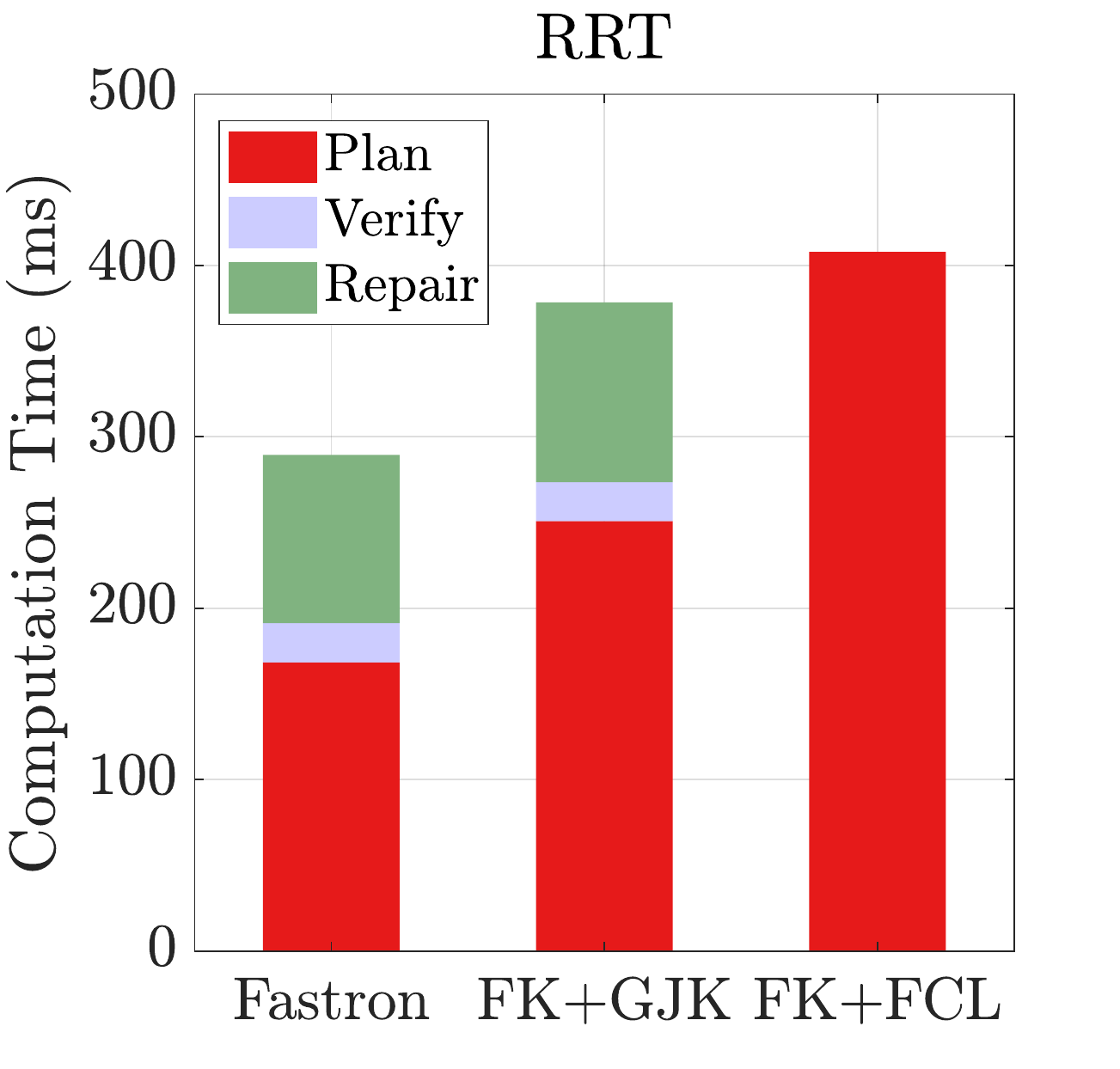}
    \end{subfigure}
    \hfill
    \begin{subfigure}[b]{0.245\linewidth}
      \includegraphics[width=\linewidth,trim={0cm 0cm 0.8cm 0cm},clip]{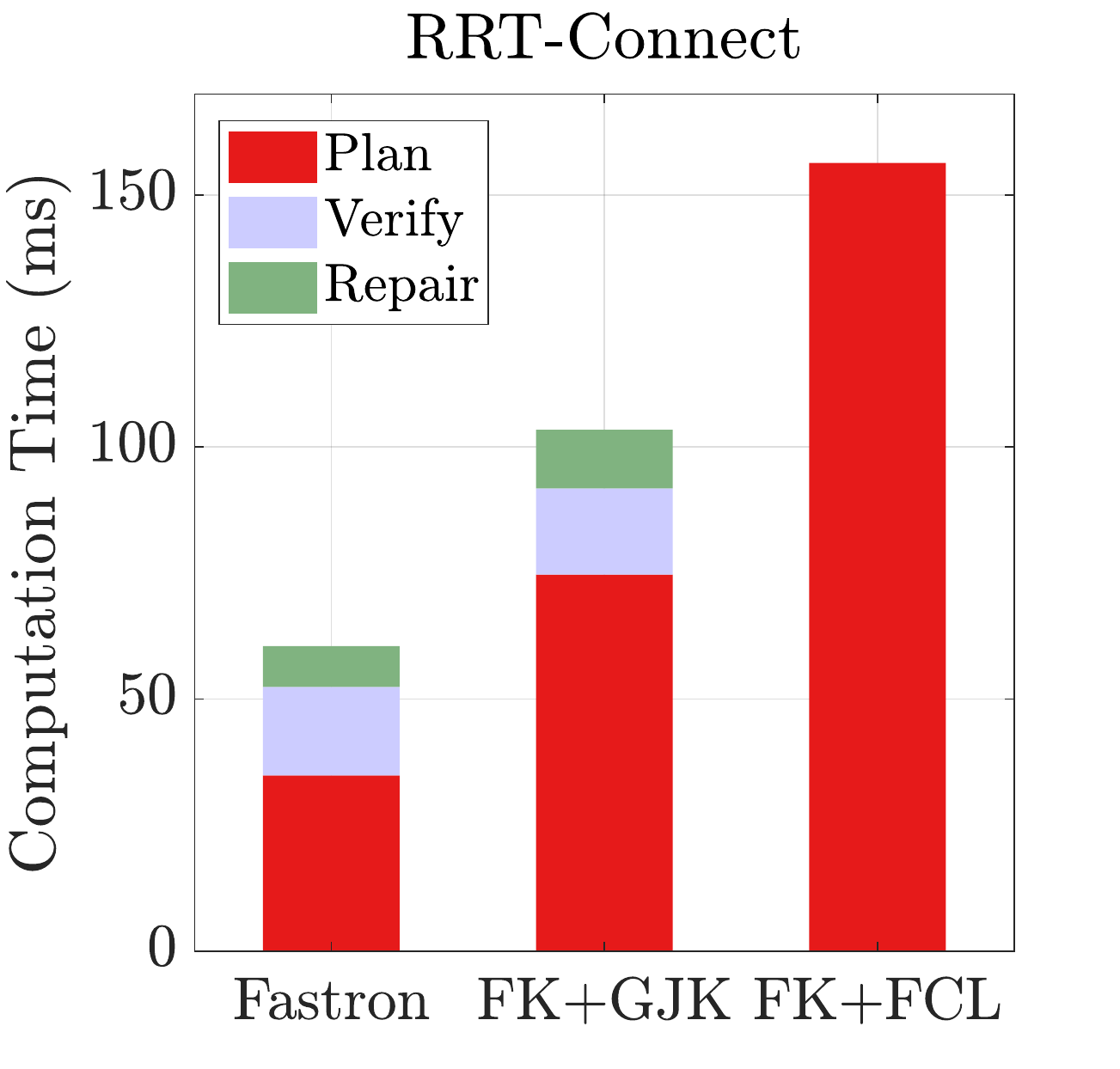}
    \end{subfigure}
    \hfill
	\begin{subfigure}[b]{0.245\linewidth}
      \includegraphics[width=\linewidth,trim={0cm 0cm 0.8cm 0cm},clip]{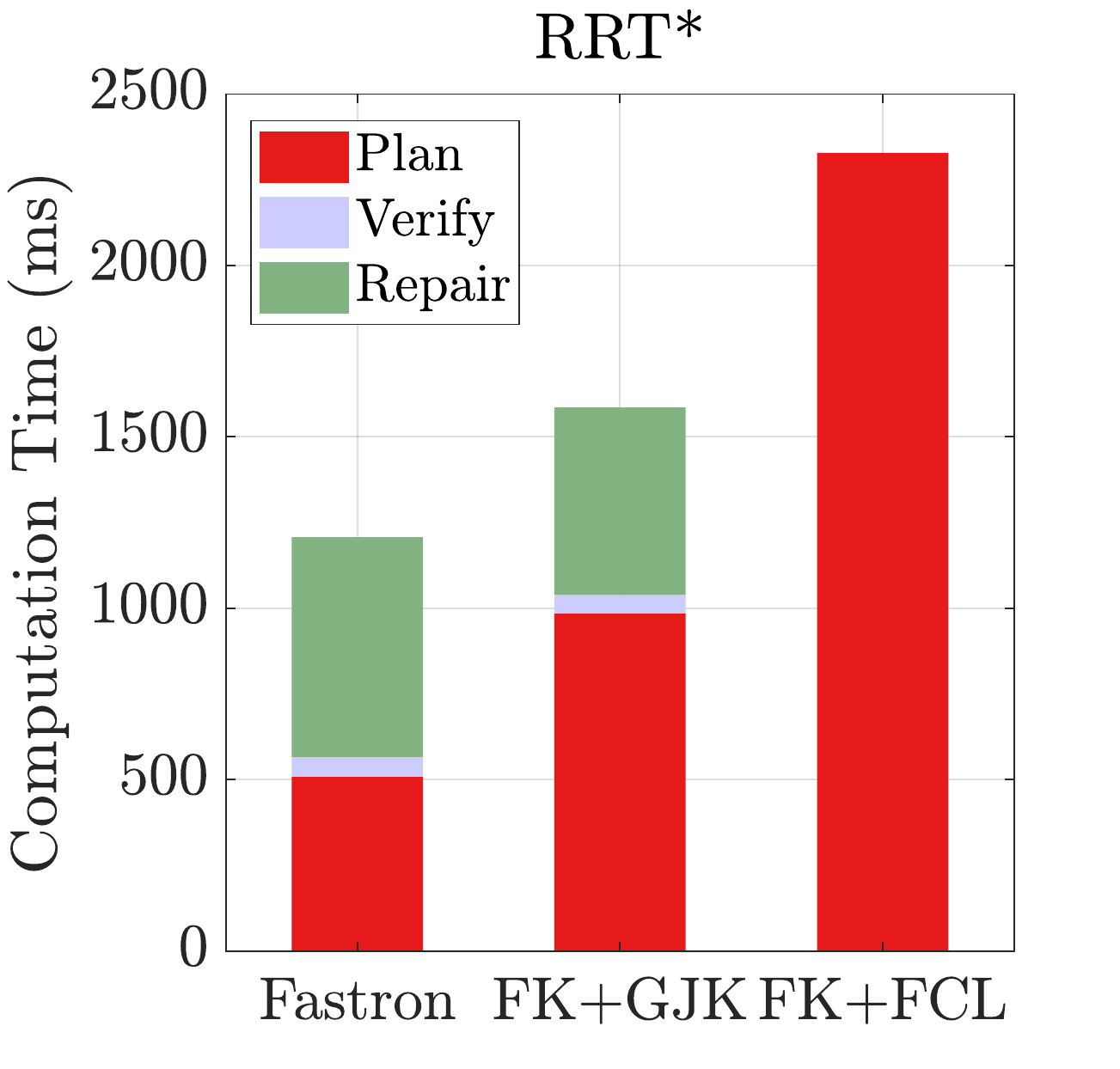}
    \end{subfigure}
    \hfill
    \begin{subfigure}[b]{0.245\linewidth}
      \includegraphics[width=\linewidth,trim={0cm 0cm 0.8cm 0cm},clip]{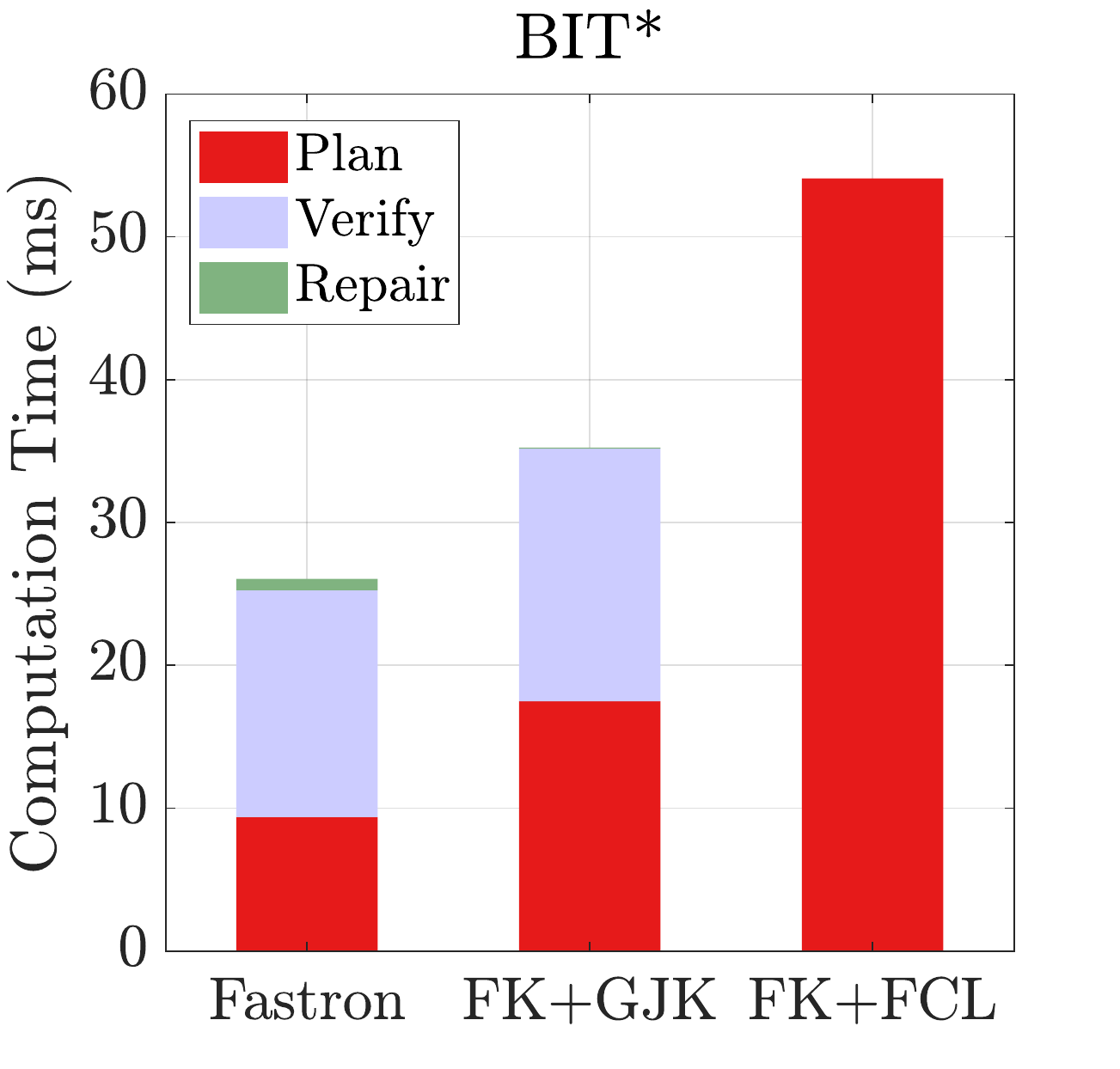}
    \end{subfigure}

   \vspace{20pt}
   
    \begin{subfigure}[b]{0.245\linewidth}
      \includegraphics[width=\linewidth,trim={0cm 0cm 0.8cm 0cm},clip]{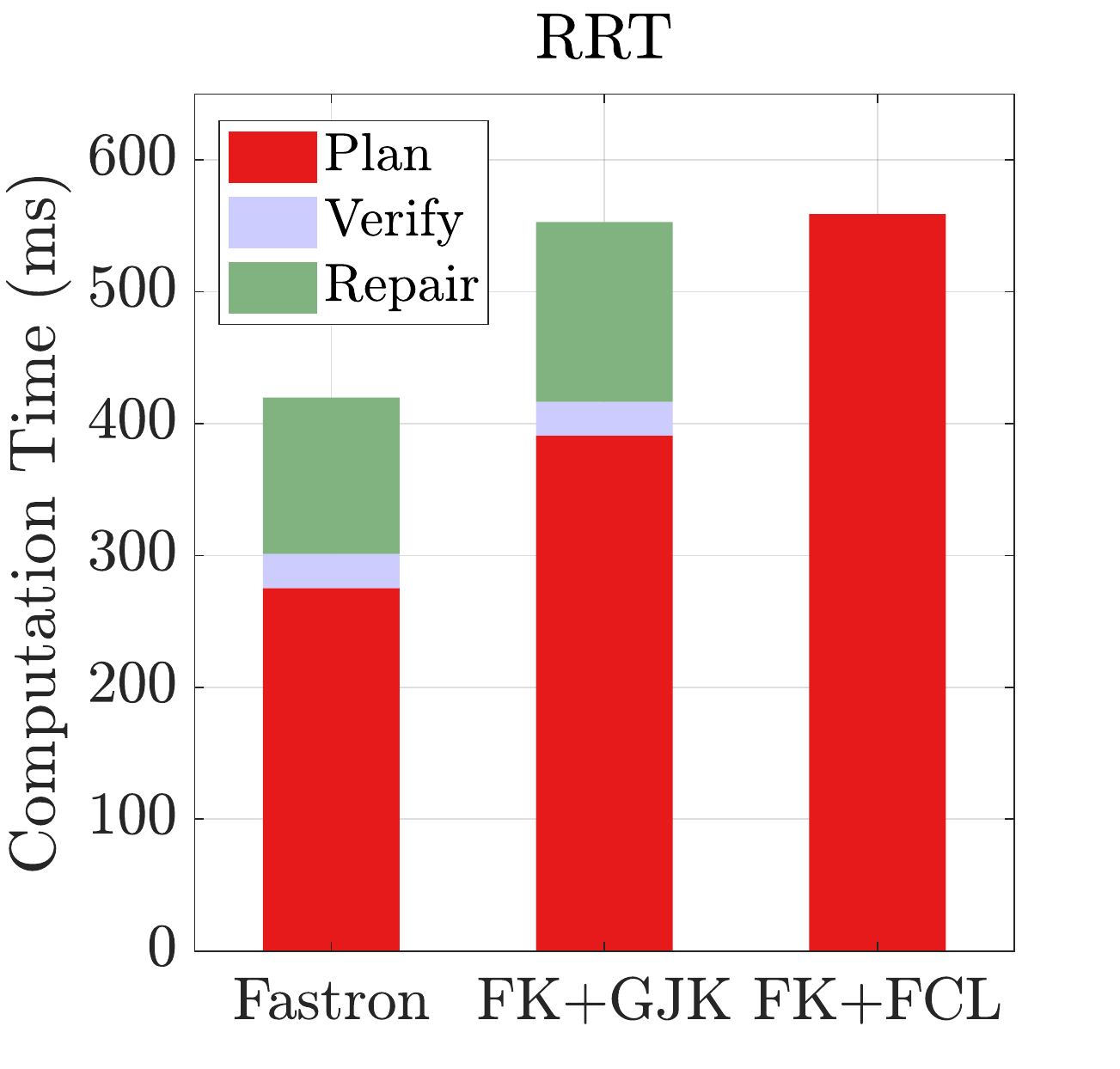}
    \end{subfigure}
    \hfill
    \begin{subfigure}[b]{0.245\linewidth}
      \includegraphics[width=\linewidth,trim={0cm 0cm 0.8cm 0cm},clip]{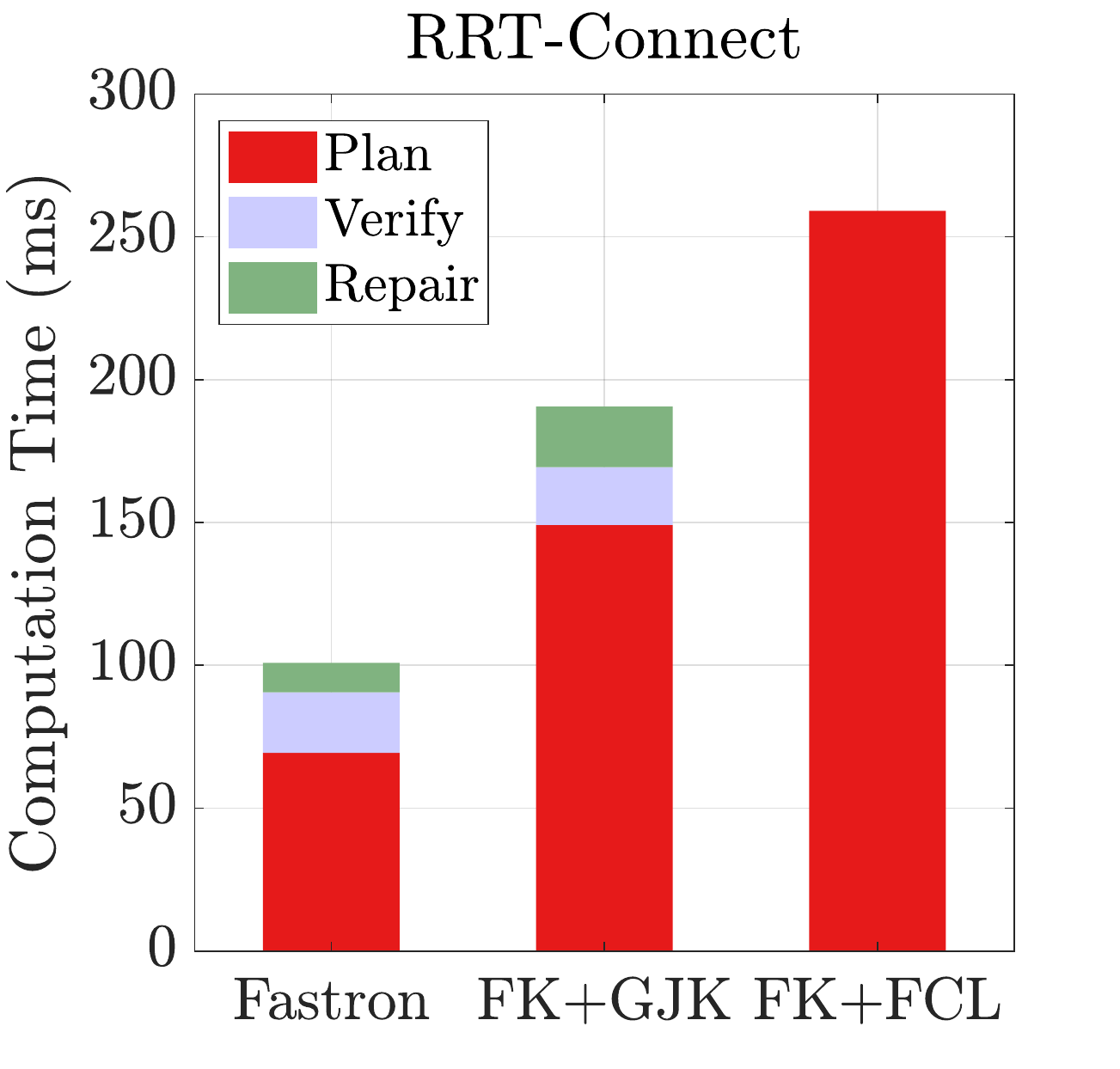}
    \end{subfigure}
    \hfill
	\begin{subfigure}[b]{0.245\linewidth}
      \includegraphics[width=\linewidth,trim={0cm 0cm 0.8cm 0cm},clip]{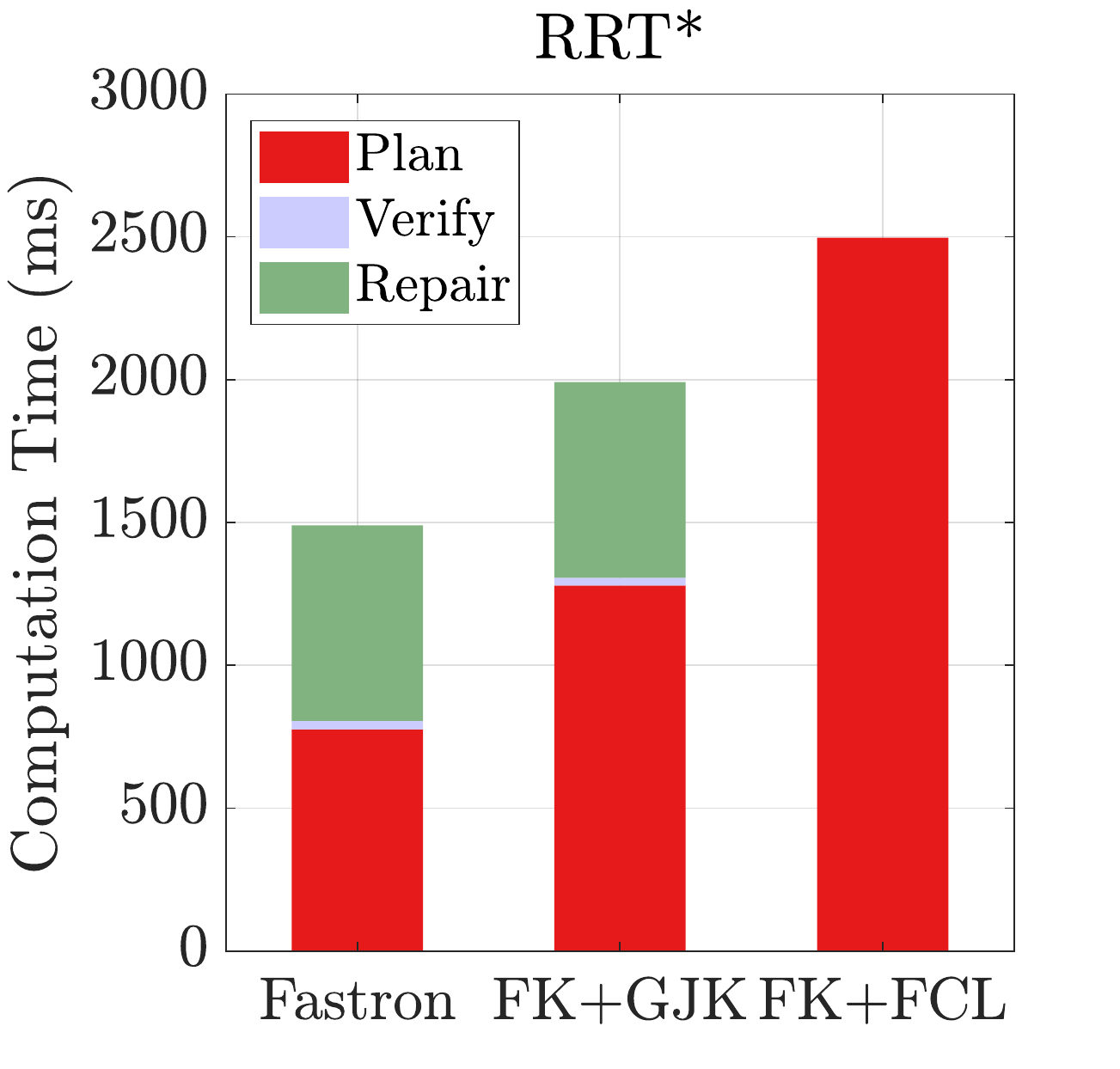}
    \end{subfigure}
    \hfill
    \begin{subfigure}[b]{0.245\linewidth}
      \includegraphics[width=\linewidth,trim={0cm 0cm 0.8cm 0cm},clip]{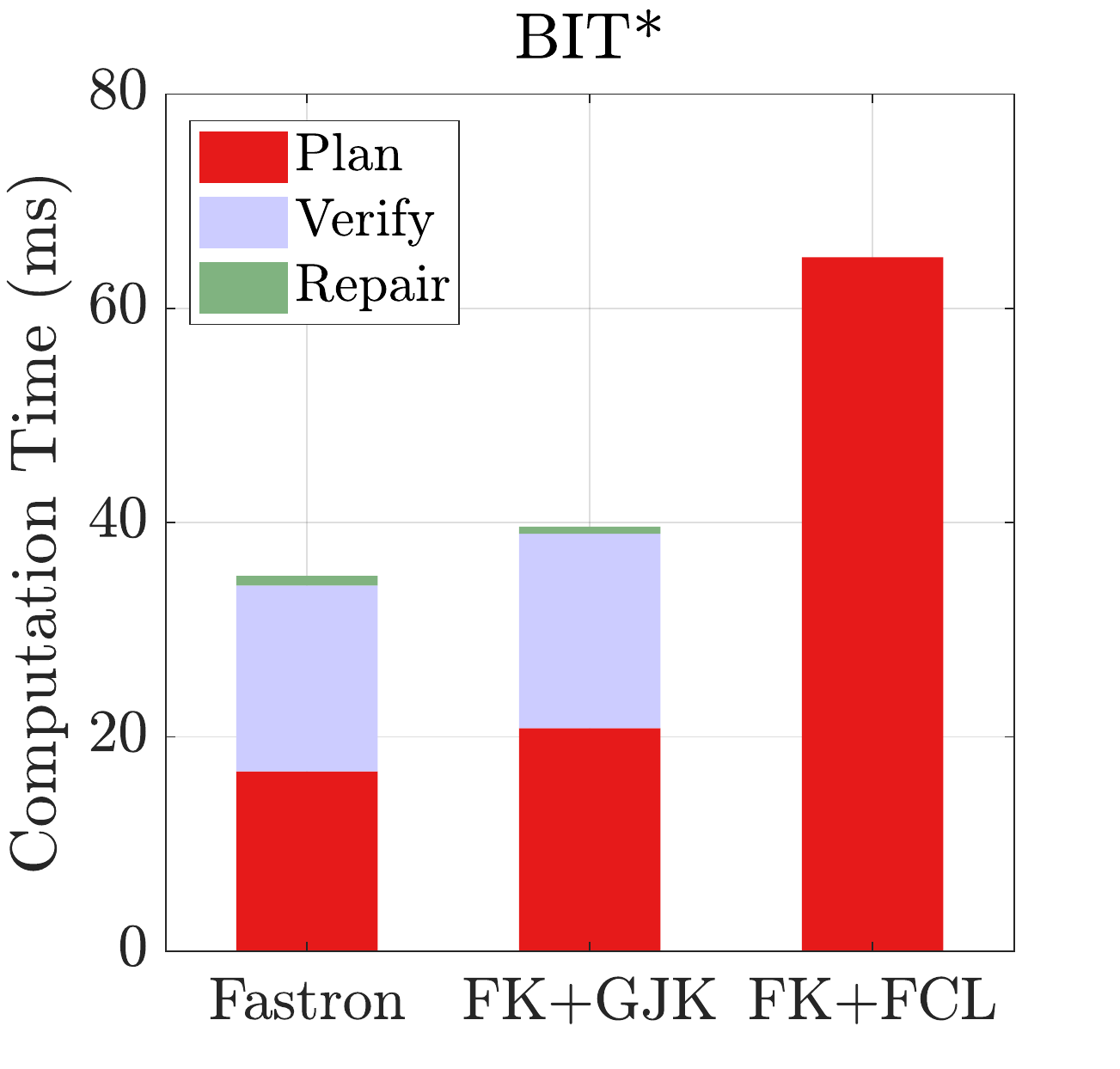}
    \end{subfigure}
    
    \caption{Timings for motion planning for the Baxter robot's right arm with one workspace obstacle using RRT, RRT-Connect, RRT*, and BIT* using Fastron, GJK, and FCL for collision detection. The three rows correspond to planning using only the first 4 DOF, only the first 6 DOF, and all 7 DOF of the arm, respectively. Forward kinematics (FK) are included in the timings for the geometry- and kinematics-based collision detectors.}
    \label{fig:motionPlanTimings}
\end{figure*}

\begin{figure*}[!t]
    \centering
    \begin{subfigure}[b]{0.245\linewidth}
      \includegraphics[width=\linewidth,trim={0cm 0cm 0.8cm 0cm},clip]{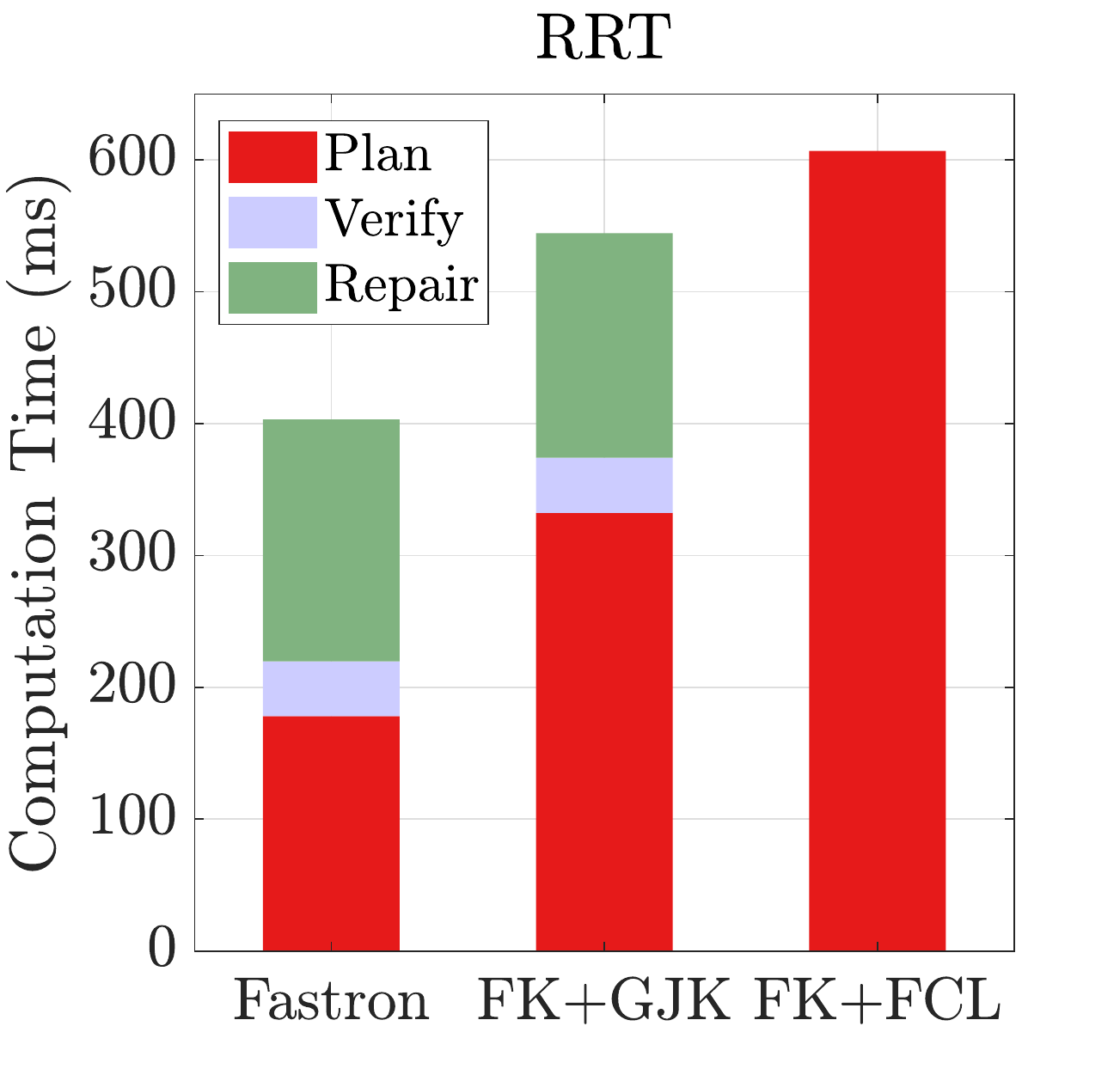}
    \end{subfigure}
    \hfill
    \begin{subfigure}[b]{0.245\linewidth}
      \includegraphics[width=\linewidth,trim={0cm 0cm 0.8cm 0cm},clip]{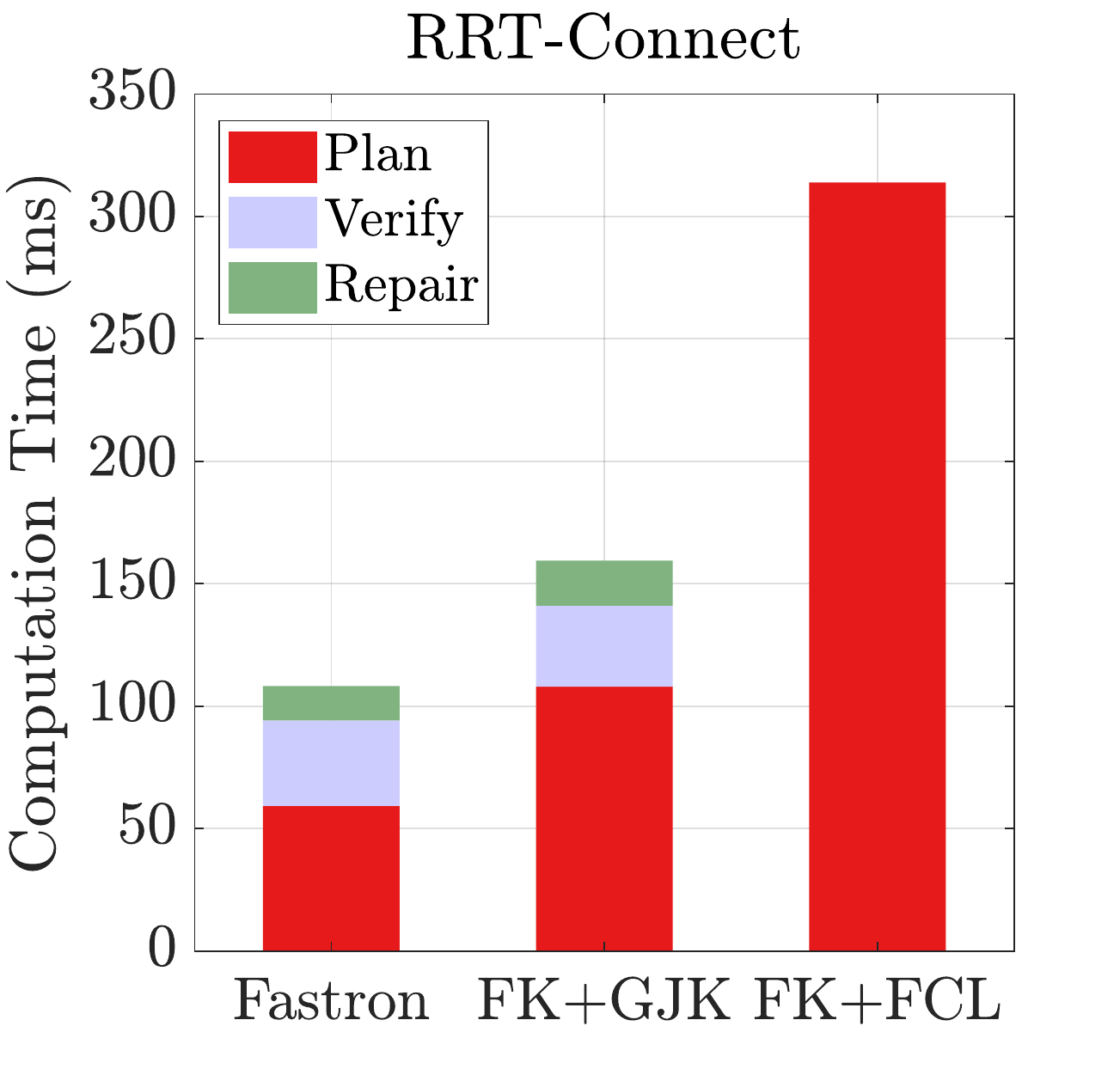}
    \end{subfigure}
    \hfill
	\begin{subfigure}[b]{0.245\linewidth}
      \includegraphics[width=\linewidth,trim={0cm 0cm 0.8cm 0cm},clip]{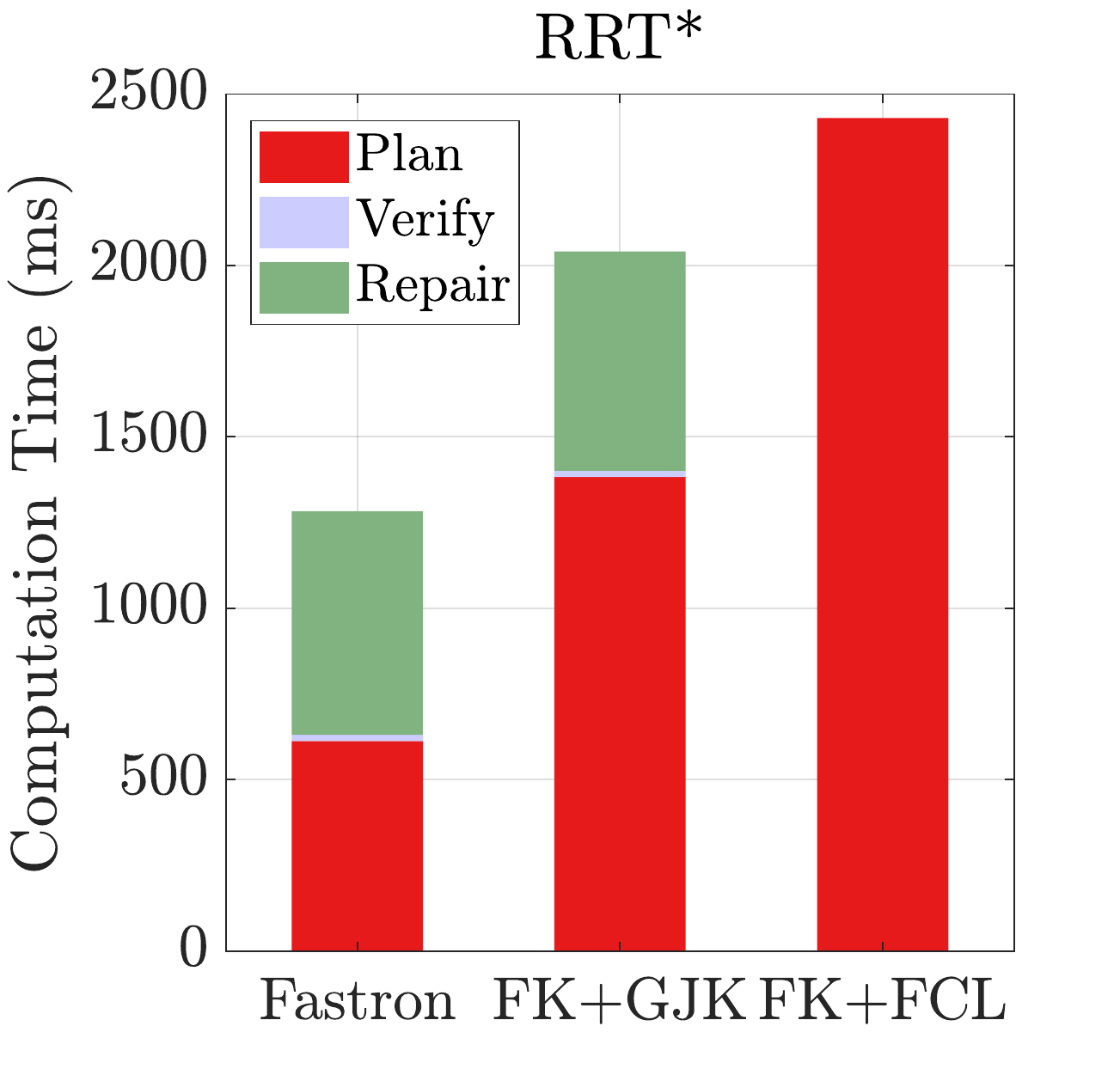}
    \end{subfigure}
    \hfill
    \begin{subfigure}[b]{0.245\linewidth}
      \includegraphics[width=\linewidth,trim={0cm 0cm 0.8cm 0cm},clip]{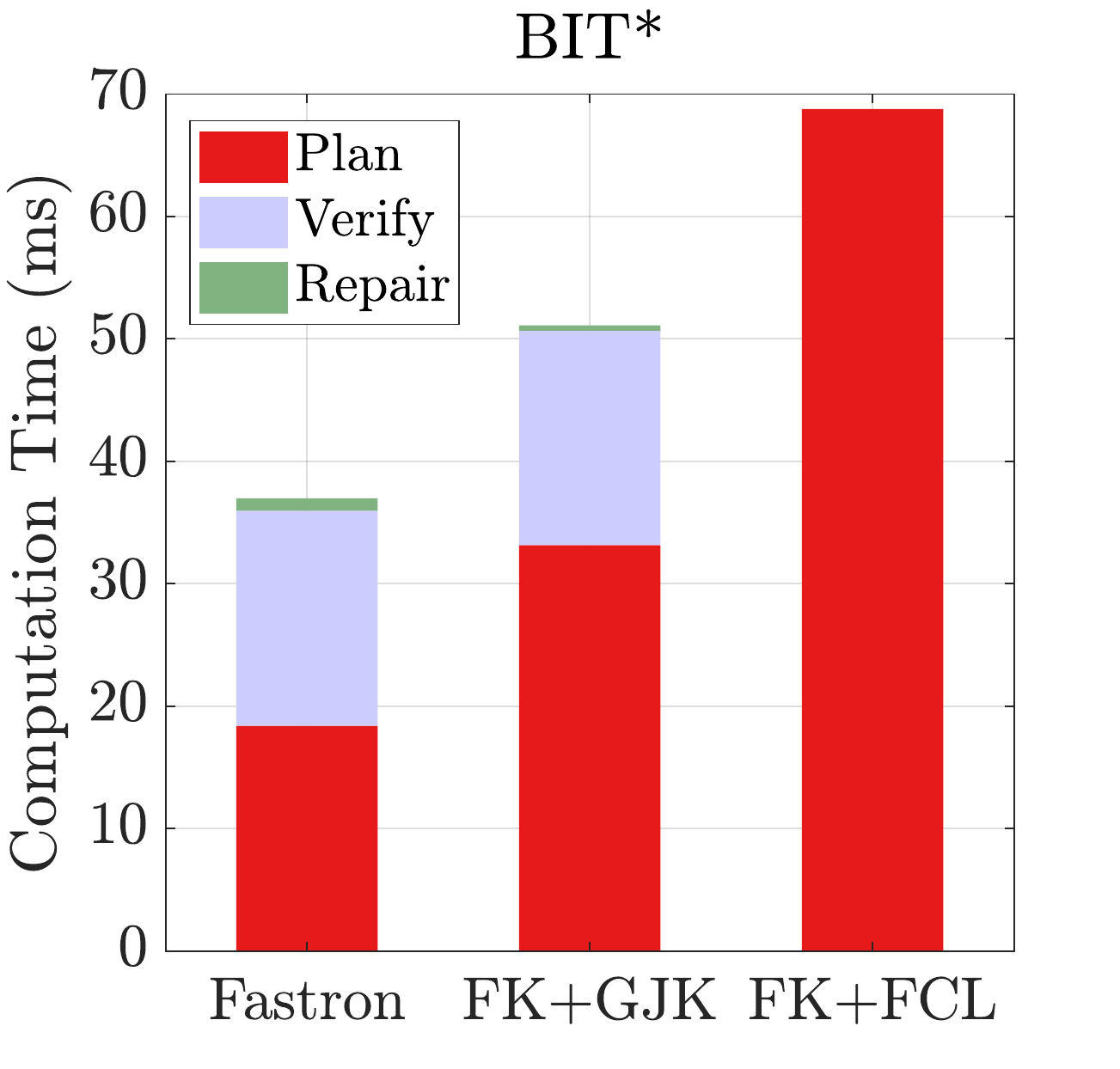}
    \end{subfigure}
    \caption{Timings for motion planning for the Baxter robot's 7 DOF right arm with three workspace obstacles using RRT, RRT-Connect, RRT*, and BIT* using Fastron, GJK, and FCL for collision detection. Forward kinematics (FK) are included in the timings for the geometry- and kinematics-based collision detectors.}
    \label{fig:motionPlanTimingsMultiobs}
\end{figure*}

Fig. \ref{fig:numObsSweep} shows the query times for Fastron, GJK, and FCL with respect to the number of workspace obstacles for up to 100 obstacles. The proportion of $\mathcal{C}_{obs}$ samples (estimated using the average proportion of test samples with the $\mathcal{C}_{obs}$ label) is also included as an explanatory variable. As geometry-based methods, GJK's and FCL's query times both increase as the number of obstacles increase (though the increase in timing for FCL is much less significant), which makes sense as the number of comparisons required increases with more obstacles in the workspace. On the other hand, as a kernel-based method, Fastron's query times increase before \textit{decreasing} with respect to the number of obstacles. For larger numbers of obstacles, Fastron provides collision status results up to an order of magnitude faster than FCL and almost 20 times faster than GJK. The maximum Fastron query time occurs when $\mathcal{C}_{obs}$ occupies approximately $50\%$ of the C-space. This result makes intuitive sense because fewer support points are required when one class is significantly more prevalent than the other, but more support points are required when both classes are roughly equally present.

\subsection{Performance in Motion Planning Application}
\subsubsection{Description of Experiment}
As motion planning is one application that requires frequent collision checks, we apply Fastron to motion planning for one of the Baxter arms to see the effect on the computation time required to generate a feasible plan. We use the same environments methods as used in Section \ref{changingEnvironment}. As proxy collision detection timings were roughly the same for both Fastron (FCL) and Fastron (GJK), we only provide results for Fastron (FCL) in this section.


The experiments involve using the various collision detection methods in standard motion planners. Each trial involved incrementally moving the obstacles in the robot's reachable workspace and generating a new motion plan (using each collision detection method) from scratch each time the obstacle is in a new position. The start and goal configurations are randomly generated such that the arm must move from one side of the workspace to the other, and are only regenerated after a motion plan has been generated using each collision detection method. The OMPL library \cite{Sucan2012} is used to handle the motion planning, and we fill in the state validity checking routine with one of the three collision checking methods. We select RRT \cite{LaValle2008}, RRT-Connect \cite{Kuffner2000}, RRT* \cite{Karaman2011}, and BIT* \cite{Gammell2015} to demonstrate the performance of each collision detection method. RRT and its bidirectional variant RRT-Connect are probabilistically complete motion planners \cite{Choset2005} that terminate once a path is found. RRT* and BIT* are optimizing planners that usually continue to search C-space for a shorter path after an initial feasible path has been found. As we are interested more in generating feasible plans than optimal plans, we terminate RRT* and BIT* once a feasible plan is found, which means the resulting path may not be close to optimal. We include these results to show how quickly the initial path can be found upon which RRT* and BIT* may improve.

Since Fastron and GJK are approximations to true collision detection, we anticipate that the plans that are generated with these approximate methods may actually include some $\mathcal{C}_{obs}$ waypoints. Thus, for Fastron- and GJK-based motion planning, we include verification and repair steps to correct any plan that includes invalid waypoints. In the verification step, we check each waypoint of the motion plan using FCL. In the repair step, invalid segments are first excised out of the motion plan, and then the same motion planner (with FCL as the collision detector) is called again to bridge the excised part of the motion plan. The combination of the verify and repair steps ensures that the final plan is truly collision-free, assuming a motion plan is found. 

\subsubsection{Results for Planning with 4, 6, and 7 DOF}
Fig. \ref{fig:motionPlanTimings} shows the timing results for the various motion planners and different number of actuated DOF. It is clear that Fastron yields the fastest motion plans (up to 4 times faster for the 7 DOF case), and GJK provides the second fastest. Even with the verify and repair steps, Fastron provides the fastest solutions. The verification and repair steps take approximately the same time for both Fastron- and GJK-based motion planning in all cases, so the speed improvement of Fastron over GJK is due to its initial plan. Using GJK in place of FCL did not seem to significantly improve the timing for RRT when including the verify and repair steps, but the timings did improve for the other motion planning methods.

Since RRT* typically took a long time to generate feasible plan using any of the motion planning methods, the verification step took a negligible amount of time compared to the planning and repairing time for both Fastron- and GJK-based planning. On the other hand, planning and repairing is fast when using BIT*, so verification appears to take a larger proportion of time for Fastron- and GJK-based planning.

The motion planning times with three workspace obstacles are shown in Fig. \ref{fig:motionPlanTimingsMultiobs}. Once again, even with the verify and repair steps, using Fastron for collision checking yields the fastest collision-free motion plans. The speed improvement of Fastron is due to its initial planning time (up to 6 times faster). 

As the total amount of planning time is lowest for Fastron for all tested motion planning algorithms, we anticipate that using Fastron will improve planning times for other sampling-based motion planners as well and may provide a fast initial path for optimizing planners.

\section{Further Applications}
\begin{figure*}[t]
\centering
\begin{subfigure}[t]{0.49\linewidth}
	{\includegraphics[width=\linewidth,trim={2.9cm 1cm 2.8cm 3.5cm},clip]{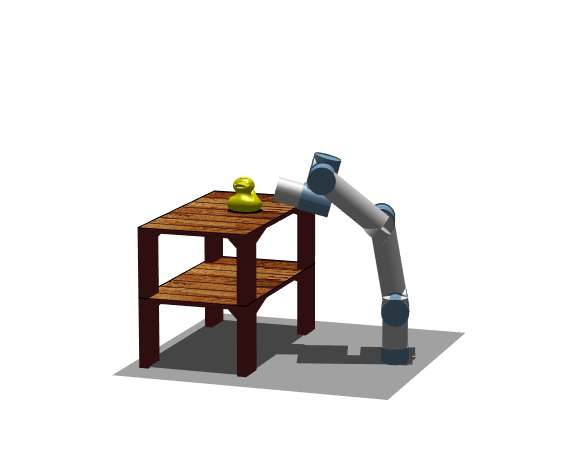}}
 \caption{}
  \label{fig:ur3Example}
  \end{subfigure}
    \hfill
\begin{subfigure}[t]{0.49\linewidth}
	{\includegraphics[width=\linewidth,trim={0.8cm 2cm 0.8cm 1.8cm},clip]{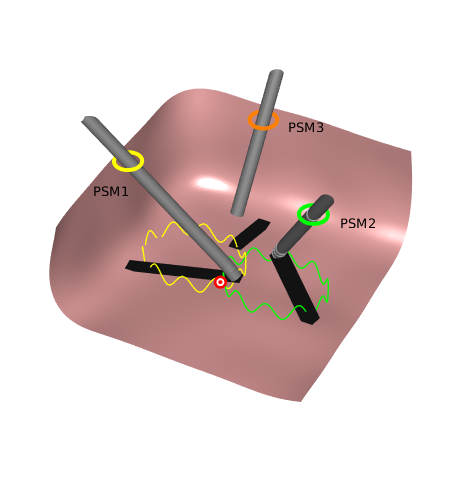}}
  \caption{}
  \label{fig:daVinciExample}
\end{subfigure}
\caption{(a) Pick and place task using a UR3 in simulation. RRTs using Fastron for collision detection are used for motion planning. (b) An autonomous assistant (PSM3) reaching a target among human-controlled manipulators (PSM1 and PSM2) in da Vinci\textregistered\ simulation. RRTs using Fastron for collision detection are used for motion planning.}
\end{figure*}

\subsection{Pick and Place}
We apply the Fastron algorithm for a pick and place task using a Universal Robots UR3 robot and a shelf. This robotic system is designed for tabletop or workbench usage, possibly in the presence of humans, thereby causing a changing enviroment. In our this case, we have the end effector reach from a lower shelf to a higher shelf (without hitting the shelf), as shown in Fig. \ref{fig:ur3Example}. We do not simulate the manipulation part of the task and thus no gripper is shown.

We generate motion plans for the 6 DOF robot using an RRT. In a MATLAB environment, solving an RRT took on average 334.7 ms when using GJK for collision detections. When using Fastron for collision checking, solving an RRT took on average 119.7 ms, which includes verification and repairing where needed. These timings show that Fastron provides almost a threefold speed improvement for this application of motion planning.

\subsection{Autonomous Surgical Assistant}
We apply the Fastron algorithm in the automation of an autonomous surgical assistant in the context of robot-assisted surgery using a da Vinci\textregistered\ Surgical System. For this simulation, we assume that two of the three patient-side manipulators (PSMs) are controlled by a human surgeon, while the third PSM must autonomously perform a task without colliding with the surgeon's movements. Example tasks that this type of assistant may perform autonomously or under supervision of a human include suction, irrigation \cite{Yuh2013}, cautery \cite{Kuntz2004}, and ultrasound transducer placement \cite{Mohareri2014}.

For our scenario, we let PSM1 and PSM2 execute volatile motions that sometimes impede PSM3 from reaching a target configuration. Fig. \ref{fig:daVinciExample} shows an example where PSM1 and PSM2 follow the circular paths while PSM3 attempts to reach the target between the paths. All PSMs are constrained by a remote center of motion, represented as a ring constraint (yellow, green, and orange). PSM3 is disallowed from colliding with the surgeon-controlled arms and any surrounding tissues. In a real system (such as a first-generation da Vinci\textregistered\ Surgical System with da Vinci Research Kit boxes \cite{Chen2013}), joint positions of all PSMs will be known and locations of surrounding tissues may be known via preoperative or intraoperative imaging. When excluding wristed motions, PSM3's controllable DOF include its pitch, yaw, and insertion, meaning we are working with 3-dimensional Fastron models. Excluding the wristed motion is permissible when controlling PSM3 because only three DOF are required to reach the target. As Fastron is based on snapshots of workspace, we use a padded version of PSM3 to improve its responsiveness to PSM1 and PSM2's motions.

While numerous motion planning methods may be chosen to dictate PSM3's motions, we are able to repeatedly solve an RRT from scratch at a 56 Hz rate (ignoring cases where PSM3 is already at its target and no planning is required) when using Fastron for collision detection in a MATLAB environment. In comparison, RRTs can be solved at a 10 Hz rate when using GJK for collision detection, demonstrating the Fastron enables a more responsive autonomous assistant.
%

\section{Conclusion}
In this paper, we theoretically and empirically validate the Fastron algorithm for C-space modeling and proxy collision detection and demonstrate the algorithm's capabilities in accelerating sampling-based motion planning in static and continuously changing environments. Fastron allows collision detection to be performed faster than kinematics-based collision detectors, providing collision status results up to an \textit{order of magnitude} faster than state-of-the-art collision detection libraries, such as the Flexible Collision Library \cite{Pan2012}, the default collision detector used in the MoveIt! motion planning framework, for a 7 DOF Baxter arm with multiple moving workspace obstacles.

Querying our faster proxy collision detection model allows faster execution of collision detection-intensive applications such as motion planning, generating initial motion plans via bidirectional RRTs (implemented with OMPL) 6 times faster than with FCL for a 7 DOF Baxter arm in continuously changing environments.

We integrate the Fastron algorithm in motion planning for various simulated scenarios, such as a pick and place task in which collision-free motion plans were generated 3 times faster with Fastron than with GJK \cite{Gilbert1988}, a collision detector for convex shapes, in a MATLAB environment. We also apply Fastron in motion planning for an autonomous surgical assistant, achieving motion planning via RRTs among volatile human-controlled manipulators at a 56 Hz rate in a MATLAB environment.

We prove analytically that Fastron will always provide a model with positive margin for all training points in a finite number of iterations. We show that Fastron's training speed allows the algorithm to generate a C-space model over 100 times faster than a competing C-space modeling approach for a 4 DOF robot simulated in MATLAB.


All results are achieved with only CPU-based calculations. Model update and query times may be improved upon through parallelization or GPU operations, which we leave to future work.

\begin{appendix}
Consider the maximum margin optimization problem:
\begin{align}
\vec{\alpha}^* = &\argmin_{\vec{\alpha}} \frac{1}{2} \left\|\sum_i\vec{\alpha}_i\phi(\mathcal{X}_i)\right\|^2\quad \notag \\
&\qquad \text{subject to}\ \vec{y}_i\sum_j\vec{\alpha}_jk(\mathcal{X}_i,\mathcal{X}_j)\geq 1\ \forall i\\
= &\argmin_{\vec{\alpha}} \frac{1}{2} \vec{\alpha}^\mathsf{T}\vec{K\alpha}\quad \notag \\
&\qquad \text{subject to}\ \mathds{1} - \vec{YK\alpha}\leq \vec{0}
\end{align}


The Lagrangian is thus defined as:
\begin{align}
& L(\vec{\alpha},\vec{\lambda}) = \frac{1}{2} \vec{\alpha}^\mathsf{T}\vec{K\alpha} + \vec{\lambda}^\mathsf{T}\left(\mathds{1} - \vec{YK\alpha}\right)\ \notag \\
&\qquad \text{subject to}\ \vec{\lambda}_i\geq 0\ \forall i
\end{align}

Consider the dual problem:
\begin{align}
\vec{\lambda}^* &= \argmax_{\vec\lambda\geq \vec 0}\min_{\vec{\alpha}} L(\vec{\alpha}, \vec\lambda)
\end{align}

With no constraints on $\vec{\alpha}$, the minimum of $L(\vec{\alpha},\vec{\lambda})$ given $\vec{\lambda}$ may be found through its gradient:
\begin{align}
& \frac{\partial L(\vec{\alpha}, \vec\lambda)}{\partial \vec \alpha} = \vec{K\alpha} - \vec{KY\lambda}\\
& \Rightarrow \vec{\alpha}^* = \vec{Y\lambda} \\
& \Rightarrow \min_{\vec{\alpha}} L(\vec{\alpha}, \vec\lambda) = -\frac{1}{2}\vec{\lambda}^\mathsf{T}\vec{YKY}\vec\lambda + \mathds{1}^\mathsf{T}\vec{\lambda}
\end{align}

Thus, the problem can be rewritten as:
\begin{align}
\lambda^* &= \argmin_{\vec{\lambda}\geq \vec{0}}\frac{1}{2}\vec{\lambda}^\mathsf{T}\vec{YKY}\vec\lambda - \mathds{1}^\mathsf{T}\vec{\lambda} \\
\Rightarrow \vec{\alpha}^* &= \argmin_{\vec{Y\alpha}\geq \vec{0}}\frac{1}{2}\vec{\alpha}^\mathsf{T}\vec{K}\vec\alpha - \mathds{1}^\mathsf{T}\vec{Y\alpha} \\
&= \argmin_{\vec{Y\alpha}\geq \vec{0}}\frac{1}{2}\vec{\alpha}^\mathsf{T}\vec{K}\vec\alpha - \vec{y}^\mathsf{T}\vec{\alpha}
\end{align}
\end{appendix}

\bibliographystyle{IEEEtran}
\bibliography{references}
\end{document}